\documentclass[twoside]{article}

\usepackage[accepted]{aistats2023}

\usepackage{etex}
\usepackage{amssymb,amsfonts,amsmath,amstext,amsthm,mathrsfs}
\usepackage{graphics,latexsym,epsfig,psfrag,wrapfig,comment,paralist}
\usepackage{xspace}
\usepackage{subcaption,bm,multirow}
\usepackage[utf8]{inputenc} 
\usepackage[T1]{fontenc}    
\usepackage{url}            
\usepackage{booktabs}       
\usepackage{amsfonts}       
\usepackage{nicefrac}       
\usepackage{microtype}      
\usepackage{xcolor}         
\usepackage{bbold}
\usepackage{centernot}
\usepackage{algorithm}
\usepackage[noend]{algpseudocode}
\usepackage{prettyref}
\usepackage{listings}
\usepackage{makecell}
\usepackage{tikz}
\usepackage{pgfplots}
\usetikzlibrary{shapes}
\usetikzlibrary{positioning, arrows, matrix, calc, patterns, fit}
\usepackage{caption}
\usepackage{array}
\usepackage{mdwmath}
\usepackage{mdwtab}
\usepackage{eqparbox}
\usepackage{amsfonts}
\usepackage{tikz}
\usepackage{bigstrut,threeparttable}
\usepackage{amsthm}
\usepackage{array}
\usepackage{bbm}
\usepackage{epstopdf}
\usepackage{mdwmath}
\usepackage{mdwtab}
\usepackage{eqparbox}
\usepackage{tikz}
\usepackage{latexsym}
\usepackage{amssymb}
\usepackage{graphicx}
\usepackage{mathrsfs}
\usepackage{epsfig}
\usepackage{psfrag}
\usepackage{setspace}
\usepackage[CJKbookmarks=true,
            bookmarksnumbered=true,
            bookmarksopen=true,
            colorlinks=true,
            citecolor=blue,
            linkcolor=blue,
            anchorcolor=red,
            urlcolor=blue
            ]{hyperref}
\usepackage{natbib}
\setcitestyle{authoryear,round}
\newtheoremstyle{break}
  {\topsep}{\topsep}%
  {\itshape}{}%
  {\bfseries}{}%
  {\newline}{}%

\usepackage{cleveref}

\usepackage{pgfplots}
\pgfplotsset{compat=1.7}
\usetikzlibrary{external}
\usetikzlibrary{spy}
\usepgfplotslibrary{fillbetween}
\newcommand{\figwidth}{0.24\textwidth}

\NewEnviron{resizealign}{\sbox0{
    $\begin{matrix}\displaystyle\BODY\end{matrix}$}%
  \sbox1{$(\theequation)$}%
  \sbox2{\parbox{\dimexpr \wd0 + 2\wd1}%
    {\begin{align}\BODY\end{align}}}
  \noindent\resizebox{\hsize}{!}{\usebox2}%
}

\def\[#1\]{\begin{align}#1\end{align}}
\def\(#1\){\begin{align*}#1\end{align*}}

\def\argmax{\operatornamewithlimits{arg\,max}}
\def\argmin{\operatornamewithlimits{arg\,min}}

\newcommand{\bprf}{\begin{proof}}
\newcommand{\eprf}{\end{proof}}
\newcommand{\blem}{\begin{lemma}}
\newcommand{\elem}{\end{lemma}}

\newcommand{\bE}{\mathbb{E}}

\usepackage{setspace}
\usepackage{amsfonts}
\usepackage{bm}
\usepackage{xspace}
\usepackage{fullpage}
\usepackage{dsfont}

\usepackage{graphicx} 
\usepackage{url}
\usepackage{enumerate}
\usepackage{xcolor,soul}
\usepackage{thm-restate}
\usepackage{verbatim}
\usepackage{setspace}
\usepackage{paralist}
\usepackage{mdwmath}
\usepackage{amssymb}
\usepackage{amsmath}
\usepackage{amsthm}
\usepackage{xfrac}
\usepackage{mathtools}  
\usepackage{chngcntr}
\usepackage{tabulary}
\usepackage{booktabs}
\usepackage{epsfig}
\usepackage{wrapfig}
\usepackage{lipsum}
\usepackage{caption}
\usepackage{multirow}
\usepackage{algcompatible} 
\usepackage{footnote}
\usepackage{paralist}
\usepackage{hyperref}

\newtheorem{asm}{Assumption}

\newcommand{\R}{\mathbb{R}}

\newcommand{\E}{\mathbb{E}}

\newcommand{\vecx}{\mathbf{x}}
\newcommand{\vecw}{\mathbf{w}}

\newcommand{\vecv}{\mathbf{v}}

\newcommand{\vecg}{\mathbf{g}}

\newcommand{\vecz}{\mathbf{z}}

\newcommand{\gradf}{\nabla f}
\newcommand{\gradh}{\nabla h}

\newcommand{\gradF}{\nabla F}

\newcommand{\bigo}{\mathcal{O}}

\newcommand{\W}{\mathcal{W}}

\newcommand{\twonms}[1]{\|#1\|_2}

\newcommand{\innerps}[2]{\langle#1,#2\rangle}

\newcommand{\floor}[1]{\lfloor#1\rfloor}

\newcommand{\Z}{\mathcal{Z}}

\definecolor{darkgreen}{rgb}{0.0, 0.6, 0.2}
\definecolor{lightgray}{rgb}{0.8, 0.85, 0.85}

\usepackage{amsmath,amsfonts,bm}

\def\eqref#1{Equation (\ref{#1})}

\def\floor#1{\lfloor #1 \rfloor}
\def\1{\bm{1}}

\DeclareMathAlphabet{\mathsfit}{\encodingdefault}{\sfdefault}{m}{sl}
\SetMathAlphabet{\mathsfit}{bold}{\encodingdefault}{\sfdefault}{bx}{n}

\newcommand{\KL}{D_{\mathrm{KL}}}

\DeclareMathOperator{\Tr}{Tr}

\newcommand{\Cov}{\mathsf{Cov}}
\newcommand{\TV}{\mathsf{TV}}

\usepackage[normalem]{ulem}

\renewcommand{\cite}[1]{\citep{#1}}

\renewcommand{\cite}[1]{\citep{#1}}

\usepackage{color}

\newtheorem{definition}{\textbf{Definition}}[section]

\newtheorem{lemma}{\textbf{Lemma}}[section]
\newtheorem{theorem}{\textbf{Theorem}}[section]

\newtheorem*{insight*}{\textbf{Observation}}

\newtheorem*{proposition*}{\textbf{Proposition}}

\newtheorem*{lemmai*}{\textbf{Lemma (informal)}}

\newcommand{\cW}{\mathcal{W}}

\newcommand{\system}{{the proposed protocols}\xspace}

\definecolor{darkgreen}{rgb}{0.0, 0.5, 0}

\newcommand{\F}{Fig.}
\newcommand{\ignore}[1]{}
\newcommand{\revision}[1]{#1}

\begin{document}

\twocolumn[

\aistatstitle{Byzantine-Robust Federated Learning with Optimal Statistical Rates}
\aistatsauthor{Banghua Zhu$^{*}$ \And Lun Wang$^{*\dagger}$ \And  Qi Pang$^{*\ddagger}$}
\aistatsaddress{UC  Berkeley \And  Google \And   Carnegie Mellon University } 

\aistatsauthor{Shuai Wang \And Jiantao Jiao \And Dawn Song \And Michael I. Jordan}

\aistatsaddress{HKUST \And  UC  Berkeley \And    UC  Berkeley \And UC  Berkeley} ]



\author{%
  Banghua Zhu$^*$\\
  Department of EECS\\
  University of California, Berkeley\\
  Berkeley, CA, 95054 \\
  \texttt{banghua@berkeley.edu} \\
   \And
   Lun Wang$^*$ \\
  Department of EECS\\
  University of California, Berkeley\\
   Berkeley, CA, 95054 \\
  \texttt{wanglun@berkeley.edu} \\
   \AND
   Qi Pang$^*$ \\
    Department of Computer Science \& Engineering\\
  Hong Kong University of Science and Technology\\
   Hong Kong, China \\
  \texttt{flashpangqi@gmail.com} \\
   \And
    Shuai Wang\\
    Department of Computer Science \& Engineering\\
  Hong Kong University of Science and Technology\\
   Hong Kong, China \\
  \texttt{shuaiw@cse.ust.hk} \\
   \And
   Jiantao Jiao  \\
  Department of EECS\\
  University of California, Berkeley\\
   Berkeley, CA, 95054 \\
  \texttt{jiantao@berkeley.edu} \\
      \And
   Dawn Song  \\
  Department of EECS\\
  University of California, Berkeley\\
   Berkeley, CA, 95054 \\
  \texttt{dawnsong@cs.berkeley.edu} \\
      \And
   Michael I. Jordan  \\
  Department of EECS\\
  University of California, Berkeley\\
   Berkeley, CA, 95054 \\
  \texttt{jordan@cs.berkeley.edu} \\
}

\def\thefootnote{*}\footnotetext{These authors contributed equally to this work.}\def\thefootnote{\arabic{footnote}}
\def\thefootnote{$\dagger$}\footnotetext{The work of Lun Wang is done at UC Berkeley.}\def\thefootnote{\arabic{footnote}}
\def\thefootnote{$\ddagger$}\footnotetext{The work of Qi Pang is done at HKUST.}\def\thefootnote{\arabic{footnote}}

\begin{abstract}
We propose Byzantine-robust federated learning protocols with nearly optimal statistical rates based on recent progress in high dimensional robust statistics. 
In contrast to prior work, our proposed protocols improve the dimension dependence and achieve a near-optimal statistical rate for strongly convex losses. We also provide  statistical lower bound for the problem. 
%
%
%
For experiments, we benchmark against competing protocols and show the empirical superiority of the proposed protocols. 
%
%
\end{abstract}

\section{Introduction}
\label{sec:intro}

Federated learning (FL) has drawn considerable attention as a novel distributed
learning paradigm in recent years. In FL, users (worker machines) collaborate to train a model
using a centralized server (master machine), while all data is stored locally to protect users'
privacy. The privacy benefit has motivated the adoption of FL in a variety of sensitive
applications, including Google GBoard, healthcare services, and self-driving
cars. 

However, vanilla FL  has been
demonstrated to be vulnerable to a range of
attacks~\cite{bagdasaryan2020backdoor,bhagoji2019analyzing,nasr2019comprehensive,sun2021data,luo2021feature}.
There are  two mainstream vulnerabilities in FL, namely \emph{Byzantine robustness} and \emph{privacy}.
In the former, a small number of clients can
behave maliciously in a large-scale FL system and influence the jointly-trained
FL model in a stealthy manner~\cite{bagdasaryan2020backdoor,bhagoji2019analyzing,fang2020local,sun2021data}. 
For the majority of  SGD-based FL
algorithms~\cite{mcmahan2017federated}, the centralized server averages the
local updates to obtain the global update, which is vulnerable to even a single
malicious client. Particularly, a malicious client can craft its update in such
a way that it prevents the global model from converging or leads it to a
sub-optimal minimum.
With respect to terms of privacy, the vulnerability is that the centralized server can infer information about the local data of the
clients by inspecting their updates~\cite{nasr2019comprehensive,luo2021feature}.

In this work, we mainly focus on the Byzantine robustness of the FL protocols.  { Byzantine-robust FL protocols~\cite{blanchard2017machine,yin2018byzantine,fu2019attack,pillutla2019robust, alistarh2018byzantine, karimireddy2021learning, allen2020byzantine, karimireddy2020byzantine, gorbunov2022variance, velicheti2021secure, data2021byzantine} have been proposed to suppress the influence of malicious clients' updates (See Appendix~\ref{sec:related_work} for a full discussion and comparison).}
In the most comparable work, \citet{yin2018byzantine} apply a coordinate-wise median/trimmed mean to aggregate local updates. { Under the assumption of bounded variance $\mathbb{E}_{p^\star}[\|\vecg - \mathbb{E}[\vecg]\|_2^2]\leq \tilde \sigma^2$ and coordinate-wise bounded skewness (or coordinate-wise sub-exponential) on the gradient distribution, one can achieve a statistical error of $\bigo(\tilde \sigma (\epsilon/\sqrt{n} + \sqrt{d/mn}))$, where $\epsilon, n, m$ are the fraction of malicious nodes, number of samples in each client, and number of clients, respectively. It has also been shown in~\citet{yin2018byzantine} that any FL protocol cannot has a rate better than $\Omega(\epsilon/\sqrt{n}+\sqrt{d/mn})$. }

\revision{
However, as is noted in~\cite{yin2018byzantine}, the term $\tilde \sigma$ usually depends on the dimension of the data $d$. As a concrete example, $\tilde \sigma = \sqrt{d}$ when the noise of the gradient is standard Gaussian with identity covariance $I$. And this results in a sub-optimal rate of $\bigo( \epsilon\sqrt{d/n} + \sqrt{d^2/mn})$. Compared to the lower bound above, the rate is far from optimal in terms of the dependence on the dimension of gradients $d$, which is usually high for modern machine learning architectures such as neural networks. Furthermore, they also rely on strong, hard-to-verify extra assumptions, including bounded coordinate-wise skewness or sub-exponential gradients, which are unlikely to be satisfied in practice. Similar issues on the dimension dependence and strong extra assumptions also appear in the analysis other FL protocols with trace bound assumption~\cite{blanchard2017machine, karimireddy2021learning,  karimireddy2020byzantine,  velicheti2021secure}, which we provide detailed discussion in Appendix~\ref{sec:related_work}.

In this paper, we make the only distributional assumption  that  the spectral norm of the covariance of the gradient is bounded by $\sigma^2$, as in Assumption~\ref{asm:bounded-variance} (note that the assumption in~\cite{yin2018byzantine} in is equivalent to bounded trace of the covariance, and for   Gaussian with identity covariance we have $\sigma = 1$ while $\tilde\sigma =\sqrt{d}$).  In this case, we aim to propose FL protocols that achieve an optimal statistical rate of $\Theta(\sigma(\epsilon / \sqrt{n} + \sqrt{d/mn}))$.  }%

\revision{On the other hand, a standard metric for  robust estimators is the \emph{breakdown point}, which measures the largest fraction of Byzantine clients an algorithm is insensitive to~\cite{huber1973robust, donoho1982breakdown}.  While this metric has been neglected in most of the previous work on computationally efficient high dimensional robust estimators, it is important for the practical performance of FL protocols.  Thus, we ask the question: }
\begin{quote}
    \emph{Under the assumption of bounded covariance of gradients, can we design Byzantine-robust FL algorithms that are near-optimal with respect to all the parameters, $\epsilon, n, m, d$, and achieve a good breakdown point?}
\end{quote}

\noindent\textbf{Our Contribution.}
In this work, we make one step towards resolving the  problem by  providing a thorough  analysis of the statistical rate and breakdown point of  FL algorithms. 
\revision{We identify the main reason for the sub-optimal rate in~\cite{yin2018byzantine} as the  sub-optimality in the robust aggregation algorithms. It is known that  \textsc{coordinate-wise median}, \textsc{coordinate-wise trimmed mean} and \textsc{geometric median} have an error of at least $\Omega(\sigma\sqrt{\epsilon d})$ under the bounded operator norm assumption for robust mean estimation~\cite{steinhardt2019lecture, lai2016agnostic}, which can be highly sub-optimal when the dimension is high. Following our lower bound analysis in Appendix~\ref{sec:lower_bound}, the lower bound for robust mean estimation directly results in a lower bound of $\Omega({\sigma\sqrt{\epsilon d/ n}})$ for the performance of any FL protocols using either of the three algorithms as aggregation rules.  }

The recent progress in statistically and computationally efficient robust estimators  in high dimension~\cite{diakonikolas2016robust,diakonikolas2017being,diakonikolas2019sever,diakonikolas2019efficient, diakonikolas2019nearly, diakonikolas2020outlier, steinhardt2017certified,  steinhardt2018resilience, steinhardt2018robust, zhu2019deconstructing, zhu2019generalized,  zhu2022robust, lugosi2017lectures, lugosi2019sub, lecue2020robust, lugosi2021robust, hopkins2020robust} aim to remove the extra $\sqrt{d}$ factor in mean estimation, and thus can be helpful in achieving near-optimal rate for FL protocols. We propose  Byzantine-robust protocols based on various robust estimators, including \textsc{Filtering}~\cite{diakonikolas2016robust,steinhardt2017certified, zhu2021robust}, \textsc{No-regret}~\cite{zhu2021robust, hopkins2020robust}, \textsc{GAN}~\cite{zhu2022robust, gao2020generative}, \textsc{Bucketing-No-regret} and \textsc{Bucketing-Filtering}~\cite{diakonikolas2020outlier}. We show that these estimators improve the statistical rate reported in~\citet{yin2018byzantine} in terms of the dependence on dimension $d$.
%
%
We provide a thorough comparison between different protocols in terms of their statistical rate, breakdown points and computational complexity. Our results are summarized in Table~\ref{tab:rate}. Note in particular that \textsc{coordinate-wise median} and \textsc{coordinate-wise trimmed mean} have a rate that loses an extra $\sqrt{d}$  compared to the lower bound. Moreover, the  
\textsc{No-regret} or \textsc{Filtering} algorithms can improve the dimension dependence in the term with $\epsilon$ and achieve a good breakdown point, but the rate without a Byzantine adversary is still $\widetilde\bigo(\sqrt{\frac{d^2}{mn}})$ instead of $\widetilde \bigo(\sqrt{\frac{d}{mn}})$. On the other hand, by a particular choice of parameters for \textsc{No-regret} and \textsc{Filtering} in~\eqref{eq:robust_subroutine4} and (\ref{eq:robust_subroutine5}), one can improve the rate to $\bigo(\sqrt{\frac{\epsilon}{n}+\frac{d(\log(d)+\log(1+mnDL))}{mn}})$ while possibly sacrificing the breakdown point (since the breakdown point is not shown to match $1/2$). Here $D, L$ are the bounds on the radius of the space of $\mathcal{W}$ and the smoothness of $f$. To remove the $\log(d)$ factor, we further apply a bucketing step as pre-processing, as is done in median-of-means procedure~\cite{lugosi2019sub} and in~\citet{diakonikolas2020outlier}.  Compared with the lower bound, we still have an extra $\log(1+mnDL)$ factor due to the union bound analysis, which exists for all the bounds we considered. It remains an open problem whether one can remove this term and achieve exactly tight rate. 

{ Due to inconsistency in the settings and assumptions, the table does not include the rate for \cite{blanchard2017machine,pillutla2019robust, alistarh2018byzantine, karimireddy2021learning, allen2020byzantine, karimireddy2020byzantine, velicheti2021secure}. Instead, we provide a detailed discussions on the setting, assumption and rate of these Byzantine-robust protocols in Appendix~\ref{sec:related_work}.}

\begin{table*}[!htbp]
\footnotesize
    \centering
    \resizebox{1.0\textwidth}{!}{
    \begin{tabular}{|c|c|c|c|}
        \hline
        \textbf{Algorithm} & \textbf{Statistical rate} &  \textbf{Breakdown point} & \textbf{Computation complexity} \\ \hline
     \textsc{Median}~\cite{yin2018byzantine} & $\widetilde\bigo(\sqrt{\frac{\epsilon d}{n}+\frac{d^2}{mn}})$& $1/2$ & $\widetilde\bigo(md)$ \\ 
    \textsc{Trimmed mean}~\cite{yin2018byzantine} & $\widetilde\bigo(\sqrt{\frac{\epsilon d}{n}+\frac{d^2}{mn}})$&  $1/2$ & $\widetilde\bigo(md)$ \\
          \textsc{No-regret} (Eq. (\ref{eq:robust_subroutine1})) & $\widetilde\bigo(\sqrt{\frac{\epsilon}{ n}+\frac{d^2}{mn}})$& $1/3$ & $\widetilde\bigo((m+d^3)d)$  \\ 
         \textsc{Filtering} (Eq. (\ref{eq:robust_subroutine2})) &$\widetilde\bigo(\sqrt{\frac{\epsilon} {n}+\frac{d^2}{mn}})$ & $1/2$ & $\widetilde\bigo(\epsilon m^2 d^3)$ \\
        \textsc{ GAN} (Eq. (\ref{eq:robust_subroutine3})) & $\widetilde\bigo(\sqrt{\frac{\epsilon} {n}+\sqrt{\frac{d}{mn^2}}})$ & - & - \\ 
          \textsc{No-regret} (Eq. (\ref{eq:robust_subroutine4})) & $\bigo(\sqrt{\frac{\epsilon}{n}+\frac{d(\log(d)+\log(1+nmDL))}{mn}})$ & - &  $\widetilde\bigo((m+d^3)d)$  \\ 
        \textsc{ Filtering } (Eq. (\ref{eq:robust_subroutine5})) & $\bigo(\sqrt{\frac{\epsilon}{n}+\frac{d(\log(d)+\log(1+nmDL))}{mn}})$ & - &  $\widetilde\bigo(\epsilon m^2 d^3)$ \\
        \textsc{Bucketing-No-regret} & $\bigo(\sqrt{\frac{\epsilon}{n}+\frac{d\log(1+nmDL)}{mn}})$ & - &  $\widetilde\bigo((m+d^3)d)$ \\
        \textsc{ Bucketing-Filtering} & $\bigo(\sqrt{\frac{\epsilon}{n}+\frac{d\log(1+nmDL)}{mn}})$ & - &  $\widetilde\bigo(\epsilon m^2 d^3)$ \\
         Lower bound (Thm.~\ref{thm:lower_bound}) & $\Omega(\sqrt{\frac{\epsilon}{n}+\frac{d}{mn}})$ & $1/2 $ & - \\
         \hline
    \end{tabular}}
    
    \caption{Comparison of rates between different Byzantine-robust federated learning protocols for gradients with bounded covariance and strongly convex loss function.  Here we assume the desired confidence $\delta$ is some constant, and $D, L$ are the bounds on the radius of the space of $\mathcal{W}$ and the smoothness of $f$. The  median and trimmed mean refer to coordinate-wise operations. $\widetilde \bigo(\cdot)$ refers to the rate when omitting logarithmic factors. We omit the communication complexity since it is $\bigo(d)$ for all the algorithms. For the computational complexity, we only consider the computation conducted in the master machine. The `-' in breakdown point represents that the breakdown point has not been completely determined for the algorithm. }
    \label{tab:rate}
\end{table*}


%

We evaluate all proposed protocols under five attacks over two datasets
and compare  with six other Byzantine-robust federated learning protocols. Evaluation results show that
\system constantly achieve optimal or close-to-optimal performance under all
the attacks. 

We also remark here that although the computational complexity is higher for the new proposed algorithms compared to median and trimmed mean, one can improve them to $\widetilde\bigo(md)$ without hurting the statistical rate by integrating the filtering / no-regret algorithm with SDP solvers or matrix multiplicative weights update, see e.g.~\cite{cheng2019high, dong2019quantum}. 

In addition,  the proposed protocol can be naturally integrated with the secure aggregation schemes~\cite{bonawitz2017practical} to
provide \emph{bidirectional defense} against both a semi-honest server and
Byzantine-malicious clients, thus achieving Byzantine robustness and privacy simultaneously. { We provide more discussions in Appendix~\ref{sec:privacy}.}
%
\section{Preliminaries}
\label{sec:threat-model}

In this section, we collect useful notation, review the general pipeline of FL and introduce the threat
model and the defense goal.

\noindent \textbf{Notation.} We use bold lower-case letters (\emph{e.g.}
\textbf{a},\textbf{b},\textbf{c}) to denote (random) vectors, and $[n]$ to denote
the set $\{1\cdots n\}$.  Let $\Cov_p(\vecx) = \E_p[\vecx \vecx^\top] - \E_p[\vecx]\E_p[\vecx]^\top$ be the covariance of random variable $\vecx$ distribution $p$. $\|\cdot\|_2$ is the $\ell_2$ norm for vector, or the spectral norm for matrix, i.e. $\|\vecw\|_2 = \vecw^\top \vecw$ for vector $\vecw$, $\|\mathbf{M}\|_2 = \sup_{\|\vecv\|_2\leq 1} \|\mathbf{M} \vecv\|_2$ for matrix $\mathbf{M}$. We write $f(x) = \bigo(g(x))$ for $x\in A$ if   there exists some positive real number $M$   such that $|f(x)|\leq M g(x)$ for all $x\in A$. If $A$ is not specified, we have $|f(x)|\leq M g(x)$ for all $x\in [0,+\infty)$ (thus the notation is non-asymptotic).  We use $\widetilde \bigo(\cdot)$ to be the big-$\bigo$ notation ignoring logarithmic factors. 

We also introduce  concepts from convex analysis for differentiable functions $h(\cdot):\R^d\rightarrow\R$.

\begin{definition}[Lipschitz]\label{def:lipschitz}
$ h $ is $ L $-Lipschitz if $|h(\vecw) - h(\vecw')| \le L\twonms{\vecw - \vecw'}, \forall~\vecw,\vecw'$.
\end{definition}

\begin{definition}[Smoothness]\label{def:smoothness}
$ h $ is $ L' $-smooth if $\twonms{\gradh(\vecw) - \gradh(\vecw')} \le L'\twonms{\vecw - \vecw'}, \forall~\vecw,\vecw'$.
\end{definition}

\begin{definition}[Strong convexity]\label{def:strong-cvx}
$ h $ is $ \lambda $-strongly convex if $ h(\vecw') \ge h(\vecw) + \innerps{\gradh(\vecw)}{\vecw' - \vecw} + \frac{\lambda}{2}\twonms{\vecw' - \vecw}^2, \forall~\vecw,\vecw'$.
\end{definition}

%
%


\smallskip
\noindent \textbf{FL Pipeline.}~In an FL system, there is one master machine
$\mathcal{S}$ and $m$ worker machines $\mathcal{C}_i, i\in[m]$. Each client holds $n$ data
samples drawn i.i.d.\ from some unknown distribution $p^\star$ on the sample space $\mathcal{Z}$.
Let $f(\vecw;\textbf{z})$ be the loss as a function of the model parameter
$\vecw\in \mathcal{W}\subset \mathbb{R}^d$ and a data sample $\textbf{z}\in\mathcal{Z}$, where $\mathcal{W}$ is the parameter space. 
Let
$F(\vecw)=\mathbb{E}_{\textbf{z}\sim p^\star}[f(\vecw;\textbf{z})]$
be the population risk function. Denote by $\vecz^{i,j}$ the $j^{th}$ data point on the $i^{th}$ worker machine, and $F_i(\vecw) := \frac{1}{n}\sum_{j=1}^n f(\vecw;\vecz^{i,j})$ the empirical risk function for the $i^{th}$ worker.
The goal is to learn a model $\vecw$ such that the population risk function is minimized: 
$    \vecw^{*}=\arg\min_{\vecw\in\mathcal{W}}F(\vecw).$

%
To learn $\vecw^*$, the whole system runs a $T$-round FL protocol.
Initially, the server stores a global model $\vecw_0$.
In the $t^{th}$ round, $\mathcal{S}$ broadcasts the global model $\vecw_{t-1}$ to the $m$ clients.
The clients then run the local optimizers (e.g., SGD, Adam, RMSprop), compute the difference $\vecg_{i}(\vecw_{t-1})$ between the optimized model and the global model, and upload the difference to $\mathcal{S}$.
In the $t^{th}$ round, $\mathcal{S}$ aggregate the differences and  updates the global model to $\vecw_t$ according to some protocol (e.g. the averaging protocol $\vecw_t=\vecw_{t-1}+\frac{1}{m}\sum_{i=1}^m \vecg_{i}(\vecw_{t-1})$).
%


\smallskip
\noindent \textbf{Threat Model.}
We assume that clients are $\epsilon$-Byzantine
malicious, meaning that in each round, at most $\epsilon m$ clients are malicious: they can
deviate arbitrarily from the protocol and tamper with their own updates for
profitable or even adversarial purposes. 
The master machine communicates with the worker machines using some predefined protocol. The Byzantine machines need not obey this protocol and can send arbitrary messages to the master; in particular, they may have complete knowledge of the system and learning algorithms, and can collude with each other. We also clarify that there is no collusion
between the server and the clients. That is, the server cannot disguise itself as
a client or hire clients to launch colluded attacks.

\noindent \textbf{Defense objective.}~We would like to design some protocol which achieves a statistically optimal rate in the presence of malicious clients with an appropriate breakdown point, communication complexity and computational complexity. For a strongly convex and smooth loss function $F$, we would like to design some protocol which outputs some parameter $w^T$ at round $T$ such that with high probability,
 \begin{align*}
     \lim_{T\rightarrow \infty}\|\vecw^T - \vecw^\star\|\leq \Delta(\epsilon, n, m, d),
 \end{align*}
 where $\Delta(\epsilon, n, m, d) = \bigo(\sqrt{\frac{\epsilon}{n}+\frac{d}{mn}})$ is the optimal statistical rate achievable in this case, according to the lower bound in Theorem~\ref{thm:lower_bound}.
 
 On the other hand, we define the breakdown point for the estimation problem as 
 \begin{align*}
     b^\star = \inf\{\epsilon \mid \lim_{T\rightarrow \infty}\mathbb{E}[\|\vecw^T - \vecw^\star\|] = \infty\}.
 \end{align*}
From the definition, we know that for any $\epsilon_0$ with $\Delta(\epsilon_0, n, m, d)<\infty$, one has $b^\star>\epsilon_0$.
 

\section{Byzantine-Robust Federated Learning Protocols}

In this section, we first  present the meta-protocol that we will use throughout this paper.  After that, we introduce existing robust estimators as a building block for the protocol, and provide theoretical guarantees for the special case of the protocol without bucketing. Then we show how bucketing helps improve the rate to exactly optimal.  
At the end of this section, we illustrate how one can naturally combine the protocols with security guarantees when the server is semi-honest.

\begin{algorithm}[!t]
  \caption{\textsc{Byzantine-Robust Federated Learning with Bucketing}}\label{alg:robust-gd}
  \begin{algorithmic}
    \REQUIRE Initialize parameter vector $\vecw^0\in\W$, step size $\eta_t$, local model update interval $H$, number of disjoint buckets $k$, and total iteration $ T $.
 \STATE  \textit{\underline{Master machine}}: send $\vecw^0$ to all the worker machines.
   \FOR{$t=1,2,\ldots, T-1$}
   \FOR{$i\in[m]$}
   \STATE \textit{\underline{Worker machine $i$}}: 
 update local parameter via $\vecw_i^{t} = \vecw_i^{t-1} - \eta_t \gradF_i(\vecw_i^{t-1})$. 
   \IF{$t${ mod }$H = 0$} send local model update $\vecg_i(\vecw_i^t) = \vecw_i^{t} - \vecw_i^{t-H}$ to master machine.
   \ENDIF
   \ENDFOR
   \IF{$t${ mod }$H = 0$}
   \STATE \textit{\underline{Master machine}}:  randomly bucket the $m$ gradients into $k$ disjoint buckets of equal size and compute their empirical means $\vecz^1, \cdots, \vecz^k$. set $\mathcal{D}_m = \{\vecz^j\}_{j=1}^k$, compute 
   \begin{align*} 
   \vecg(\vecw^t) \leftarrow \mathsf{RobustEstimation Subroutine}(\mathcal{D}_m). 
   \end{align*}
   \STATE update model parameter $\vecw^{t} \leftarrow \Pi_\W(\vecw^{t-H} +  \vecg(\vecw^t))$, send and update $\vecw^{t}$ to all other worker machines by setting $\vecw_i^{t} = \vecw^t, i\in[m]$.
   \ENDIF
   \ENDFOR
  \end{algorithmic}
\end{algorithm}

\subsection{Description of the main FL Protocol}\label{sec:gradient-descent}
Our proposed FL protocol is quite straightforward: in each round, the honest worker machines compute the gradients of their local loss functions and then perform local update of their models. The Byzantine machines may update their parameters in an arbitrary way. 
In every $H$ parallel iteration of the algorithms, all the worker machines send the current updates to the master machine. The master machine, after receiving the updates from all workers, first divides the updates into $k$ different buckets, computes the mean of each bucket, and then aggregates the updates  according to some robust estimation subroutine. At the end, the master machine  broadcasts the updated  parameter to all worker machines.
We propose to apply the three robust estimators as aggregation subroutines for the FL protocol.  The resulting FL  protocols are presented formally in Algorithm~\ref{alg:robust-gd}.

The $\mathsf{RobustEstimationSubroutine}$  can be replaced with any robust mean estimator. In this paper, 
we will rely on three recently proposed  robust estimators as building blocks, namely \textsc{No-regret} algorithm~\cite{hopkins2020robust, zhu2021robust}, \textsc{Filtering} algorithm~\cite{ diakonikolas2017being, li2018principled, steinhardt2018resilience,  zhu2021robust} and a \textsc{GAN}-based algorithm~\cite{gao2018robust,gao2020generative,zhu2022robust}. It has been shown that each of these algorithms are efficient with near-optimal performance guarantee.  Due to the limit of space, we present the three algorithms and discuss their guarantee for  robust mean estimation  in Appendix~\ref{app:robust_guarantee}. To distinguish between different choices of parameters for the algorithms, we list the five choices of algorithms we will use as $\mathsf{RobustEstimationSubroutine}$ here:

\begin{resizealign}
    \mu_1 &= \textsc{No-regret}\left(\mathcal{D}_m, \epsilon,  \frac{\eta_t^2\sigma^2}{n},  \Big(\frac{2\eta+7}{3(1 - (3+\eta)\epsilon)}\Big)^2\cdot (1+\frac{d\log(d/\delta)}{m\epsilon})\cdot \frac{\eta_t^2\sigma^2}{n}, \eta\right)\label{eq:robust_subroutine1} \\
    \mu_2 &= \textsc{Filtering}\left(\mathcal{D}_m,  \frac{2(1-\epsilon)}{(1-2\epsilon)^2} \cdot(1+\frac{d\log(d/\delta)}{m\epsilon})\cdot\frac{\eta_t^2\sigma^2}{n}\right) \label{eq:robust_subroutine2}\\
\mu_3 &= \textsc{GAN}\left(\mathcal{D}_m, \frac{\eta_t^2\sigma^2}{n}\right) \label{eq:robust_subroutine3}\\
    \mu_4 &= \textsc{No-regret}\left(\mathcal{D}_m, \frac{\eta_t^2\sigma^2}{n}, \eta, \frac{C_1}{(1-C_2(\epsilon+\frac{\log(1/\delta)}{n})^2)}\cdot\left(1+\frac{d\log(d)+\log(1/\delta)}{m\epsilon}\right)\cdot\frac{\eta_t^2\sigma^2}{n}\right) \label{eq:robust_subroutine4}\\
    \mu_5 &= \textsc{Filtering}\left(\mathcal{D}_m,  \frac{C_3}{(1-C_4(\epsilon+\frac{\log(1/\delta)}{n})^2)} \cdot\left(1+\frac{d\log(d)+\log(1/\delta)}{m\epsilon}\right)\cdot\frac{\eta_t^2\sigma^2}{n}\right),   \label{eq:robust_subroutine5}
\end{resizealign}

where the definition of $\textsc{Filtering}, \textsc{No}$-$\textsc{regret}$ and $\textsc{GAN}$ can be found in Appendix~\ref{app:robust_guarantee}.
\subsection{Theoretical Analysis for Byzantine-Robust Federated Learning Protocol}\label{sec:med-convergence}

We begin with the most basic case of Algorithm~\ref{alg:robust-gd}, where  no bucketing is used and the master machine communicates with the worker machines in each round, i.e., $k=m, H=1$. 
Throughout the paper,  we make the following assumptions about the parameter space and loss function. First, we assume that the parameter space is convex and compact.
\begin{asm}[Convex and compact parameter space]\label{asm:parameter}
The parameter space $\mathcal{W}$ is convex and compact with diameter $D$, i.e, $\forall \vecw, \vecw'\in\mathcal{W}$, one has $\|\vecw-\vecw'\|_2\leq D$.
\end{asm}
Second, we assume that each loss function $f(\vecw;\vecz)$ is $L$-smooth, which also implies that the population risk function is $L$-smooth. 
\begin{asm}[Smoothness of $f$ and $F$]\label{asm:smoothness}
	For any $\vecz\in\Z$, $f(\cdot;\vecz)$ is $L$-smooth: $\|\nabla f(\vecw_1; \vecz) - \nabla f(\vecw_2;\vecz)\|\leq L\|\vecw_1-\vecw_2\|$. 
\end{asm}
Third, we impose the assumption that the  distribution of the gradient of the loss function $f(\cdot; \vecz)$ has bounded covariance:
\begin{asm}[Bounded covariance of gradient]\label{asm:bounded-variance} For any $\vecw\in\W$, $\|\Cov_{p^\star}(\gradf(\vecw;\vecz))\|_2 \le \sigma^2$.
\end{asm}

Note that this assumption considers the spectral norm of the covariance matrix. The corresponding upper bound $\sigma^2$  is usually dimension-free and is different from the bounded variance assumption in~\citet{yin2018byzantine}, which assumes that the trace of covariance $\Tr(\Cov_{p^\star}(\gradf(\vecw;\vecz))) = \mathbb{E}_{p^\star}[\|\gradf(\vecw;\vecz)) - \mathbb{E}_{p^\star}[\gradf(\vecw;\vecz))]\|_2^2]$ is bounded by $\tilde \sigma^2$. \revision{As a concrete example, when the distribution of the gradient is standard Gaussian with covariance $I$. We have $\sigma^2 = 1$ while $\tilde \sigma^2 = d$. In the worst case, when Assumption~\ref{asm:bounded-variance} is satisfied, one has $\Tr(\Cov_{p^\star}(\gradf(\vecw;\vecz))) \leq d\sigma^2$, which introduces an extra dimension factor. We provide more discussions  in Appendix~\ref{sec:related_work}.
}

In this paper, we consider three different categories of population risk function $F$, namely strongly convex, non-strongly convex and smooth non-convex functions. We also include the extra assumptions that are necessary for analysis in each case.
\begin{asm}\label{asm:convexity}
Consider the following three options of assumptions on loss function $F$.
\begin{enumerate}[(a)]
    \vspace{-5pt}
    \item (Strongly convex) $F(\cdot)$ is $\lambda$-strongly convex. The minimizer of $F(\cdot)$ in $\W$, $\vecw^\star$, is also the minimizer of $F(\cdot)$ in $\R^d$, i.e., $\gradF(\vecw^*)=0$.
    \vspace{-5pt}
    \item (Convex) $F(\cdot)$ is  convex. The minimizer of $F(\cdot)$ in $\W$, $\vecw^\star$, is also the minimizer of $F(\cdot)$ in $\R^d$, i.e., $\gradF(\vecw^*)=0$. The parameter space $\W$ contains the following $f_2$ ball centered at $\vecw^*$: $\{\vecw\in\R^d : \twonms{\vecw-\vecw^*} \le 2\twonms{\vecw^0 - \vecw^*} \}$.
    \vspace{-5pt}
    \item (Non-convex) Suppose that $\forall~\vecw\in\W$, $\twonms{\gradF(\vecw)}\le M$. We assume that $\W$ contains the $f_2$ ball $
\{\vecw\in\R^d : \twonms{\vecw - \vecw^0} \le \frac{2}{\Delta^2}(M+\Delta)(F(\vecw^0) - F(\vecw^*))\} $, where $\Delta$ is defined later as in Equation (\ref{eq:def-error-delta1}), (\ref{eq:def-error-delta2}), (\ref{eq:def-error-delta3}), (\ref{eq:def-error-delta4}) or (\ref{eq:def-error-delta5}), depending on the subroutine used.
\vspace{-5pt}
\end{enumerate}
\end{asm}

Note that we do not make any explicit assumption on the  
convexity of the individual loss functions $f(\cdot;\vecz)$. We also remark here that for the simplicity of analysis, we focus on the case of $H=1$ in the main text, which means that the worker machine communicates gradient information with the master machine in each round. We show in Appendix~\ref{app:multi-round} how to extend the analysis to multi-round local model updates and present the corresponding theorem there.
Our main technical results on the FL protocol are stated as below.

\begin{theorem}\label{thm:main-gd-sc}
Let $\vecw_i^t, i\in\{1, 2, 3,4, 5\}$  denote the output of Algorithm~\ref{alg:robust-gd} with $H=1, k=m$ and step-size $\eta_t = 1/L$ at round $t$ when the $\mathsf{RobustEstimationSubroutine}$ takes the algorithm in~\eqref{eq:robust_subroutine1},~(\ref{eq:robust_subroutine2}),~(\ref{eq:robust_subroutine3}),~(\ref{eq:robust_subroutine4}),~(\ref{eq:robust_subroutine5}), respectively.    
Under Assumption~\ref{asm:parameter},~\ref{asm:smoothness},~\ref{asm:bounded-variance} and~\ref{asm:convexity}(a), with probability at least $1-\delta$, 
$$
\twonms{\vecw^T_i - \vecw^*} \le (1-\frac{\lambda}{L + \lambda})^T\twonms{ \vecw^0 - \vecw^* } + \frac{2}{\lambda}\Delta_i,
$$
where for some universal  constants $C_i$,

\begin{resizealign}\label{eq:def-error-delta1}
\Delta_1 &:= \bigo\left(\frac{\sigma}{(1-3\epsilon)\sqrt{n}}\cdot \left({\sqrt    \epsilon} + \sqrt{\frac{d^2\log(1+nmDL)+d\log(d/\delta)}{m}}\right)\right), \\
\label{eq:def-error-delta2}\Delta_2 &:= \bigo\left(\frac{\sigma}{(1-2\epsilon)\sqrt{n}}\cdot \left({\sqrt    \epsilon} + \sqrt{\frac{d^2\log(1+nmDL)+d\log(d/\delta)}{m}}\right)\right), \\
\label{eq:def-error-delta3}\Delta_3 &:= \bigo\left(\frac{\sigma}{({1-C_1\epsilon})\sqrt{n}}\cdot \left({\sqrt    {\epsilon+\sqrt{\frac{d\log(1+nmDL)+\log(1/\delta)}{m}}}}\right)\right), \\
\label{eq:def-error-delta4}\Delta_4 &:=\bigo\left(\frac{\sigma}{1-C_2(\epsilon+\log(1/\delta)/m)}\cdot \left({\sqrt{\epsilon}}+ \sqrt{\frac{d(\log(d)+\log(1+nmDL))+\log(1/\delta)}{m}}\right)\right), \\
\label{eq:def-error-delta5}\Delta_5 &:=\bigo\left(\frac{\sigma}{1-C_3(\epsilon+\log(1/\delta)/m)}\cdot \left({\sqrt{\epsilon}}+ \sqrt{\frac{d(\log(d)+\log(1+nmDL))+\log(1/\delta)}{m}}\right)\right),
\end{resizealign}

Under Assumption~\ref{asm:parameter},~\ref{asm:smoothness},~\ref{asm:bounded-variance} and~\ref{asm:convexity}(b), with probability at least $1-\delta$, 
after $T = \frac{L}{\Delta_i}\twonms{\vecw^0 - \vecw^*}$ parallel iterations, we have
$$
F(\vecw_i^T) - F(\vecw^*) \le 16 \twonms{\vecw^0 - \vecw^*} \Delta_i \Big(1 + \frac{1}{2L}\Delta_i \Big).
$$
Under Assumption~\ref{asm:parameter},~\ref{asm:smoothness},~\ref{asm:bounded-variance} and~\ref{asm:convexity}(c), with probability at least $1-\delta$, we have
$$
\min_{t=0,1,\ldots, T} \twonms{\gradF(\vecw^t_i)}^2  \le \frac{2L}{T} (F(\vecw^0) - F(\vecw^*)) + \Delta_i^2.
$$
\end{theorem}

We prove Theorem~\ref{thm:main-gd-sc} in Appendix~\ref{prf:main-gd-sc}. We also provide the analysis for  the case of multi-round updates ($H>1$) in Appendix~\ref{app:multi-round}. We can see from the statement that the conclusion takes a similar form as that in~\citet{yin2018byzantine} except that our results come with a tighter statistical error bound under weaker distributional assumptions. Under the same assumption, the best rate achieved by~\citet{yin2018byzantine} as $T\rightarrow \infty$  is    $\widetilde{\bigo}(\sqrt{\frac{\epsilon d}{n}+\frac{d^2}{mn}})$, while our algorithm improves  to ${\bigo}(\sqrt{\frac{\epsilon}{n}+\frac{d\log(d)}{mn}})$. For convex losses, we show that the excess risk converges at a rate of $\Delta_i$. For the non-convex case, we show that it approaches the critical point with a rate of $\Delta_i$.
In Section~\ref{sec:lower_bound}, we provide a lower
bound showing that the worst-case statistical rate of parameter estimation error under strongly convex loss is
at least $\Omega(\sqrt{\frac{\epsilon }{n}+\frac{d}{mn}})$, matching our upper bound for the strongly convex case. 

For the five different estimators, we can see that the first two estimators based on \textsc{No-regret} and \textsc{Filtering}  in~\eqref{eq:robust_subroutine1} and (\ref{eq:robust_subroutine2}) have high breakdown point $1/3$ and $1/2$, respectively. But the rate is sub-optimal in terms of the dependence in $d$. Technically, this is from the covering argument when applied to $\sqrt{\frac{d\log(1/\delta)}{mn}}$ instead of $\sqrt{\frac{d+\log(1/\delta)}{mn}}$. This shows the importance of achieving sub-Gaussian rate in robust estimation. The GAN estimator in~\eqref{eq:robust_subroutine3} is usually easier to compute in high dimension. However, the rate is sub-optimal due to the $(d/m)^{1/4}$ factors. We can improve the rate of \textsc{No-regret} and \textsc{Filtering}  by taking different parameters as in~\eqref{eq:robust_subroutine4} and (\ref{eq:robust_subroutine5}). However, the  breakdown point is not as high.  It is an open problem whether one can design estimators that are statistically optimal, computationally efficient, and have optimal breakdown point. 

Computationally, the proposed estimators are inferior to coordinate-wise median or trimmed mean due to the query of second moment. On one hand, one can improve the implementation by combining the no-regret and filtering algorithm with covering / packing SDP solver or quantum entropy scoring to improve the computational complexity to linear, see e.g.~\citep{cheng2019high, dong2019quantum}. Furthermore, one can also split the data dimension into $K$ groups, thus increasing the sample complexity by a $\sqrt{K}$ factor while reducing the computational complexity significantly. We omit the details here since the paper mainly focuses on the statistical part.  

\subsection{Improved Protocol  via Bucketing}

We can improve the rate and remove the extra $\log(d)$ factor in Theorem~\ref{thm:main-gd-sc} by setting a nontrivial bucket size   $k$  in Algorithm~\ref{alg:robust-gd}. The idea of bucketing was first introduced in the median-of-means procedure~\cite{lecue2020robust}. \citet{diakonikolas2020outlier} showed that the bucketing procedure, combined with \textsc{Filtering} or \textsc{No-regret} algorithm, can achieve a sub-Gaussian rate, which is mini-max rate-optimal with respect to all parameters. The concrete subroutine in Algorithm~\ref{alg:robust-gd} is chosen to be one of the following based on algorithms in Appendix~\ref{subsec:alg_description}:

\begin{resizealign}
     \mu_1 &= \textsc{No-regret}\left(\mathcal{D}_m, 0.1, \frac{k\sigma^2}{mn}, \bigo\left(\frac{(d+k)\eta_t^2\sigma^2}{mn}\right),0.1\right), \label{eq:robust_subroutine6}\\
    \mu_2 &= \textsc{Filtering}\left(\mathcal{D}_m,  \bigo\left(\frac{\eta_t^2(d+k)\sigma^2}{mn}\right)\right)  . \label{eq:robust_subroutine7}
\end{resizealign}

With the bucketing step, one can improve the statistical bound to near-optimal, as stated in the following theorem.
\begin{theorem}\label{thm:main-gd-improved}
Let $\vecw^t$  denote the output of Algorithm~\ref{alg:robust-gd} with step-size $\eta = 1/L$, $H=1$ and \revision{$k=\floor{2\epsilon m + \log(1/\delta)}$}.    Under Assumption~\ref{asm:parameter},~\ref{asm:smoothness},~\ref{asm:bounded-variance} and~\ref{asm:convexity}(a), with probability at least $1-\delta$, we have 
$$
\twonms{\vecw^T - \vecw^*} \le (1-\frac{\lambda}{L + \lambda})^T\twonms{ \vecw^0 - \vecw^* } + \frac{2}{\lambda}\Delta,
$$
where for some universal constant $C$,

\begin{resizealign}\label{eq:def-error-delta-improved}
\Delta &:= \bigo\left(\frac{\sigma}{(1-C\epsilon)\sqrt{n}}\cdot \left({\sqrt    \epsilon} + \sqrt{\frac{d\log(1+nmDL)+\log(1/\delta)}{m}}\right)\right)
\end{resizealign}

Under Assumption~\ref{asm:parameter},~\ref{asm:smoothness},~\ref{asm:bounded-variance} and~\ref{asm:convexity}(b), with probability at least $1-\delta$, 
after $T = \frac{L}{\Delta}\twonms{\vecw^0 - \vecw^*}$  iterations,
$$
F(\vecw_i^T) - F(\vecw^*) \le 16 \twonms{\vecw^0 - \vecw^*} \Delta \Big(1 + \frac{1}{2L}\Delta \Big).
$$
Under Assumption~\ref{asm:parameter},~\ref{asm:smoothness},~\ref{asm:bounded-variance} and~\ref{asm:convexity}(c), with probability at least $1-\delta$, we have
$$
\min_{t=0,1,\ldots, T} \twonms{\gradF(\vecw^t)}^2  \le \frac{2L}{T} (F(\vecw^0) - F(\vecw^*)) + \Delta^2.
$$
\end{theorem}

The proof is deferred to Appendix~\ref{app:proof_improved}. Compared to Theorem~\ref{thm:main-gd-sc}, the extra $\log(d)$ factor is removed and the rate exactly matches the lower bound in Theorem~\ref{thm:lower_bound} when $nmDL=O(1)$. A similar idea of bucketing also appears in~\citet{karimireddy2020byzantine},  motivated from heterogeneous clients. Our results establish the benefit of the bucketing even when the clients are homogeneous. 

\section{Evaluation}
\label{sec:eval}


\tikzset{font={\fontsize{15pt}{12}\selectfont}}
\captionsetup[sub]{skip=2pt}
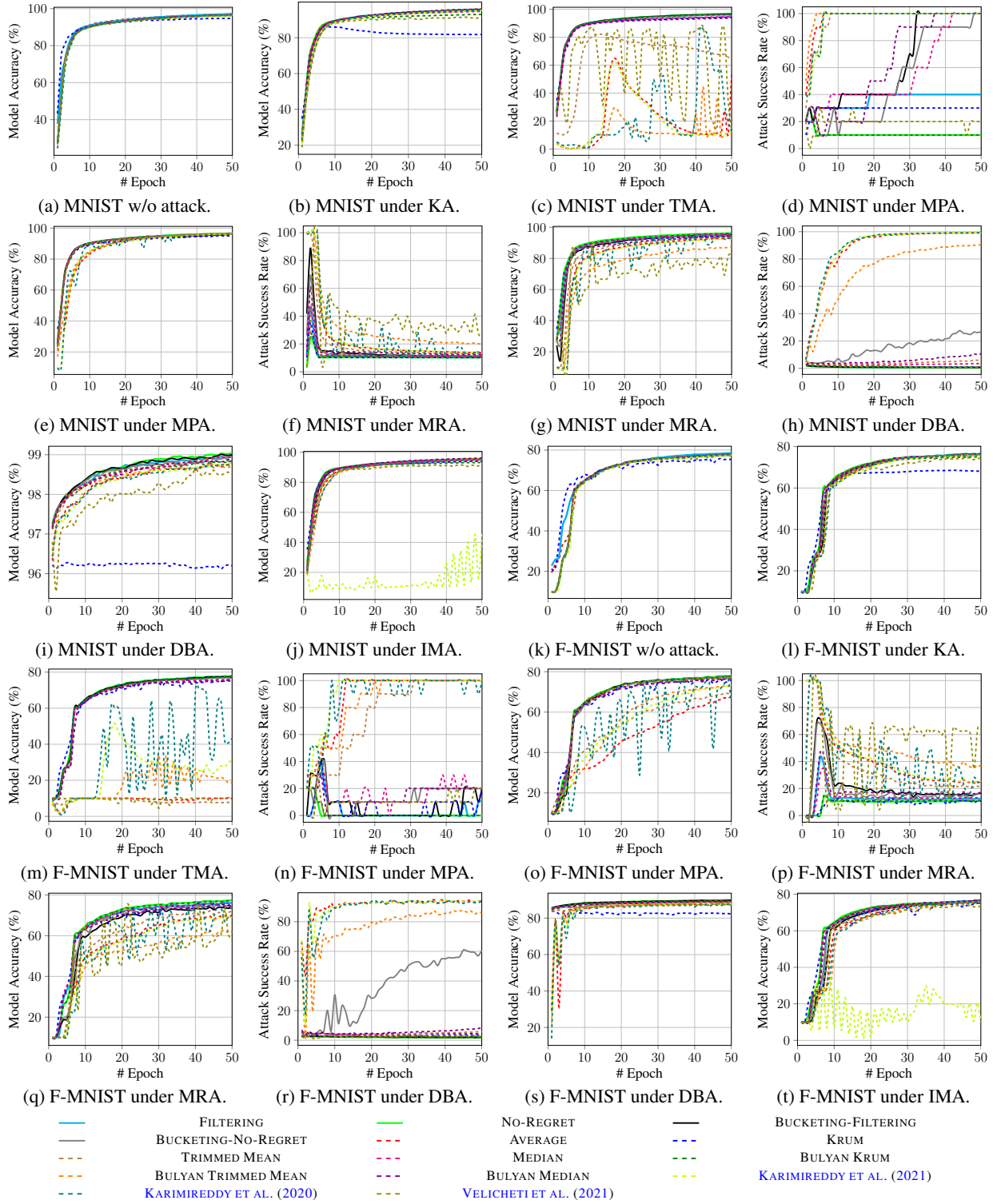
\begin{figure*}[!htbp]
     \centering
     \hfill
     \begin{subfigure}[b]{\figwidth}
         \centering
         \resizebox{\textwidth}{!}{\begin{tikzpicture}[spy using outlines={rectangle, magnification=2, width=6cm, height=2cm, connect spies}]
\begin{axis}[
  grid=major,
  xmin=0, xmax=50,
  ytick align=outside, ytick pos=left,
  xtick align=outside, xtick pos=left,
  xlabel=\# Epoch,
  ylabel={Model Accuracy (\%)},
  enlarge y limits=0.05
  ]

\addplot+[
  red, dashed, mark=none, line width=1.6pt, 
  smooth, 
  error bars/.cd, 
    y fixed,
    y dir=both, 
    y explicit
] table [x expr=\thisrow{x}, y=y, col sep=comma] {data/main/MNIST/noattack_average.txt};
\addlegendentry{Average}

\addplot+[
  cyan, mark=none, line width=2pt, 
  smooth, 
  error bars/.cd, 
    y fixed,
    y dir=both, 
    y explicit
] table [x expr=\thisrow{x}, y=y, col sep=comma] {data/main/MNIST/noattack_filterl2.txt};
\addlegendentry{filterL2}

\addplot+[
  green, mark=none, line width=2pt, 
  smooth, 
  error bars/.cd, 
    y fixed,
    y dir=both, 
    y explicit
] table [x expr=\thisrow{x}, y=y, col sep=comma] {data/main/MNIST/noattack_ex_noregret.txt};
\addlegendentry{Explicit\_No\_Regret}

\addplot+[
  black, mark=none, line width=1.6pt, 
  smooth, 
  error bars/.cd, 
    y fixed,
    y dir=both, 
    y explicit
] table [x expr=\thisrow{x}, y=y, col sep=comma] {data/main/MNIST/noattack_mom_filterl2.txt};

\addplot+[
  gray, mark=none, line width=1.6pt, 
  smooth, 
  error bars/.cd, 
    y fixed,
    y dir=both, 
    y explicit
] table [x expr=\thisrow{x}, y=y, col sep=comma] {data/main/MNIST/noattack_mom_ex_noregret.txt};

\addplot+[
  blue, dashed,  mark=none, line width=1.6pt, 
  smooth, 
  error bars/.cd, 
    y fixed,
    y dir=both, 
    y explicit
] table [x expr=\thisrow{x}, y=y, col sep=comma] {data/main/MNIST/noattack_krum.txt};
\addlegendentry{Krum}

\addplot+[
  brown, dashed, mark=none, line width=1.6pt, 
  smooth, 
  error bars/.cd, 
    y fixed,
    y dir=both, 
    y explicit
] table [x expr=\thisrow{x}, y=y, col sep=comma] {data/main/MNIST/noattack_trimmedmean.txt};
\addlegendentry{Trimmed Mean}

\addplot+[
  darkgreen, dashed, mark=none, line width=1.6pt, 
  smooth, 
  error bars/.cd, 
    y fixed,
    y dir=both, 
    y explicit
] table [x expr=\thisrow{x}, y=y, col sep=comma] {data/main/MNIST/noattack_bulyankrum.txt};
\addlegendentry{Bulyan\_Krum}

\addplot+[
  orange, dashed, mark=none, line width=1.6pt, 
  smooth, 
  error bars/.cd, 
    y fixed,
    y dir=both, 
    y explicit
] table [x expr=\thisrow{x}, y=y, col sep=comma] {data/main/MNIST/noattack_bulyantrimmedmean.txt};
\addlegendentry{Bulyan\_Trimmed\_Mean}

\addplot+[
  magenta, dashed, mark=none, line width=1.6pt, 
  smooth, 
  error bars/.cd, 
    y fixed,
    y dir=both, 
    y explicit
] table [x expr=\thisrow{x}, y=y, col sep=comma] {data/main/MNIST/noattack_median.txt};
\addlegendentry{Median}

\addplot+[
  violet, dashed, mark=none, line width=1.6pt, 
  smooth, 
  error bars/.cd, 
    y fixed,
    y dir=both, 
    y explicit
] table [x expr=\thisrow{x}, y=y, col sep=comma] {data/main/MNIST/noattack_bulyanmedian.txt};
\addlegendentry{Bulyan\_Median}

\addplot+[
  lime, dashed, mark=none, line width=1.6pt, 
  smooth, 
  error bars/.cd, 
    y fixed,
    y dir=both, 
    y explicit
] table [x expr=\thisrow{x}, y=y, col sep=comma] {data/main/MNIST/noattack_history.txt};
\addlegendentry{\cite{karimireddy2021learning}}

\addplot+[
  teal, dashed, mark=none, line width=1.6pt, 
  smooth, 
  error bars/.cd, 
    y fixed,
    y dir=both, 
    y explicit
] table [x expr=\thisrow{x}, y=y, col sep=comma] {data/main/MNIST/noattack_bucketing.txt};
\addlegendentry{\cite{karimireddy2020byzantine}}

\addplot+[
  olive, dashed, mark=none, line width=1.6pt, 
  smooth, 
  error bars/.cd, 
    y fixed,
    y dir=both, 
    y explicit
] table [x expr=\thisrow{x}, y=y, col sep=comma] {data/main/MNIST/noattack_clustering.txt};

\legend{};


\end{axis}
\end{tikzpicture}}
         \caption{MNIST w/o attack.}
         \label{fig:mnist_noattack}
     \end{subfigure}
     \hfill
     \begin{subfigure}[b]{\figwidth}
         \centering
         \resizebox{\textwidth}{!}{\begin{tikzpicture}
\begin{axis}[
  set layers,
  grid=major,
  xmin=0, xmax=50,
  ytick align=outside, ytick pos=left,
  xtick align=outside, xtick pos=left,
  xlabel=\# Epoch,
  ylabel={Model Accuracy (\%)},
  enlarge y limits=0.05
  ]

\addplot+[
  red, dashed, mark=none, line width=1.6pt, 
  smooth, 
  error bars/.cd, 
    y fixed,
    y dir=both, 
    y explicit
] table [x expr=\thisrow{x}, y=y, col sep=comma] {data/main/MNIST/krum_average.txt};
\addlegendentry{Average}

\addplot+[
  cyan, mark=none, line width=2pt, 
  smooth, 
  error bars/.cd, 
    y fixed,
    y dir=both, 
    y explicit
] table [x expr=\thisrow{x}, y=y, col sep=comma] {data/main/MNIST/krum_filterl2.txt};
\addlegendentry{filterL2}

\addplot+[
  green, mark=none, line width=2pt, 
  smooth, 
  error bars/.cd, 
    y fixed,
    y dir=both, 
    y explicit
] table [x expr=\thisrow{x}, y=y, col sep=comma] {data/main/MNIST/krum_ex_noregret.txt};
\addlegendentry{Explicit\_No\_Regret}

\addplot+[
  black, mark=none, line width=1.6pt, 
  smooth, 
  error bars/.cd, 
    y fixed,
    y dir=both, 
    y explicit
] table [x expr=\thisrow{x}, y=y, col sep=comma] {data/main/MNIST/krum_mom_filterl2.txt};

\addplot+[
  gray, mark=none, line width=1.6pt, 
  smooth, 
  error bars/.cd, 
    y fixed,
    y dir=both, 
    y explicit
] table [x expr=\thisrow{x}, y=y, col sep=comma] {data/main/MNIST/krum_mom_ex_noregret.txt};

\addplot+[
  blue, dashed,  mark=none, line width=1.6pt, 
  smooth, 
  error bars/.cd, 
    y fixed,
    y dir=both, 
    y explicit
] table [x expr=\thisrow{x}, y=y, col sep=comma] {data/main/MNIST/krum_krum.txt};
\addlegendentry{Krum}

\addplot+[
  brown, dashed, mark=none, line width=1.6pt, 
  smooth, 
  error bars/.cd, 
    y fixed,
    y dir=both, 
    y explicit
] table [x expr=\thisrow{x}, y=y, col sep=comma] {data/main/MNIST/krum_trimmedmean.txt};
\addlegendentry{Trimmed Mean}

\addplot+[
  darkgreen, dashed, mark=none, line width=1.6pt, 
  smooth, 
  error bars/.cd, 
    y fixed,
    y dir=both, 
    y explicit
] table [x expr=\thisrow{x}, y=y, col sep=comma] {data/main/MNIST/krum_bulyankrum.txt};
\addlegendentry{Bulyan\_Krum}

\addplot+[
  orange, dashed, mark=none, line width=1.6pt, 
  smooth, 
  error bars/.cd, 
    y fixed,
    y dir=both, 
    y explicit
] table [x expr=\thisrow{x}, y=y, col sep=comma] {data/main/MNIST/krum_bulyantrimmedmean.txt};
\addlegendentry{Bulyan\_Trimmed\_Mean}

\addplot+[
  magenta, dashed, mark=none, line width=1.6pt, 
  smooth, 
  error bars/.cd, 
    y fixed,
    y dir=both, 
    y explicit
] table [x expr=\thisrow{x}, y=y, col sep=comma] {data/main/MNIST/krum_median.txt};
\addlegendentry{Median}

\addplot+[
  violet, dashed, mark=none, line width=1.6pt, 
  smooth, 
  error bars/.cd, 
    y fixed,
    y dir=both, 
    y explicit
] table [x expr=\thisrow{x}, y=y, col sep=comma] {data/main/MNIST/krum_bulyanmedian.txt};
\addlegendentry{Bulyan\_Median}

\addplot+[
  lime, dashed, mark=none, line width=1.6pt, 
  smooth, 
  error bars/.cd, 
    y fixed,
    y dir=both, 
    y explicit
] table [x expr=\thisrow{x}, y=y, col sep=comma] {data/main/MNIST/krum_history.txt};
\addlegendentry{\cite{karimireddy2021learning}}

\addplot+[
  teal, dashed, mark=none, line width=1.6pt, 
  smooth, 
  error bars/.cd, 
    y fixed,
    y dir=both, 
    y explicit
] table [x expr=\thisrow{x}, y=y, col sep=comma] {data/main/MNIST/krum_bucketing.txt};
\addlegendentry{\cite{karimireddy2020byzantine}}

\addplot+[
  olive, dashed, mark=none, line width=1.6pt, 
  smooth, 
  error bars/.cd, 
    y fixed,
    y dir=both, 
    y explicit
] table [x expr=\thisrow{x}, y=y, col sep=comma] {data/main/MNIST/krum_clustering.txt};

\legend{};

\end{axis}
\end{tikzpicture}}
         \caption{MNIST under KA.}
         \label{fig:mnist_krum}
     \end{subfigure}
     \hfill
     \begin{subfigure}[b]{\figwidth}
         \centering
         \resizebox{\textwidth}{!}{\begin{tikzpicture}
\begin{axis}[
  set layers,
  grid=major,
  xmin=0, xmax=50,
  ytick align=outside, ytick pos=left,
  xtick align=outside, xtick pos=left,
  xlabel=\# Epoch,
  ylabel={Model Accuracy (\%)},
  legend pos=south east,
  enlarge y limits=0.05
  ]

\addplot+[
  red, dashed, mark=none, line width=1.6pt, 
  smooth, 
  error bars/.cd, 
    y fixed,
    y dir=both, 
    y explicit
] table [x expr=\thisrow{x}, y=y, col sep=comma] {data/main/MNIST/trimmedmean_average.txt};
\addlegendentry{Average}

\addplot+[
  cyan, mark=none, line width=2pt, 
  smooth, 
  error bars/.cd, 
    y fixed,
    y dir=both, 
    y explicit
] table [x expr=\thisrow{x}, y=y, col sep=comma] {data/main/MNIST/trimmedmean_filterl2.txt};
\addlegendentry{filterL2}

\addplot+[
  green, mark=none, line width=2pt, 
  smooth, 
  error bars/.cd, 
    y fixed,
    y dir=both, 
    y explicit
] table [x expr=\thisrow{x}, y=y, col sep=comma] {data/main/MNIST/trimmedmean_ex_noregret.txt};
\addlegendentry{Explicit\_No\_Regret}

\addplot+[
  black, mark=none, line width=1.6pt, 
  smooth, 
  error bars/.cd, 
    y fixed,
    y dir=both, 
    y explicit
] table [x expr=\thisrow{x}, y=y, col sep=comma] {data/main/MNIST/trimmedmean_mom_filterl2.txt};

\addplot+[
  gray, mark=none, line width=1.6pt, 
  smooth, 
  error bars/.cd, 
    y fixed,
    y dir=both, 
    y explicit
] table [x expr=\thisrow{x}, y=y, col sep=comma] {data/main/MNIST/trimmedmean_mom_ex_noregret.txt};

\addplot+[
  blue, dashed, mark=none, line width=1.6pt, 
  smooth, 
  error bars/.cd, 
    y fixed,
    y dir=both, 
    y explicit
] table [x expr=\thisrow{x}, y=y, col sep=comma] {data/main/MNIST/trimmedmean_krum.txt};
\addlegendentry{Krum}

\addplot+[
  brown, dashed, mark=none, line width=1.6pt, 
  smooth, 
  error bars/.cd, 
    y fixed,
    y dir=both, 
    y explicit
] table [x expr=\thisrow{x}, y=y, col sep=comma] {data/main/MNIST/trimmedmean_trimmedmean.txt};
\addlegendentry{Trimmed Mean}

\addplot+[
  darkgreen, dashed, mark=none, line width=1.6pt, 
  smooth, 
  error bars/.cd, 
    y fixed,
    y dir=both, 
    y explicit
] table [x expr=\thisrow{x}, y=y, col sep=comma] {data/main/MNIST/trimmedmean_bulyankrum.txt};
\addlegendentry{Bulyan\_Krum}

\addplot+[
  orange, dashed, mark=none, line width=1.6pt, 
  smooth, 
  error bars/.cd, 
    y fixed,
    y dir=both, 
    y explicit
] table [x expr=\thisrow{x}, y=y, col sep=comma] {data/main/MNIST/trimmedmean_bulyantrimmedmean.txt};
\addlegendentry{Bulyan\_Trimmed\_Mean}

\addplot+[
  magenta, dashed, mark=none, line width=1.6pt, 
  smooth, 
  error bars/.cd, 
    y fixed,
    y dir=both, 
    y explicit
] table [x expr=\thisrow{x}, y=y, col sep=comma] {data/main/MNIST/trimmedmean_median.txt};
\addlegendentry{Median}

\addplot+[
  violet, dashed, mark=none, line width=1.6pt, 
  smooth, 
  error bars/.cd, 
    y fixed,
    y dir=both, 
    y explicit
] table [x expr=\thisrow{x}, y=y, col sep=comma] {data/main/MNIST/trimmedmean_bulyanmedian.txt};
\addlegendentry{Bulyan\_Median}

\addplot+[
  lime, dashed, mark=none, line width=1.6pt, 
  smooth, 
  error bars/.cd, 
    y fixed,
    y dir=both, 
    y explicit
] table [x expr=\thisrow{x}, y=y, col sep=comma] {data/main/MNIST/trimmedmean_history.txt};
\addlegendentry{\cite{karimireddy2021learning}}

\addplot+[
  teal, dashed, mark=none, line width=1.6pt, 
  smooth, 
  error bars/.cd, 
    y fixed,
    y dir=both, 
    y explicit
] table [x expr=\thisrow{x}, y=y, col sep=comma] {data/main/MNIST/trimmedmean_bucketing.txt};
\addlegendentry{\cite{karimireddy2020byzantine}}

\addplot+[
  olive, dashed, mark=none, line width=1.6pt, 
  smooth, 
  error bars/.cd, 
    y fixed,
    y dir=both, 
    y explicit
] table [x expr=\thisrow{x}, y=y, col sep=comma] {data/main/MNIST/trimmedmean_clustering.txt};

\legend{};

\end{axis}
\end{tikzpicture}}
         \caption{MNIST under TMA.}
         \label{fig:mnist_trimmedmean}
     \end{subfigure}
     \hfill
     \begin{subfigure}[b]{\figwidth}
         \centering
         \resizebox{\textwidth}{!}{\begin{tikzpicture}
\begin{axis}[
  set layers,
  grid=major,
  xmin=0, xmax=50,
  ytick align=outside, ytick pos=left,
  xtick align=outside, xtick pos=left,
  xlabel=\# Epoch,
  ylabel={Attack Success Rate (\%)},
  legend pos=south east,
  enlarge y limits=0.05
]

\addplot+[
  red, dashed, mark=none, line width=1.6pt, 
  smooth, 
  error bars/.cd, 
    y fixed,
    y dir=both, 
    y explicit
] table [x expr=\thisrow{x}, y=asr, col sep=comma] {data/main/MNIST/modelpoisoning_average.txt};
\addlegendentry{Average}

\addplot+[
  cyan, mark=none, line width=2pt, 
  smooth, 
  error bars/.cd, 
    y fixed,
    y dir=both, 
    y explicit
] table [x expr=\thisrow{x}, y=asr, col sep=comma] {data/main/MNIST/modelpoisoning_filterl2.txt};
\addlegendentry{Krum Attack + filterL2}

\addplot+[
  green, mark=none, line width=2pt, 
  smooth, 
  error bars/.cd, 
    y fixed,
    y dir=both, 
    y explicit
] table [x expr=\thisrow{x}, y=asr, col sep=comma] {data/main/MNIST/modelpoisoning_ex_noregret.txt};
\addlegendentry{Explicit\_No\_Regret}

\addplot+[
  black, mark=none, line width=1.6pt, 
  smooth, 
  error bars/.cd, 
    y fixed,
    y dir=both, 
    y explicit
] table [x expr=\thisrow{x}, y=asr, col sep=comma] {data/main/MNIST/modelpoisoning_mom_filterl2.txt};

\addplot+[
  gray, mark=none, line width=1.6pt, 
  smooth, 
  error bars/.cd, 
    y fixed,
    y dir=both, 
    y explicit
] table [x expr=\thisrow{x}, y=asr, col sep=comma] {data/main/MNIST/modelpoisoning_mom_ex_noregret.txt};

\addplot+[
  blue, dashed,  mark=none, line width=1.6pt, 
  smooth, 
  error bars/.cd, 
    y fixed,
    y dir=both, 
    y explicit
] table [x expr=\thisrow{x}, y=asr, col sep=comma] {data/main/MNIST/modelpoisoning_krum.txt};
\addlegendentry{Krum}

\addplot+[
  brown, dashed, mark=none, line width=1.6pt, 
  smooth, 
  error bars/.cd, 
    y fixed,
    y dir=both, 
    y explicit
] table [x expr=\thisrow{x}, y=asr, col sep=comma] {data/main/MNIST/modelpoisoning_trimmedmean.txt};
\addlegendentry{Trimmed Mean}

\addplot+[
  darkgreen, dashed, mark=none, line width=1.6pt, 
  smooth, 
  error bars/.cd, 
    y fixed,
    y dir=both, 
    y explicit
] table [x expr=\thisrow{x}, y=asr, col sep=comma] {data/main/MNIST/modelpoisoning_bulyankrum.txt};
\addlegendentry{Bulyan\_Krum}

\addplot+[
  orange, dashed, mark=none, line width=1.6pt, 
  smooth, 
  error bars/.cd, 
    y fixed,
    y dir=both, 
    y explicit
] table [x expr=\thisrow{x}, y=asr, col sep=comma] {data/main/MNIST/modelpoisoning_bulyantrimmedmean.txt};
\addlegendentry{Bulyan\_Trimmed\_Mean}

\addplot+[
  magenta, dashed, mark=none, line width=1.6pt, 
  smooth, 
  error bars/.cd, 
    y fixed,
    y dir=both, 
    y explicit
] table [x expr=\thisrow{x}, y=asr, col sep=comma] {data/main/MNIST/modelpoisoning_median.txt};
\addlegendentry{Median}

\addplot+[
  violet, dashed, mark=none, line width=1.6pt, 
  smooth, 
  error bars/.cd, 
    y fixed,
    y dir=both, 
    y explicit
] table [x expr=\thisrow{x}, y=asr, col sep=comma] {data/main/MNIST/modelpoisoning_bulyanmedian.txt};
\addlegendentry{Bulyan\_Median}

\addplot+[
  lime, dashed, mark=none, line width=1.6pt, 
  smooth, 
  error bars/.cd, 
    y fixed,
    y dir=both, 
    y explicit
] table [x expr=\thisrow{x}, y=asr, col sep=comma] {data/main/MNIST/modelpoisoning_history.txt};
\addlegendentry{\cite{karimireddy2021learning}}

\addplot+[
  teal, dashed, mark=none, line width=1.6pt, 
  smooth, 
  error bars/.cd, 
    y fixed,
    y dir=both, 
    y explicit
] table [x expr=\thisrow{x}, y=asr, col sep=comma] {data/main/MNIST/modelpoisoning_bucketing.txt};
\addlegendentry{\cite{karimireddy2020byzantine}}

\addplot+[
  olive, dashed, mark=none, line width=1.6pt, 
  smooth, 
  error bars/.cd, 
    y fixed,
    y dir=both, 
    y explicit
] table [x expr=\thisrow{x}, y=asr, col sep=comma] {data/main/MNIST/modelpoisoning_clustering.txt};

\legend{};

\end{axis}
\end{tikzpicture}}
         \caption{MNIST under MPA.}
         \label{fig:mnist_modelpoisoning_asr}
     \end{subfigure}
     
     \hfill
     \begin{subfigure}[b]{\figwidth}
         \centering
         \resizebox{\textwidth}{!}{\begin{tikzpicture}
\begin{axis}[
  set layers,
  grid=major,
  xmin=0, xmax=50,
  ytick align=outside, ytick pos=left,
  xtick align=outside, xtick pos=left,
  xlabel=\# Epoch,
  ylabel={Model Accuracy (\%)},
  legend pos=south east,
  enlarge y limits=0.05
]

\addplot+[
  red, dashed, mark=none, line width=1.6pt, 
  smooth, 
  error bars/.cd, 
    y fixed,
    y dir=both, 
    y explicit
] table [x expr=\thisrow{x}, y=acc, col sep=comma] {data/main/MNIST/modelpoisoning_average.txt};
\addlegendentry{Average}

\addplot+[
  cyan, mark=none, line width=2pt, 
  smooth, 
  error bars/.cd, 
    y fixed,
    y dir=both, 
    y explicit
] table [x expr=\thisrow{x}, y=acc, col sep=comma] {data/main/MNIST/modelpoisoning_filterl2.txt};
\addlegendentry{Krum Attack + filterL2}

\addplot+[
  green, mark=none, line width=2pt, 
  smooth, 
  error bars/.cd, 
    y fixed,
    y dir=both, 
    y explicit
] table [x expr=\thisrow{x}, y=acc, col sep=comma] {data/main/MNIST/modelpoisoning_ex_noregret.txt};
\addlegendentry{Explicit\_No\_Regret}

\addplot+[
  black, mark=none, line width=1.6pt, 
  smooth, 
  error bars/.cd, 
    y fixed,
    y dir=both, 
    y explicit
] table [x expr=\thisrow{x}, y=acc, col sep=comma] {data/main/MNIST/modelpoisoning_mom_filterl2.txt};

\addplot+[
  gray, mark=none, line width=1.6pt, 
  smooth, 
  error bars/.cd, 
    y fixed,
    y dir=both, 
    y explicit
] table [x expr=\thisrow{x}, y=acc, col sep=comma] {data/main/MNIST/modelpoisoning_mom_ex_noregret.txt};

\addplot+[
  blue, dashed,  mark=none, line width=1.6pt, 
  smooth, 
  error bars/.cd, 
    y fixed,
    y dir=both, 
    y explicit
] table [x expr=\thisrow{x}, y=acc, col sep=comma] {data/main/MNIST/modelpoisoning_krum.txt};
\addlegendentry{Krum}

\addplot+[
  brown, dashed, mark=none, line width=1.6pt, 
  smooth, 
  error bars/.cd, 
    y fixed,
    y dir=both, 
    y explicit
] table [x expr=\thisrow{x}, y=acc, col sep=comma] {data/main/MNIST/modelpoisoning_trimmedmean.txt};
\addlegendentry{Trimmed Mean}

\addplot+[
  darkgreen, dashed, mark=none, line width=1.6pt, 
  smooth, 
  error bars/.cd, 
    y fixed,
    y dir=both, 
    y explicit
] table [x expr=\thisrow{x}, y=acc, col sep=comma] {data/main/MNIST/modelpoisoning_bulyankrum.txt};
\addlegendentry{Bulyan\_Krum}

\addplot+[
  orange, dashed, mark=none, line width=1.6pt, 
  smooth, 
  error bars/.cd, 
    y fixed,
    y dir=both, 
    y explicit
] table [x expr=\thisrow{x}, y=acc, col sep=comma] {data/main/MNIST/modelpoisoning_bulyantrimmedmean.txt};
\addlegendentry{Bulyan\_Trimmed\_Mean}

\addplot+[
  magenta, dashed, mark=none, line width=1.6pt, 
  smooth, 
  error bars/.cd, 
    y fixed,
    y dir=both, 
    y explicit
] table [x expr=\thisrow{x}, y=acc, col sep=comma] {data/main/MNIST/modelpoisoning_median.txt};
\addlegendentry{Median}

\addplot+[
  violet, dashed, mark=none, line width=1.6pt, 
  smooth, 
  error bars/.cd, 
    y fixed,
    y dir=both, 
    y explicit
] table [x expr=\thisrow{x}, y=acc, col sep=comma] {data/main/MNIST/modelpoisoning_bulyanmedian.txt};
\addlegendentry{Bulyan\_Median}

\addplot+[
  lime, dashed, mark=none, line width=1.6pt, 
  smooth, 
  error bars/.cd, 
    y fixed,
    y dir=both, 
    y explicit
] table [x expr=\thisrow{x}, y=acc, col sep=comma] {data/main/MNIST/modelpoisoning_history.txt};
\addlegendentry{\cite{karimireddy2021learning}}

\addplot+[
  teal, dashed, mark=none, line width=1.6pt, 
  smooth, 
  error bars/.cd, 
    y fixed,
    y dir=both, 
    y explicit
] table [x expr=\thisrow{x}, y=acc, col sep=comma] {data/main/MNIST/modelpoisoning_bucketing.txt};
\addlegendentry{\cite{karimireddy2020byzantine}}

\addplot+[
  olive, dashed, mark=none, line width=1.6pt, 
  smooth, 
  error bars/.cd, 
    y fixed,
    y dir=both, 
    y explicit
] table [x expr=\thisrow{x}, y=acc, col sep=comma] {data/main/MNIST/modelpoisoning_clustering.txt};

\legend{};

\end{axis}
\end{tikzpicture}}
         \caption{MNIST under MPA.}
         \label{fig:mnist_modelpoisoning_acc}
     \end{subfigure}
     \hfill
     \begin{subfigure}[b]{\figwidth}
         \centering
         \resizebox{\textwidth}{!}{\begin{tikzpicture}
\begin{axis}[
  set layers,
  grid=major,
  xmin=0, xmax=50,
  ytick align=outside, ytick pos=left,
  xtick align=outside, xtick pos=left,
  xlabel=\# Epoch,
  ylabel={Attack Success Rate (\%)},
  legend pos=south east,
  enlarge y limits=0.05
]

\addplot+[
  red, dashed, mark=none, line width=1.6pt, 
  smooth, 
  error bars/.cd, 
    y fixed,
    y dir=both, 
    y explicit
] table [x expr=\thisrow{x}, y=asr, col sep=comma] {data/main/MNIST/backdoor_average.txt};
\addlegendentry{Average}

\addplot+[
  cyan, mark=none, line width=2pt, 
  smooth, 
  error bars/.cd, 
    y fixed,
    y dir=both, 
    y explicit
] table [x expr=\thisrow{x}, y=asr, col sep=comma] {data/main/MNIST/backdoor_filterl2.txt};
\addlegendentry{Krum Attack + filterL2}

\addplot+[
  green, mark=none, line width=2pt, 
  smooth, 
  error bars/.cd, 
    y fixed,
    y dir=both, 
    y explicit
] table [x expr=\thisrow{x}, y=asr, col sep=comma] {data/main/MNIST/backdoor_ex_noregret.txt};
\addlegendentry{Explicit\_No\_Regret}

\addplot+[
  black, mark=none, line width=1.6pt, 
  smooth, 
  error bars/.cd, 
    y fixed,
    y dir=both, 
    y explicit
] table [x expr=\thisrow{x}, y=asr, col sep=comma] {data/main/MNIST/backdoor_mom_filterl2.txt};

\addplot+[
  gray, mark=none, line width=1.6pt, 
  smooth, 
  error bars/.cd, 
    y fixed,
    y dir=both, 
    y explicit
] table [x expr=\thisrow{x}, y=asr, col sep=comma] {data/main/MNIST/backdoor_mom_ex_noregret.txt};

\addplot+[
  blue, dashed,  mark=none, line width=1.6pt, 
  smooth, 
  error bars/.cd, 
    y fixed,
    y dir=both, 
    y explicit
] table [x expr=\thisrow{x}, y=asr, col sep=comma] {data/main/MNIST/backdoor_krum.txt};
\addlegendentry{Krum}

\addplot+[
  brown, dashed, mark=none, line width=1.6pt, 
  smooth, 
  error bars/.cd, 
    y fixed,
    y dir=both, 
    y explicit
] table [x expr=\thisrow{x}, y=asr, col sep=comma] {data/main/MNIST/backdoor_trimmedmean.txt};
\addlegendentry{Trimmed Mean}

\addplot+[
  darkgreen, dashed, mark=none, line width=1.6pt, 
  smooth, 
  error bars/.cd, 
    y fixed,
    y dir=both, 
    y explicit
] table [x expr=\thisrow{x}, y=asr, col sep=comma] {data/main/MNIST/backdoor_bulyankrum.txt};
\addlegendentry{Bulyan\_Krum}

\addplot+[
  orange, dashed, mark=none, line width=1.6pt, 
  smooth, 
  error bars/.cd, 
    y fixed,
    y dir=both, 
    y explicit
] table [x expr=\thisrow{x}, y=asr, col sep=comma] {data/main/MNIST/backdoor_bulyantrimmedmean.txt};
\addlegendentry{Bulyan\_Trimmed\_Mean}

\addplot+[
  magenta, dashed, mark=none, line width=1.6pt, 
  smooth, 
  error bars/.cd, 
    y fixed,
    y dir=both, 
    y explicit
] table [x expr=\thisrow{x}, y=asr, col sep=comma] {data/main/MNIST/backdoor_median.txt};
\addlegendentry{Median}

\addplot+[
  violet, dashed, mark=none, line width=1.6pt, 
  smooth, 
  error bars/.cd, 
    y fixed,
    y dir=both, 
    y explicit
] table [x expr=\thisrow{x}, y=asr, col sep=comma] {data/main/MNIST/backdoor_bulyanmedian.txt};
\addlegendentry{Bulyan\_Median}

\addplot+[
  lime, dashed, mark=none, line width=1.6pt, 
  smooth, 
  error bars/.cd, 
    y fixed,
    y dir=both, 
    y explicit
] table [x expr=\thisrow{x}, y=asr, col sep=comma] {data/main/MNIST/backdoor_history.txt};
\addlegendentry{\cite{karimireddy2021learning}}

\addplot+[
  teal, dashed, mark=none, line width=1.6pt, 
  smooth, 
  error bars/.cd, 
    y fixed,
    y dir=both, 
    y explicit
] table [x expr=\thisrow{x}, y=asr, col sep=comma] {data/main/MNIST/backdoor_bucketing.txt};
\addlegendentry{\cite{karimireddy2020byzantine}}

\addplot+[
  olive, dashed, mark=none, line width=1.6pt, 
  smooth, 
  error bars/.cd, 
    y fixed,
    y dir=both, 
    y explicit
] table [x expr=\thisrow{x}, y=asr, col sep=comma] {data/main/MNIST/backdoor_clustering.txt};

\legend{};

\end{axis}
\end{tikzpicture}}
         \caption{MNIST under MRA.}
         \label{fig:mnist_backdoor_asr}
     \end{subfigure}
     \hfill
     \begin{subfigure}[b]{\figwidth}
         \centering
         \resizebox{\textwidth}{!}{\begin{tikzpicture}
\begin{axis}[
  set layers,
  grid=major,
  xmin=0, xmax=50,
  ytick align=outside, ytick pos=left,
  xtick align=outside, xtick pos=left,
  xlabel=\# Epoch,
  ylabel={Model Accuracy (\%)},
  legend pos=south east,
  enlarge y limits=0.05
]

\addplot+[
  red, dashed, mark=none, line width=1.6pt, 
  smooth, 
  error bars/.cd, 
    y fixed,
    y dir=both, 
    y explicit
] table [x expr=\thisrow{x}, y=acc, col sep=comma] {data/main/MNIST/backdoor_average.txt};
\addlegendentry{Average}

\addplot+[
  cyan, mark=none, line width=2pt, 
  smooth, 
  error bars/.cd, 
    y fixed,
    y dir=both, 
    y explicit
] table [x expr=\thisrow{x}, y=acc, col sep=comma] {data/main/MNIST/backdoor_filterl2.txt};
\addlegendentry{Krum Attack + filterL2}

\addplot+[
  green, mark=none, line width=2pt, 
  smooth, 
  error bars/.cd, 
    y fixed,
    y dir=both, 
    y explicit
] table [x expr=\thisrow{x}, y=acc, col sep=comma] {data/main/MNIST/backdoor_ex_noregret.txt};
\addlegendentry{Explicit\_No\_Regret}

\addplot+[
  black, mark=none, line width=1.6pt, 
  smooth, 
  error bars/.cd, 
    y fixed,
    y dir=both, 
    y explicit
] table [x expr=\thisrow{x}, y=acc, col sep=comma] {data/main/MNIST/backdoor_mom_filterl2.txt};

\addplot+[
  gray, mark=none, line width=1.6pt, 
  smooth, 
  error bars/.cd, 
    y fixed,
    y dir=both, 
    y explicit
] table [x expr=\thisrow{x}, y=acc, col sep=comma] {data/main/MNIST/backdoor_mom_ex_noregret.txt};

\addplot+[
  blue, dashed,  mark=none, line width=1.6pt, 
  smooth, 
  error bars/.cd, 
    y fixed,
    y dir=both, 
    y explicit
] table [x expr=\thisrow{x}, y=acc, col sep=comma] {data/main/MNIST/backdoor_krum.txt};
\addlegendentry{Krum}

\addplot+[
  brown, dashed, mark=none, line width=1.6pt, 
  smooth, 
  error bars/.cd, 
    y fixed,
    y dir=both, 
    y explicit
] table [x expr=\thisrow{x}, y=acc, col sep=comma] {data/main/MNIST/backdoor_trimmedmean.txt};
\addlegendentry{Trimmed Mean}

\addplot+[
  darkgreen, dashed, mark=none, line width=1.6pt, 
  smooth, 
  error bars/.cd, 
    y fixed,
    y dir=both, 
    y explicit
] table [x expr=\thisrow{x}, y=acc, col sep=comma] {data/main/MNIST/backdoor_bulyankrum.txt};
\addlegendentry{Bulyan\_Krum}

\addplot+[
  orange, dashed, mark=none, line width=1.6pt, 
  smooth, 
  error bars/.cd, 
    y fixed,
    y dir=both, 
    y explicit
] table [x expr=\thisrow{x}, y=acc, col sep=comma] {data/main/MNIST/backdoor_bulyantrimmedmean.txt};
\addlegendentry{Bulyan\_Trimmed\_Mean}

\addplot+[
  magenta, dashed, mark=none, line width=1.6pt, 
  smooth, 
  error bars/.cd, 
    y fixed,
    y dir=both, 
    y explicit
] table [x expr=\thisrow{x}, y=acc, col sep=comma] {data/main/MNIST/backdoor_median.txt};
\addlegendentry{Median}

\addplot+[
  violet, dashed, mark=none, line width=1.6pt, 
  smooth, 
  error bars/.cd, 
    y fixed,
    y dir=both, 
    y explicit
] table [x expr=\thisrow{x}, y=acc, col sep=comma] {data/main/MNIST/backdoor_bulyanmedian.txt};
\addlegendentry{Bulyan\_Median}

\addplot+[
  lime, dashed, mark=none, line width=1.6pt, 
  smooth, 
  error bars/.cd, 
    y fixed,
    y dir=both, 
    y explicit
] table [x expr=\thisrow{x}, y=acc, col sep=comma] {data/main/MNIST/backdoor_history.txt};
\addlegendentry{\cite{karimireddy2021learning}}

\addplot+[
  teal, dashed, mark=none, line width=1.6pt, 
  smooth, 
  error bars/.cd, 
    y fixed,
    y dir=both, 
    y explicit
] table [x expr=\thisrow{x}, y=acc, col sep=comma] {data/main/MNIST/backdoor_bucketing.txt};
\addlegendentry{\cite{karimireddy2020byzantine}}

\addplot+[
  olive, dashed, mark=none, line width=1.6pt, 
  smooth, 
  error bars/.cd, 
    y fixed,
    y dir=both, 
    y explicit
] table [x expr=\thisrow{x}, y=acc, col sep=comma] {data/main/MNIST/backdoor_clustering.txt};

\legend{};

\end{axis}
\end{tikzpicture}}
         \caption{MNIST under MRA.}
         \label{fig:mnist_backdoor_acc}
     \end{subfigure}
     \hfill
     \begin{subfigure}[b]{\figwidth}
         \centering
         \resizebox{\textwidth}{!}{\begin{tikzpicture}
\begin{axis}[
  set layers,
  grid=major,
  xmin=0, xmax=50,
  ytick align=outside, ytick pos=left,
  xtick align=outside, xtick pos=left,
  xlabel=\# Epoch,
  ylabel={Attack Success Rate (\%)},
  legend pos=south east,
  enlarge y limits=0.05
]

\addplot+[
  red, dashed, mark=none, line width=1.6pt, 
  smooth, 
  error bars/.cd, 
    y fixed,
    y dir=both, 
    y explicit
] table [x expr=\thisrow{x}, y=asr, col sep=comma] {data/main/MNIST/dba_average.txt};
\addlegendentry{Average}

\addplot+[
  cyan, mark=none, line width=2pt, 
  smooth, 
  error bars/.cd, 
    y fixed,
    y dir=both, 
    y explicit
] table [x expr=\thisrow{x}, y=asr, col sep=comma] {data/main/MNIST/dba_filterL2.txt};
\addlegendentry{FilterL2}

\addplot+[
  green, mark=none, line width=2pt, 
  smooth, 
  error bars/.cd, 
    y fixed,
    y dir=both, 
    y explicit
] table [x expr=\thisrow{x}, y=asr, col sep=comma] {data/main/MNIST/dba_ex_noregret.txt};
\addlegendentry{Explicit\_No\_Regret}

\addplot+[
  black, mark=none, line width=1.6pt, 
  smooth, 
  error bars/.cd, 
    y fixed,
    y dir=both, 
    y explicit
] table [x expr=\thisrow{x}, y=asr, col sep=comma] {data/main/MNIST/dba_mom_filterl2.txt};

\addplot+[
  gray, mark=none, line width=1.6pt, 
  smooth, 
  error bars/.cd, 
    y fixed,
    y dir=both, 
    y explicit
] table [x expr=\thisrow{x}, y=asr, col sep=comma] {data/main/MNIST/dba_mom_ex_noregret.txt};

\addplot+[
  blue, dashed,  mark=none, line width=1.6pt, 
  smooth, 
  error bars/.cd, 
    y fixed,
    y dir=both, 
    y explicit
] table [x expr=\thisrow{x}, y=asr, col sep=comma] {data/main/MNIST/dba_krum.txt};
\addlegendentry{Krum}

\addplot+[
  brown, dashed, mark=none, line width=1.6pt, 
  smooth, 
  error bars/.cd, 
    y fixed,
    y dir=both, 
    y explicit
] table [x expr=\thisrow{x}, y=asr, col sep=comma] {data/main/MNIST/dba_trimmedmean.txt};
\addlegendentry{Trimmed Mean}

\addplot+[
  darkgreen, dashed, mark=none, line width=1.6pt, 
  smooth, 
  error bars/.cd, 
    y fixed,
    y dir=both, 
    y explicit
] table [x expr=\thisrow{x}, y=asr, col sep=comma] {data/main/MNIST/dba_bulyankrum.txt};
\addlegendentry{Bulyan\_Krum}

\addplot+[
  orange, dashed, mark=none, line width=1.6pt, 
  smooth, 
  error bars/.cd, 
    y fixed,
    y dir=both, 
    y explicit
] table [x expr=\thisrow{x}, y=asr, col sep=comma] {data/main/MNIST/dba_bulyantrm.txt};
\addlegendentry{Bulyan\_Trimmed\_Mean}

\addplot+[
  magenta, dashed, mark=none, line width=1.6pt, 
  smooth, 
  error bars/.cd, 
    y fixed,
    y dir=both, 
    y explicit
] table [x expr=\thisrow{x}, y=asr, col sep=comma] {data/main/MNIST/dba_median.txt};
\addlegendentry{Median}

\addplot+[
  violet, dashed, mark=none, line width=1.6pt, 
  smooth, 
  error bars/.cd, 
    y fixed,
    y dir=both, 
    y explicit
] table [x expr=\thisrow{x}, y=asr, col sep=comma] {data/main/MNIST/dba_bulyanmedian.txt};
\addlegendentry{Bulyan\_Median}

\addplot+[
  lime, dashed, mark=none, line width=1.6pt, 
  smooth, 
  error bars/.cd, 
    y fixed,
    y dir=both, 
    y explicit
] table [x expr=\thisrow{x}, y=asr, col sep=comma] {data/main/MNIST/dba_history.txt};
\addlegendentry{\cite{karimireddy2021learning}}

\addplot+[
  teal, dashed, mark=none, line width=1.6pt, 
  smooth, 
  error bars/.cd, 
    y fixed,
    y dir=both, 
    y explicit
] table [x expr=\thisrow{x}, y=asr, col sep=comma] {data/main/MNIST/dba_bucketing.txt};
\addlegendentry{\cite{karimireddy2020byzantine}}

\addplot+[
  olive, dashed, mark=none, line width=1.6pt, 
  smooth, 
  error bars/.cd, 
    y fixed,
    y dir=both, 
    y explicit
] table [x expr=\thisrow{x}, y=asr, col sep=comma] {data/main/MNIST/dba_clustering.txt};

\legend{};

\end{axis}
\end{tikzpicture}}
         \caption{MNIST under DBA.}
         \label{fig:mnist_dba_asr}
     \end{subfigure}
     
     \hfill
     \begin{subfigure}[b]{\figwidth}
         \centering
         \resizebox{\textwidth}{!}{\begin{tikzpicture}
\begin{axis}[
  set layers,
  grid=major,
  xmin=0, xmax=50,
  ytick align=outside, ytick pos=left,
  xtick align=outside, xtick pos=left,
  xlabel=\# Epoch,
  ylabel={Model Accuracy (\%)},
  legend pos=south east,
  enlarge y limits=0.05
]

\addplot+[
  red, dashed, mark=none, line width=1.6pt, 
  smooth, 
  error bars/.cd, 
    y fixed,
    y dir=both, 
    y explicit
] table [x expr=\thisrow{x}, y=acc, col sep=comma] {data/main/MNIST/dba_average.txt};
\addlegendentry{Average}

\addplot+[
  cyan, mark=none, line width=2pt, 
  smooth, 
  error bars/.cd, 
    y fixed,
    y dir=both, 
    y explicit
] table [x expr=\thisrow{x}, y=acc, col sep=comma] {data/main/MNIST/dba_filterL2.txt};
\addlegendentry{FilterL2}

\addplot+[
  green, mark=none, line width=2pt, 
  smooth, 
  error bars/.cd, 
    y fixed,
    y dir=both, 
    y explicit
] table [x expr=\thisrow{x}, y=acc, col sep=comma] {data/main/MNIST/dba_ex_noregret.txt};
\addlegendentry{Explicit\_No\_Regret}

\addplot+[
  black, mark=none, line width=1.6pt, 
  smooth, 
  error bars/.cd, 
    y fixed,
    y dir=both, 
    y explicit
] table [x expr=\thisrow{x}, y=acc, col sep=comma] {data/main/MNIST/dba_mom_filterl2.txt};

\addplot+[
  gray, mark=none, line width=1.6pt, 
  smooth, 
  error bars/.cd, 
    y fixed,
    y dir=both, 
    y explicit
] table [x expr=\thisrow{x}, y=acc, col sep=comma] {data/main/MNIST/dba_mom_ex_noregret.txt};

\addplot+[
  blue, dashed,  mark=none, line width=1.6pt, 
  smooth, 
  error bars/.cd, 
    y fixed,
    y dir=both, 
    y explicit
] table [x expr=\thisrow{x}, y=acc, col sep=comma] {data/main/MNIST/dba_krum.txt};
\addlegendentry{Krum}

\addplot+[
  brown, dashed, mark=none, line width=1.6pt, 
  smooth, 
  error bars/.cd, 
    y fixed,
    y dir=both, 
    y explicit
] table [x expr=\thisrow{x}, y=acc, col sep=comma] {data/main/MNIST/dba_trimmedmean.txt};
\addlegendentry{Trimmed Mean}

\addplot+[
  darkgreen, dashed, mark=none, line width=1.6pt, 
  smooth, 
  error bars/.cd, 
    y fixed,
    y dir=both, 
    y explicit
] table [x expr=\thisrow{x}, y=acc, col sep=comma] {data/main/MNIST/dba_bulyankrum.txt};
\addlegendentry{Bulyan\_Krum}

\addplot+[
  orange, dashed, mark=none, line width=1.6pt, 
  smooth, 
  error bars/.cd, 
    y fixed,
    y dir=both, 
    y explicit
] table [x expr=\thisrow{x}, y=acc, col sep=comma] {data/main/MNIST/dba_bulyantrm.txt};
\addlegendentry{Bulyan\_Trimmed\_Mean}

\addplot+[
  magenta, dashed, mark=none, line width=1.6pt, 
  smooth, 
  error bars/.cd, 
    y fixed,
    y dir=both, 
    y explicit
] table [x expr=\thisrow{x}, y=acc, col sep=comma] {data/main/MNIST/dba_median.txt};
\addlegendentry{Median}

\addplot+[
  violet, dashed, mark=none, line width=1.6pt, 
  smooth, 
  error bars/.cd, 
    y fixed,
    y dir=both, 
    y explicit
] table [x expr=\thisrow{x}, y=acc, col sep=comma] {data/main/MNIST/dba_bulyanmedian.txt};
\addlegendentry{Bulyan\_Median}

\addplot+[
  lime, dashed, mark=none, line width=1.6pt, 
  smooth, 
  error bars/.cd, 
    y fixed,
    y dir=both, 
    y explicit
] table [x expr=\thisrow{x}, y=acc, col sep=comma] {data/main/MNIST/dba_history.txt};
\addlegendentry{\cite{karimireddy2021learning}}

\addplot+[
  teal, dashed, mark=none, line width=1.6pt, 
  smooth, 
  error bars/.cd, 
    y fixed,
    y dir=both, 
    y explicit
] table [x expr=\thisrow{x}, y=acc, col sep=comma] {data/main/MNIST/dba_bucketing.txt};
\addlegendentry{\cite{karimireddy2020byzantine}}

\addplot+[
  olive, dashed, mark=none, line width=1.6pt, 
  smooth, 
  error bars/.cd, 
    y fixed,
    y dir=both, 
    y explicit
] table [x expr=\thisrow{x}, y=acc, col sep=comma] {data/main/MNIST/dba_clustering.txt};

\legend{};

\end{axis}
\end{tikzpicture}}
         \caption{MNIST under DBA.}
         \label{fig:mnist_dba_acc}
     \end{subfigure}
     \hfill
     \begin{subfigure}[b]{\figwidth}
         \centering
         \resizebox{\textwidth}{!}{\begin{tikzpicture}[spy using outlines={rectangle, magnification=2, width=6cm, height=2cm, connect spies}]
\begin{axis}[
  grid=major,
  xmin=0, xmax=50,
  ytick align=outside, ytick pos=left,
  xtick align=outside, xtick pos=left,
  xlabel=\# Epoch,
  ylabel={Model Accuracy (\%)},
  enlarge y limits=0.05
  ]

\addplot+[
  red, dashed, mark=none, line width=1.6pt, 
  smooth, 
  error bars/.cd, 
    y fixed,
    y dir=both, 
    y explicit
] table [x expr=\thisrow{x}, y=y, col sep=comma] {data/main/MNIST/xie_average.txt};
\addlegendentry{Average}

\addplot+[
  cyan, mark=none, line width=2pt, 
  smooth, 
  error bars/.cd, 
    y fixed,
    y dir=both, 
    y explicit
] table [x expr=\thisrow{x}, y=y, col sep=comma] {data/main/MNIST/xie_filterl2.txt};
\addlegendentry{filterL2}

\addplot+[
  green, mark=none, line width=2pt, 
  smooth, 
  error bars/.cd, 
    y fixed,
    y dir=both, 
    y explicit
] table [x expr=\thisrow{x}, y=y, col sep=comma] {data/main/MNIST/xie_ex_noregret.txt};
\addlegendentry{Explicit\_No\_Regret}

\addplot+[
  black, mark=none, line width=1.6pt, 
  smooth, 
  error bars/.cd, 
    y fixed,
    y dir=both, 
    y explicit
] table [x expr=\thisrow{x}, y=y, col sep=comma] {data/main/MNIST/xie_mom_filterl2.txt};

\addplot+[
  gray, mark=none, line width=1.6pt, 
  smooth, 
  error bars/.cd, 
    y fixed,
    y dir=both, 
    y explicit
] table [x expr=\thisrow{x}, y=y, col sep=comma] {data/main/MNIST/xie_mom_ex_noregret.txt};

\addplot+[
  blue, dashed,  mark=none, line width=1.6pt, 
  smooth, 
  error bars/.cd, 
    y fixed,
    y dir=both, 
    y explicit
] table [x expr=\thisrow{x}, y=y, col sep=comma] {data/main/MNIST/xie_krum.txt};
\addlegendentry{Krum}

\addplot+[
  brown, dashed, mark=none, line width=1.6pt, 
  smooth, 
  error bars/.cd, 
    y fixed,
    y dir=both, 
    y explicit
] table [x expr=\thisrow{x}, y=y, col sep=comma] {data/main/MNIST/xie_trimmedmean.txt};
\addlegendentry{Trimmed Mean}

\addplot+[
  darkgreen, dashed, mark=none, line width=1.6pt, 
  smooth, 
  error bars/.cd, 
    y fixed,
    y dir=both, 
    y explicit
] table [x expr=\thisrow{x}, y=y, col sep=comma] {data/main/MNIST/xie_bulyankrum.txt};
\addlegendentry{Bulyan\_Krum}

\addplot+[
  orange, dashed, mark=none, line width=1.6pt, 
  smooth, 
  error bars/.cd, 
    y fixed,
    y dir=both, 
    y explicit
] table [x expr=\thisrow{x}, y=y, col sep=comma] {data/main/MNIST/xie_bulyantrimmedmean.txt};
\addlegendentry{Bulyan\_Trimmed\_Mean}

\addplot+[
  magenta, dashed, mark=none, line width=1.6pt, 
  smooth, 
  error bars/.cd, 
    y fixed,
    y dir=both, 
    y explicit
] table [x expr=\thisrow{x}, y=y, col sep=comma] {data/main/MNIST/xie_median.txt};
\addlegendentry{Median}

\addplot+[
  violet, dashed, mark=none, line width=1.6pt, 
  smooth, 
  error bars/.cd, 
    y fixed,
    y dir=both, 
    y explicit
] table [x expr=\thisrow{x}, y=y, col sep=comma] {data/main/MNIST/xie_bulyanmedian.txt};
\addlegendentry{Bulyan\_Median}

\addplot+[
  lime, dashed, mark=none, line width=1.6pt, 
  smooth, 
  error bars/.cd, 
    y fixed,
    y dir=both, 
    y explicit
] table [x expr=\thisrow{x}, y=y, col sep=comma] {data/main/MNIST/xie_history.txt};
\addlegendentry{\cite{karimireddy2021learning}}

\addplot+[
  teal, dashed, mark=none, line width=1.6pt, 
  smooth, 
  error bars/.cd, 
    y fixed,
    y dir=both, 
    y explicit
] table [x expr=\thisrow{x}, y=y, col sep=comma] {data/main/MNIST/xie_bucketing.txt};
\addlegendentry{\cite{karimireddy2020byzantine}}

\addplot+[
  olive, dashed, mark=none, line width=1.6pt, 
  smooth, 
  error bars/.cd, 
    y fixed,
    y dir=both, 
    y explicit
] table [x expr=\thisrow{x}, y=y, col sep=comma] {data/main/MNIST/xie_clustering.txt};

\legend{};


\end{axis}
\end{tikzpicture}}
         \caption{MNIST under IMA.}
         \label{fig:mnist_xie}
     \end{subfigure}
     \hfill
     \begin{subfigure}[b]{\figwidth}
         \centering
         \resizebox{\textwidth}{!}{\begin{tikzpicture}[spy using outlines={rectangle, magnification=2, width=6cm, height=2cm, connect spies}]
\begin{axis}[
  grid=major,
  xmin=0, xmax=50,
  ytick align=outside, ytick pos=left,
  xtick align=outside, xtick pos=left,
  xlabel=\# Epoch,
  ylabel={Model Accuracy (\%)},
  enlarge y limits=0.05,
  ]

\addplot+[
  red, dashed, mark=none, line width=1.6pt, 
  smooth, 
  error bars/.cd, 
    y fixed,
    y dir=both, 
    y explicit
] table [x expr=\thisrow{x}, y=y, col sep=comma] {data/main/fashionMNIST/noattack_average.txt};
\addlegendentry{Average}

\addplot+[
  cyan, mark=none, line width=2pt, 
  smooth, 
  error bars/.cd, 
    y fixed,
    y dir=both, 
    y explicit
] table [x expr=\thisrow{x}, y=y, col sep=comma] {data/main/fashionMNIST/noattack_filterl2.txt};
\addlegendentry{filterL2}

\addplot+[
  green, mark=none, line width=2pt, 
  smooth, 
  error bars/.cd, 
    y fixed,
    y dir=both, 
    y explicit
] table [x expr=\thisrow{x}, y=y, col sep=comma] {data/main/fashionMNIST/noattack_ex_noregret.txt};
\addlegendentry{Explicit\_No\_Regret}

\addplot+[
  black, mark=none, line width=1.6pt, 
  smooth, 
  error bars/.cd, 
    y fixed,
    y dir=both, 
    y explicit
] table [x expr=\thisrow{x}, y=y, col sep=comma] {data/main/fashionMNIST/noattack_mom_filterl2.txt};

\addplot+[
  gray, mark=none, line width=1.6pt, 
  smooth, 
  error bars/.cd, 
    y fixed,
    y dir=both, 
    y explicit
] table [x expr=\thisrow{x}, y=y, col sep=comma] {data/main/fashionMNIST/noattack_mom_ex_noregret.txt};

\addplot+[
  blue, dashed,  mark=none, line width=1.6pt, 
  smooth, 
  error bars/.cd, 
    y fixed,
    y dir=both, 
    y explicit
] table [x expr=\thisrow{x}, y=y, col sep=comma] {data/main/fashionMNIST/noattack_krum.txt};
\addlegendentry{Krum}

\addplot+[
  brown, dashed, mark=none, line width=1.6pt, 
  smooth, 
  error bars/.cd, 
    y fixed,
    y dir=both, 
    y explicit
] table [x expr=\thisrow{x}, y=y, col sep=comma] {data/main/fashionMNIST/noattack_trimmedmean.txt};
\addlegendentry{Trimmed Mean}

\addplot+[
  darkgreen, dashed, mark=none, line width=1.6pt, 
  smooth, 
  error bars/.cd, 
    y fixed,
    y dir=both, 
    y explicit
] table [x expr=\thisrow{x}, y=y, col sep=comma] {data/main/fashionMNIST/noattack_bulyankrum.txt};
\addlegendentry{Bulyan\_Krum}

\addplot+[
  orange, dashed, mark=none, line width=1.6pt, 
  smooth, 
  error bars/.cd, 
    y fixed,
    y dir=both, 
    y explicit
] table [x expr=\thisrow{x}, y=y, col sep=comma] {data/main/fashionMNIST/noattack_bulyantrimmedmean.txt};
\addlegendentry{Bulyan\_Trimmed\_Mean}

\addplot+[
  magenta, dashed, mark=none, line width=1.6pt, 
  smooth, 
  error bars/.cd, 
    y fixed,
    y dir=both, 
    y explicit
] table [x expr=\thisrow{x}, y=y, col sep=comma] {data/main/fashionMNIST/noattack_median.txt};
\addlegendentry{Median}

\addplot+[
  violet, dashed, mark=none, line width=1.6pt, 
  smooth, 
  error bars/.cd, 
    y fixed,
    y dir=both, 
    y explicit
] table [x expr=\thisrow{x}, y=y, col sep=comma] {data/main/fashionMNIST/noattack_bulyanmedian.txt};
\addlegendentry{Bulyan\_Median}

\addplot+[
  lime, dashed, mark=none, line width=1.6pt, 
  smooth, 
  error bars/.cd, 
    y fixed,
    y dir=both, 
    y explicit
] table [x expr=\thisrow{x}, y=y, col sep=comma] {data/main/fashionMNIST/noattack_history.txt};
\addlegendentry{\cite{karimireddy2021learning}}

\addplot+[
  teal, dashed, mark=none, line width=1.6pt, 
  smooth, 
  error bars/.cd, 
    y fixed,
    y dir=both, 
    y explicit
] table [x expr=\thisrow{x}, y=y, col sep=comma] {data/main/fashionMNIST/noattack_bucketing.txt};
\addlegendentry{\cite{karimireddy2020byzantine}}

\addplot+[
  olive, dashed, mark=none, line width=1.6pt, 
  smooth, 
  error bars/.cd, 
    y fixed,
    y dir=both, 
    y explicit
] table [x expr=\thisrow{x}, y=y, col sep=comma] {data/main/fashionMNIST/noattack_clustering.txt};

\legend{};


\end{axis}
\end{tikzpicture}}
         \caption{F-MNIST w/o attack.}
         \label{fig:fashionmnist_noattack}
     \end{subfigure}
     \hfill
     \begin{subfigure}[b]{\figwidth}
         \centering
         \resizebox{\textwidth}{!}{\begin{tikzpicture}
\begin{axis}[
  set layers,
  grid=major,
  xmin=0, xmax=50,
  ytick align=outside, ytick pos=left,
  xtick align=outside, xtick pos=left,
  xlabel=\# Epoch,
  ylabel={Model Accuracy (\%)},
  enlarge y limits=0.05,
  ]

\addplot+[
  red, dashed, mark=none, line width=1.6pt, 
  smooth, 
  error bars/.cd, 
    y fixed,
    y dir=both, 
    y explicit
] table [x expr=\thisrow{x}, y=y, col sep=comma] {data/main/fashionMNIST/krum_average.txt};
\addlegendentry{Average}

\addplot+[
  cyan, mark=none, line width=2pt, 
  smooth, 
  error bars/.cd, 
    y fixed,
    y dir=both, 
    y explicit
] table [x expr=\thisrow{x}, y=y, col sep=comma] {data/main/fashionMNIST/krum_filterl2.txt};
\addlegendentry{filterL2}

\addplot+[
  green, mark=none, line width=2pt, 
  smooth, 
  error bars/.cd, 
    y fixed,
    y dir=both, 
    y explicit
] table [x expr=\thisrow{x}, y=y, col sep=comma] {data/main/fashionMNIST/krum_ex_noregret.txt};
\addlegendentry{Explicit\_No\_Regret}

\addplot+[
  black, mark=none, line width=1.6pt, 
  smooth, 
  error bars/.cd, 
    y fixed,
    y dir=both, 
    y explicit
] table [x expr=\thisrow{x}, y=y, col sep=comma] {data/main/fashionMNIST/krum_mom_filterl2.txt};

\addplot+[
  gray, mark=none, line width=1.6pt, 
  smooth, 
  error bars/.cd, 
    y fixed,
    y dir=both, 
    y explicit
] table [x expr=\thisrow{x}, y=y, col sep=comma] {data/main/fashionMNIST/krum_mom_ex_noregret.txt};

\addplot+[
  blue, dashed,  mark=none, line width=1.6pt, 
  smooth, 
  error bars/.cd, 
    y fixed,
    y dir=both, 
    y explicit
] table [x expr=\thisrow{x}, y=y, col sep=comma] {data/main/fashionMNIST/krum_krum.txt};
\addlegendentry{Krum}

\addplot+[
  brown, dashed, mark=none, line width=1.6pt, 
  smooth, 
  error bars/.cd, 
    y fixed,
    y dir=both, 
    y explicit
] table [x expr=\thisrow{x}, y=y, col sep=comma] {data/main/fashionMNIST/krum_trimmedmean.txt};
\addlegendentry{Trimmed Mean}

\addplot+[
  darkgreen, dashed, mark=none, line width=1.6pt, 
  smooth, 
  error bars/.cd, 
    y fixed,
    y dir=both, 
    y explicit
] table [x expr=\thisrow{x}, y=y, col sep=comma] {data/main/fashionMNIST/krum_bulyankrum.txt};
\addlegendentry{Bulyan\_Krum}

\addplot+[
  orange, dashed, mark=none, line width=1.6pt, 
  smooth, 
  error bars/.cd, 
    y fixed,
    y dir=both, 
    y explicit
] table [x expr=\thisrow{x}, y=y, col sep=comma] {data/main/fashionMNIST/krum_bulyantrimmedmean.txt};
\addlegendentry{Bulyan\_Trimmed\_Mean}

\addplot+[
  magenta, dashed, mark=none, line width=1.6pt, 
  smooth, 
  error bars/.cd, 
    y fixed,
    y dir=both, 
    y explicit
] table [x expr=\thisrow{x}, y=y, col sep=comma] {data/main/fashionMNIST/krum_median.txt};
\addlegendentry{Median}

\addplot+[
  violet, dashed, mark=none, line width=1.6pt, 
  smooth, 
  error bars/.cd, 
    y fixed,
    y dir=both, 
    y explicit
] table [x expr=\thisrow{x}, y=y, col sep=comma] {data/main/fashionMNIST/krum_bulyanmedian.txt};
\addlegendentry{Bulyan\_Median}

\addplot+[
  lime, dashed, mark=none, line width=1.6pt, 
  smooth, 
  error bars/.cd, 
    y fixed,
    y dir=both, 
    y explicit
] table [x expr=\thisrow{x}, y=y, col sep=comma] {data/main/fashionMNIST/krum_history.txt};
\addlegendentry{\cite{karimireddy2021learning}}

\addplot+[
  teal, dashed, mark=none, line width=1.6pt, 
  smooth, 
  error bars/.cd, 
    y fixed,
    y dir=both, 
    y explicit
] table [x expr=\thisrow{x}, y=y, col sep=comma] {data/main/fashionMNIST/krum_bucketing.txt};
\addlegendentry{\cite{karimireddy2020byzantine}}

\addplot+[
  olive, dashed, mark=none, line width=1.6pt, 
  smooth, 
  error bars/.cd, 
    y fixed,
    y dir=both, 
    y explicit
] table [x expr=\thisrow{x}, y=y, col sep=comma] {data/main/fashionMNIST/krum_clustering.txt};

\legend{};

\end{axis}
\end{tikzpicture}}
         \caption{F-MNIST under KA.}
         \label{fig:fashionmnist_krum}
     \end{subfigure}
     
     \hfill
     \begin{subfigure}[b]{\figwidth}
         \centering
         \resizebox{\textwidth}{!}{\begin{tikzpicture}
\begin{axis}[
  set layers,
  grid=major,
  xmin=0, xmax=50,
  ytick align=outside, ytick pos=left,
  xtick align=outside, xtick pos=left,
  xlabel=\# Epoch,
  ylabel={Model Accuracy (\%)},
  legend pos=south east,
  enlarge y limits=0.05,
  ]

\addplot+[
  red, dashed, mark=none, line width=1.6pt, 
  smooth, 
  error bars/.cd, 
    y fixed,
    y dir=both, 
    y explicit
] table [x expr=\thisrow{x}, y=y, col sep=comma] {data/main/fashionMNIST/trimmedmean_average.txt};
\addlegendentry{Average}

\addplot+[
  cyan, mark=none, line width=2pt, 
  smooth, 
  error bars/.cd, 
    y fixed,
    y dir=both, 
    y explicit
] table [x expr=\thisrow{x}, y=y, col sep=comma] {data/main/fashionMNIST/trimmedmean_filterl2.txt};
\addlegendentry{filterL2}

\addplot+[
  green, mark=none, line width=2pt, 
  smooth, 
  error bars/.cd, 
    y fixed,
    y dir=both, 
    y explicit
] table [x expr=\thisrow{x}, y=y, col sep=comma] {data/main/fashionMNIST/trimmedmean_ex_noregret.txt};
\addlegendentry{Explicit\_No\_Regret}

\addplot+[
  black, mark=none, line width=1.6pt, 
  smooth, 
  error bars/.cd, 
    y fixed,
    y dir=both, 
    y explicit
] table [x expr=\thisrow{x}, y=y, col sep=comma] {data/main/fashionMNIST/trimmedmean_mom_filterl2.txt};

\addplot+[
  gray, mark=none, line width=1.6pt, 
  smooth, 
  error bars/.cd, 
    y fixed,
    y dir=both, 
    y explicit
] table [x expr=\thisrow{x}, y=y, col sep=comma] {data/main/fashionMNIST/trimmedmean_mom_ex_noregret.txt};

\addplot+[
  blue, dashed, mark=none, line width=1.6pt, 
  smooth, 
  error bars/.cd, 
    y fixed,
    y dir=both, 
    y explicit
] table [x expr=\thisrow{x}, y=y, col sep=comma] {data/main/fashionMNIST/trimmedmean_krum.txt};
\addlegendentry{Krum}

\addplot+[
  brown, dashed, mark=none, line width=1.6pt, 
  smooth, 
  error bars/.cd, 
    y fixed,
    y dir=both, 
    y explicit
] table [x expr=\thisrow{x}, y=y, col sep=comma] {data/main/fashionMNIST/trimmedmean_trimmedmean.txt};
\addlegendentry{Trimmed Mean}

\addplot+[
  darkgreen, dashed, mark=none, line width=1.6pt, 
  smooth, 
  error bars/.cd, 
    y fixed,
    y dir=both, 
    y explicit
] table [x expr=\thisrow{x}, y=y, col sep=comma] {data/main/fashionMNIST/trimmedmean_bulyankrum.txt};
\addlegendentry{Bulyan\_Krum}

\addplot+[
  orange, dashed, mark=none, line width=1.6pt, 
  smooth, 
  error bars/.cd, 
    y fixed,
    y dir=both, 
    y explicit
] table [x expr=\thisrow{x}, y=y, col sep=comma] {data/main/fashionMNIST/trimmedmean_bulyantrimmedmean.txt};
\addlegendentry{Bulyan\_Trimmed\_Mean}

\addplot+[
  magenta, dashed, mark=none, line width=1.6pt, 
  smooth, 
  error bars/.cd, 
    y fixed,
    y dir=both, 
    y explicit
] table [x expr=\thisrow{x}, y=y, col sep=comma] {data/main/fashionMNIST/trimmedmean_median.txt};
\addlegendentry{Median}

\addplot+[
  violet, dashed, mark=none, line width=1.6pt, 
  smooth, 
  error bars/.cd, 
    y fixed,
    y dir=both, 
    y explicit
] table [x expr=\thisrow{x}, y=y, col sep=comma] {data/main/fashionMNIST/trimmedmean_bulyanmedian.txt};
\addlegendentry{Bulyan\_Median}

\addplot+[
  lime, dashed, mark=none, line width=1.6pt, 
  smooth, 
  error bars/.cd, 
    y fixed,
    y dir=both, 
    y explicit
] table [x expr=\thisrow{x}, y=y, col sep=comma] {data/main/fashionMNIST/trimmedmean_history.txt};
\addlegendentry{\cite{karimireddy2021learning}}

\addplot+[
  teal, dashed, mark=none, line width=1.6pt, 
  smooth, 
  error bars/.cd, 
    y fixed,
    y dir=both, 
    y explicit
] table [x expr=\thisrow{x}, y=y, col sep=comma] {data/main/fashionMNIST/trimmedmean_bucketing.txt};
\addlegendentry{\cite{karimireddy2020byzantine}}

\addplot+[
  olive, dashed, mark=none, line width=1.6pt, 
  smooth, 
  error bars/.cd, 
    y fixed,
    y dir=both, 
    y explicit
] table [x expr=\thisrow{x}, y=y, col sep=comma] {data/main/fashionMNIST/trimmedmean_clustering.txt};

\legend{};

\end{axis}
\end{tikzpicture}}
         \caption{F-MNIST under TMA.}
         \label{fig:fashionmnist_trimmedmean}
     \end{subfigure}
     \hfill
     \begin{subfigure}[b]{\figwidth}
         \centering
         \resizebox{\textwidth}{!}{\begin{tikzpicture}
\begin{axis}[
  set layers,
  grid=major,
  xmin=0, xmax=50,
  ytick align=outside, ytick pos=left,
  xtick align=outside, xtick pos=left,
  xlabel=\# Epoch,
  ylabel={Attack Success Rate (\%)},
  enlarge y limits=0.05,
  legend style={at={(1.5,0.4)}, anchor=north, nodes={scale=0.75, transform shape}}]

\addplot+[
  red, dashed, mark=none, line width=1.6pt, 
  smooth, 
  error bars/.cd, 
    y fixed,
    y dir=both, 
    y explicit
] table [x expr=\thisrow{x}, y=asr, col sep=comma] {data/main/fashionMNIST/modelpoisoning_average.txt};
\addlegendentry{Average}

\addplot+[
  cyan, mark=none, line width=2pt, 
  smooth, 
  error bars/.cd, 
    y fixed,
    y dir=both, 
    y explicit
] table [x expr=\thisrow{x}, y=asr, col sep=comma] {data/main/fashionMNIST/modelpoisoning_filterl2.txt};
\addlegendentry{Krum Attack + filterL2}

\addplot+[
  green, mark=none, line width=2pt, 
  smooth, 
  error bars/.cd, 
    y fixed,
    y dir=both, 
    y explicit
] table [x expr=\thisrow{x}, y=asr, col sep=comma] {data/main/fashionMNIST/modelpoisoning_ex_noregret.txt};
\addlegendentry{Explicit\_No\_Regret}

\addplot+[
  black, mark=none, line width=1.6pt, 
  smooth, 
  error bars/.cd, 
    y fixed,
    y dir=both, 
    y explicit
] table [x expr=\thisrow{x}, y=asr, col sep=comma] {data/main/fashionMNIST/modelpoisoning_mom_filterl2.txt};

\addplot+[
  gray, mark=none, line width=1.6pt, 
  smooth, 
  error bars/.cd, 
    y fixed,
    y dir=both, 
    y explicit
] table [x expr=\thisrow{x}, y=asr, col sep=comma] {data/main/fashionMNIST/modelpoisoning_mom_ex_noregret.txt};

\addplot+[
  blue, dashed,  mark=none, line width=1.6pt, 
  smooth, 
  error bars/.cd, 
    y fixed,
    y dir=both, 
    y explicit
] table [x expr=\thisrow{x}, y=asr, col sep=comma] {data/main/fashionMNIST/modelpoisoning_krum.txt};
\addlegendentry{Krum}

\addplot+[
  brown, dashed, mark=none, line width=1.6pt, 
  smooth, 
  error bars/.cd, 
    y fixed,
    y dir=both, 
    y explicit
] table [x expr=\thisrow{x}, y=asr, col sep=comma] {data/main/fashionMNIST/modelpoisoning_trimmedmean.txt};
\addlegendentry{Trimmed Mean}

\addplot+[
  darkgreen, dashed, mark=none, line width=1.6pt, 
  smooth, 
  error bars/.cd, 
    y fixed,
    y dir=both, 
    y explicit
] table [x expr=\thisrow{x}, y=asr, col sep=comma] {data/main/fashionMNIST/modelpoisoning_bulyankrum.txt};
\addlegendentry{Bulyan\_Krum}

\addplot+[
  orange, dashed, mark=none, line width=1.6pt, 
  smooth, 
  error bars/.cd, 
    y fixed,
    y dir=both, 
    y explicit
] table [x expr=\thisrow{x}, y=asr, col sep=comma] {data/main/fashionMNIST/modelpoisoning_bulyantrimmedmean.txt};
\addlegendentry{Bulyan\_Trimmed\_Mean}

\addplot+[
  magenta, dashed, mark=none, line width=1.6pt, 
  smooth, 
  error bars/.cd, 
    y fixed,
    y dir=both, 
    y explicit
] table [x expr=\thisrow{x}, y=asr, col sep=comma] {data/main/fashionMNIST/modelpoisoning_median.txt};
\addlegendentry{Median}

\addplot+[
  violet, dashed, mark=none, line width=1.6pt, 
  smooth, 
  error bars/.cd, 
    y fixed,
    y dir=both, 
    y explicit
] table [x expr=\thisrow{x}, y=asr, col sep=comma] {data/main/fashionMNIST/modelpoisoning_bulyanmedian.txt};
\addlegendentry{Bulyan\_Median}

\addplot+[
  lime, dashed, mark=none, line width=1.6pt, 
  smooth, 
  error bars/.cd, 
    y fixed,
    y dir=both, 
    y explicit
] table [x expr=\thisrow{x}, y=asr, col sep=comma] {data/main/fashionMNIST/modelpoisoning_history.txt};
\addlegendentry{\cite{karimireddy2021learning}}

\addplot+[
  teal, dashed, mark=none, line width=1.6pt, 
  smooth, 
  error bars/.cd, 
    y fixed,
    y dir=both, 
    y explicit
] table [x expr=\thisrow{x}, y=asr, col sep=comma] {data/main/fashionMNIST/modelpoisoning_bucketing.txt};
\addlegendentry{\cite{karimireddy2020byzantine}}

\addplot+[
  olive, dashed, mark=none, line width=1.6pt, 
  smooth, 
  error bars/.cd, 
    y fixed,
    y dir=both, 
    y explicit
] table [x expr=\thisrow{x}, y=asr, col sep=comma] {data/main/fashionMNIST/modelpoisoning_clustering.txt};

\legend{};

\end{axis}
\end{tikzpicture}}
         \caption{F-MNIST under MPA.}
         \label{fig:fashionmnist_modelpoisoning_asr}
     \end{subfigure}
     \hfill
     \begin{subfigure}[b]{\figwidth}
         \centering
         \resizebox{\textwidth}{!}{\begin{tikzpicture}
\begin{axis}[
  set layers,
  grid=major,
  xmin=0, xmax=50,
  ytick align=outside, ytick pos=left,
  xtick align=outside, xtick pos=left,
  xlabel=\# Epoch,
  ylabel={Model Accuracy (\%)},
  enlarge y limits=0.05,
  legend style={at={(1.5,0.4)}, anchor=north, nodes={scale=0.75, transform shape}}]

\addplot+[
  red, dashed, mark=none, line width=1.6pt, 
  smooth, 
  error bars/.cd, 
    y fixed,
    y dir=both, 
    y explicit
] table [x expr=\thisrow{x}, y=acc, col sep=comma] {data/main/fashionMNIST/modelpoisoning_average.txt};
\addlegendentry{Average}

\addplot+[
  cyan, mark=none, line width=2pt, 
  smooth, 
  error bars/.cd, 
    y fixed,
    y dir=both, 
    y explicit
] table [x expr=\thisrow{x}, y=acc, col sep=comma] {data/main/fashionMNIST/modelpoisoning_filterl2.txt};
\addlegendentry{Krum Attack + filterL2}

\addplot+[
  green, mark=none, line width=2pt, 
  smooth, 
  error bars/.cd, 
    y fixed,
    y dir=both, 
    y explicit
] table [x expr=\thisrow{x}, y=acc, col sep=comma] {data/main/fashionMNIST/modelpoisoning_ex_noregret.txt};
\addlegendentry{Explicit\_No\_Regret}

\addplot+[
  black, mark=none, line width=1.6pt, 
  smooth, 
  error bars/.cd, 
    y fixed,
    y dir=both, 
    y explicit
] table [x expr=\thisrow{x}, y=acc, col sep=comma] {data/main/fashionMNIST/modelpoisoning_mom_filterl2.txt};

\addplot+[
  gray, mark=none, line width=1.6pt, 
  smooth, 
  error bars/.cd, 
    y fixed,
    y dir=both, 
    y explicit
] table [x expr=\thisrow{x}, y=acc, col sep=comma] {data/main/fashionMNIST/modelpoisoning_mom_ex_noregret.txt};

\addplot+[
  blue, dashed,  mark=none, line width=1.6pt, 
  smooth, 
  error bars/.cd, 
    y fixed,
    y dir=both, 
    y explicit
] table [x expr=\thisrow{x}, y=acc, col sep=comma] {data/main/fashionMNIST/modelpoisoning_krum.txt};
\addlegendentry{Krum}

\addplot+[
  brown, dashed, mark=none, line width=1.6pt, 
  smooth, 
  error bars/.cd, 
    y fixed,
    y dir=both, 
    y explicit
] table [x expr=\thisrow{x}, y=acc, col sep=comma] {data/main/fashionMNIST/modelpoisoning_trimmedmean.txt};
\addlegendentry{Trimmed Mean}

\addplot+[
  darkgreen, dashed, mark=none, line width=1.6pt, 
  smooth, 
  error bars/.cd, 
    y fixed,
    y dir=both, 
    y explicit
] table [x expr=\thisrow{x}, y=acc, col sep=comma] {data/main/fashionMNIST/modelpoisoning_bulyankrum.txt};
\addlegendentry{Bulyan\_Krum}

\addplot+[
  orange, dashed, mark=none, line width=1.6pt, 
  smooth, 
  error bars/.cd, 
    y fixed,
    y dir=both, 
    y explicit
] table [x expr=\thisrow{x}, y=acc, col sep=comma] {data/main/fashionMNIST/modelpoisoning_bulyantrimmedmean.txt};
\addlegendentry{Bulyan\_Trimmed\_Mean}

\addplot+[
  magenta, dashed, mark=none, line width=1.6pt, 
  smooth, 
  error bars/.cd, 
    y fixed,
    y dir=both, 
    y explicit
] table [x expr=\thisrow{x}, y=acc, col sep=comma] {data/main/fashionMNIST/modelpoisoning_median.txt};
\addlegendentry{Median}

\addplot+[
  violet, dashed, mark=none, line width=1.6pt, 
  smooth, 
  error bars/.cd, 
    y fixed,
    y dir=both, 
    y explicit
] table [x expr=\thisrow{x}, y=acc, col sep=comma] {data/main/fashionMNIST/modelpoisoning_bulyanmedian.txt};
\addlegendentry{Bulyan\_Median}

\addplot+[
  lime, dashed, mark=none, line width=1.6pt, 
  smooth, 
  error bars/.cd, 
    y fixed,
    y dir=both, 
    y explicit
] table [x expr=\thisrow{x}, y=acc, col sep=comma] {data/main/fashionMNIST/modelpoisoning_history.txt};
\addlegendentry{\cite{karimireddy2021learning}}

\addplot+[
  teal, dashed, mark=none, line width=1.6pt, 
  smooth, 
  error bars/.cd, 
    y fixed,
    y dir=both, 
    y explicit
] table [x expr=\thisrow{x}, y=acc, col sep=comma] {data/main/fashionMNIST/modelpoisoning_bucketing.txt};
\addlegendentry{\cite{karimireddy2020byzantine}}

\addplot+[
  olive, dashed, mark=none, line width=1.6pt, 
  smooth, 
  error bars/.cd, 
    y fixed,
    y dir=both, 
    y explicit
] table [x expr=\thisrow{x}, y=acc, col sep=comma] {data/main/fashionMNIST/modelpoisoning_clustering.txt};

\legend{};

\end{axis}
\end{tikzpicture}}
         \caption{F-MNIST under MPA.}
         \label{fig:fashionmnist_modelpoisoning_acc}
     \end{subfigure}
     \hfill
     \begin{subfigure}[b]{\figwidth}
         \centering
         \resizebox{\textwidth}{!}{\begin{tikzpicture}
\begin{axis}[
  set layers,
  grid=major,
  xmin=0, xmax=50,
  ytick align=outside, ytick pos=left,
  xtick align=outside, xtick pos=left,
  xlabel=\# Epoch,
  ylabel={Attack Success Rate (\%)},
  enlarge y limits=0.05,
  legend style={at={(1.5,0.4)}, anchor=north, nodes={scale=0.75, transform shape}}]

\addplot+[
  red, dashed, mark=none, line width=1.6pt, 
  smooth, 
  error bars/.cd, 
    y fixed,
    y dir=both, 
    y explicit
] table [x expr=\thisrow{x}, y=asr, col sep=comma] {data/main/fashionMNIST/backdoor_average.txt};
\addlegendentry{Average}

\addplot+[
  cyan, mark=none, line width=2pt, 
  smooth, 
  error bars/.cd, 
    y fixed,
    y dir=both, 
    y explicit
] table [x expr=\thisrow{x}, y=asr, col sep=comma] {data/main/fashionMNIST/backdoor_filterl2.txt};
\addlegendentry{Krum Attack + filterL2}

\addplot+[
  green, mark=none, line width=2pt, 
  smooth, 
  error bars/.cd, 
    y fixed,
    y dir=both, 
    y explicit
] table [x expr=\thisrow{x}, y=asr, col sep=comma] {data/main/fashionMNIST/backdoor_ex_noregret.txt};
\addlegendentry{Explicit\_No\_Regret}

\addplot+[
  black, mark=none, line width=1.6pt, 
  smooth, 
  error bars/.cd, 
    y fixed,
    y dir=both, 
    y explicit
] table [x expr=\thisrow{x}, y=asr, col sep=comma] {data/main/fashionMNIST/backdoor_mom_filterl2.txt};

\addplot+[
  gray, mark=none, line width=1.6pt, 
  smooth, 
  error bars/.cd, 
    y fixed,
    y dir=both, 
    y explicit
] table [x expr=\thisrow{x}, y=asr, col sep=comma] {data/main/fashionMNIST/backdoor_mom_ex_noregret.txt};

\addplot+[
  blue, dashed,  mark=none, line width=1.6pt, 
  smooth, 
  error bars/.cd, 
    y fixed,
    y dir=both, 
    y explicit
] table [x expr=\thisrow{x}, y=asr, col sep=comma] {data/main/fashionMNIST/backdoor_krum.txt};
\addlegendentry{Krum}

\addplot+[
  brown, dashed, mark=none, line width=1.6pt, 
  smooth, 
  error bars/.cd, 
    y fixed,
    y dir=both, 
    y explicit
] table [x expr=\thisrow{x}, y=asr, col sep=comma] {data/main/fashionMNIST/backdoor_trimmedmean.txt};
\addlegendentry{Trimmed Mean}

\addplot+[
  darkgreen, dashed, mark=none, line width=1.6pt, 
  smooth, 
  error bars/.cd, 
    y fixed,
    y dir=both, 
    y explicit
] table [x expr=\thisrow{x}, y=asr, col sep=comma] {data/main/fashionMNIST/backdoor_bulyankrum.txt};
\addlegendentry{Bulyan\_Krum}

\addplot+[
  orange, dashed, mark=none, line width=1.6pt, 
  smooth, 
  error bars/.cd, 
    y fixed,
    y dir=both, 
    y explicit
] table [x expr=\thisrow{x}, y=asr, col sep=comma] {data/main/fashionMNIST/backdoor_bulyantrimmedmean.txt};
\addlegendentry{Bulyan\_Trimmed\_Mean}

\addplot+[
  magenta, dashed, mark=none, line width=1.6pt, 
  smooth, 
  error bars/.cd, 
    y fixed,
    y dir=both, 
    y explicit
] table [x expr=\thisrow{x}, y=asr, col sep=comma] {data/main/fashionMNIST/backdoor_median.txt};
\addlegendentry{Median}

\addplot+[
  violet, dashed, mark=none, line width=1.6pt, 
  smooth, 
  error bars/.cd, 
    y fixed,
    y dir=both, 
    y explicit
] table [x expr=\thisrow{x}, y=asr, col sep=comma] {data/main/fashionMNIST/backdoor_bulyanmedian.txt};
\addlegendentry{Bulyan\_Median}

\addplot+[
  lime, dashed, mark=none, line width=1.6pt, 
  smooth, 
  error bars/.cd, 
    y fixed,
    y dir=both, 
    y explicit
] table [x expr=\thisrow{x}, y=asr, col sep=comma] {data/main/fashionMNIST/backdoor_history.txt};
\addlegendentry{\cite{karimireddy2021learning}}

\addplot+[
  teal, dashed, mark=none, line width=1.6pt, 
  smooth, 
  error bars/.cd, 
    y fixed,
    y dir=both, 
    y explicit
] table [x expr=\thisrow{x}, y=asr, col sep=comma] {data/main/fashionMNIST/backdoor_bucketing.txt};
\addlegendentry{\cite{karimireddy2020byzantine}}

\addplot+[
  olive, dashed, mark=none, line width=1.6pt, 
  smooth, 
  error bars/.cd, 
    y fixed,
    y dir=both, 
    y explicit
] table [x expr=\thisrow{x}, y=asr, col sep=comma] {data/main/fashionMNIST/backdoor_clustering.txt};

\legend{};

\end{axis}
\end{tikzpicture}}
         \caption{F-MNIST under MRA.}
         \label{fig:fashionmnist_backdoor_asr}
     \end{subfigure}
     
     \hfill
     \begin{subfigure}[b]{\figwidth}
         \centering
         \resizebox{\textwidth}{!}{\begin{tikzpicture}
\begin{axis}[
  set layers,
  grid=major,
  xmin=0, xmax=50,
  ytick align=outside, ytick pos=left,
  xtick align=outside, xtick pos=left,
  xlabel=\# Epoch,
  ylabel={Model Accuracy (\%)},
  enlarge y limits=0.05,
  legend style={at={(1.5,0.4)}, anchor=north, nodes={scale=0.75, transform shape}}]

\addplot+[
  red, dashed, mark=none, line width=1.6pt, 
  smooth, 
  error bars/.cd, 
    y fixed,
    y dir=both, 
    y explicit
] table [x expr=\thisrow{x}, y=acc, col sep=comma] {data/main/fashionMNIST/backdoor_average.txt};
\addlegendentry{Average}

\addplot+[
  cyan, mark=none, line width=2pt, 
  smooth, 
  error bars/.cd, 
    y fixed,
    y dir=both, 
    y explicit
] table [x expr=\thisrow{x}, y=acc, col sep=comma] {data/main/fashionMNIST/backdoor_filterl2.txt};
\addlegendentry{Krum Attack + filterL2}

\addplot+[
  green, mark=none, line width=2pt, 
  smooth, 
  error bars/.cd, 
    y fixed,
    y dir=both, 
    y explicit
] table [x expr=\thisrow{x}, y=acc, col sep=comma] {data/main/fashionMNIST/backdoor_ex_noregret.txt};
\addlegendentry{Explicit\_No\_Regret}

\addplot+[
  black, mark=none, line width=1.6pt, 
  smooth, 
  error bars/.cd, 
    y fixed,
    y dir=both, 
    y explicit
] table [x expr=\thisrow{x}, y=acc, col sep=comma] {data/main/fashionMNIST/backdoor_mom_filterl2.txt};

\addplot+[
  gray, mark=none, line width=1.6pt, 
  smooth, 
  error bars/.cd, 
    y fixed,
    y dir=both, 
    y explicit
] table [x expr=\thisrow{x}, y=acc, col sep=comma] {data/main/fashionMNIST/backdoor_mom_ex_noregret.txt};

\addplot+[
  blue, dashed,  mark=none, line width=1.6pt, 
  smooth, 
  error bars/.cd, 
    y fixed,
    y dir=both, 
    y explicit
] table [x expr=\thisrow{x}, y=acc, col sep=comma] {data/main/fashionMNIST/backdoor_krum.txt};
\addlegendentry{Krum}

\addplot+[
  brown, dashed, mark=none, line width=1.6pt, 
  smooth, 
  error bars/.cd, 
    y fixed,
    y dir=both, 
    y explicit
] table [x expr=\thisrow{x}, y=acc, col sep=comma] {data/main/fashionMNIST/backdoor_trimmedmean.txt};
\addlegendentry{Trimmed Mean}

\addplot+[
  darkgreen, dashed, mark=none, line width=1.6pt, 
  smooth, 
  error bars/.cd, 
    y fixed,
    y dir=both, 
    y explicit
] table [x expr=\thisrow{x}, y=acc, col sep=comma] {data/main/fashionMNIST/backdoor_bulyankrum.txt};
\addlegendentry{Bulyan\_Krum}

\addplot+[
  orange, dashed, mark=none, line width=1.6pt, 
  smooth, 
  error bars/.cd, 
    y fixed,
    y dir=both, 
    y explicit
] table [x expr=\thisrow{x}, y=acc, col sep=comma] {data/main/fashionMNIST/backdoor_bulyantrimmedmean.txt};
\addlegendentry{Bulyan\_Trimmed\_Mean}

\addplot+[
  magenta, dashed, mark=none, line width=1.6pt, 
  smooth, 
  error bars/.cd, 
    y fixed,
    y dir=both, 
    y explicit
] table [x expr=\thisrow{x}, y=acc, col sep=comma] {data/main/fashionMNIST/backdoor_median.txt};
\addlegendentry{Median}

\addplot+[
  violet, dashed, mark=none, line width=1.6pt, 
  smooth, 
  error bars/.cd, 
    y fixed,
    y dir=both, 
    y explicit
] table [x expr=\thisrow{x}, y=acc, col sep=comma] {data/main/fashionMNIST/backdoor_bulyanmedian.txt};
\addlegendentry{Bulyan\_Median}

\addplot+[
  lime, dashed, mark=none, line width=1.6pt, 
  smooth, 
  error bars/.cd, 
    y fixed,
    y dir=both, 
    y explicit
] table [x expr=\thisrow{x}, y=acc, col sep=comma] {data/main/fashionMNIST/backdoor_history.txt};
\addlegendentry{\cite{karimireddy2021learning}}

\addplot+[
  teal, dashed, mark=none, line width=1.6pt, 
  smooth, 
  error bars/.cd, 
    y fixed,
    y dir=both, 
    y explicit
] table [x expr=\thisrow{x}, y=acc, col sep=comma] {data/main/fashionMNIST/backdoor_bucketing.txt};
\addlegendentry{\cite{karimireddy2020byzantine}}

\addplot+[
  olive, dashed, mark=none, line width=1.6pt, 
  smooth, 
  error bars/.cd, 
    y fixed,
    y dir=both, 
    y explicit
] table [x expr=\thisrow{x}, y=acc, col sep=comma] {data/main/fashionMNIST/backdoor_clustering.txt};

\legend{};

\end{axis}
\end{tikzpicture}}
         \caption{F-MNIST under MRA.}
         \label{fig:fashionmnist_backdoor_acc}
     \end{subfigure}
     \hfill
     \begin{subfigure}[b]{\figwidth}
         \centering
         \resizebox{\textwidth}{!}{\begin{tikzpicture}
\begin{axis}[
  set layers,
  grid=major,
  xmin=0, xmax=50,
  ytick align=outside, ytick pos=left,
  xtick align=outside, xtick pos=left,
  xlabel=\# Epoch,
  ylabel={Attack Success Rate (\%)},
  enlarge y limits=0.05,
  legend style={at={(1.5,0.4)}, anchor=north, nodes={scale=0.75, transform shape}}]

\addplot+[
  red, dashed, mark=none, line width=1.6pt, 
  smooth, 
  error bars/.cd, 
    y fixed,
    y dir=both, 
    y explicit
] table [x expr=\thisrow{x}, y=asr, col sep=comma] {data/main/fashionMNIST/dba_average.txt};
\addlegendentry{Average}

\addplot+[
  cyan, mark=none, line width=2pt, 
  smooth, 
  error bars/.cd, 
    y fixed,
    y dir=both, 
    y explicit
] table [x expr=\thisrow{x}, y=asr, col sep=comma] {data/main/fashionMNIST/dba_filterL2.txt};
\addlegendentry{FilterL2}

\addplot+[
  green, mark=none, line width=2pt, 
  smooth, 
  error bars/.cd, 
    y fixed,
    y dir=both, 
    y explicit
] table [x expr=\thisrow{x}, y=asr, col sep=comma] {data/main/fashionMNIST/dba_ex_noregret.txt};
\addlegendentry{Explicit\_No\_Regret}

\addplot+[
  black, mark=none, line width=1.6pt, 
  smooth, 
  error bars/.cd, 
    y fixed,
    y dir=both, 
    y explicit
] table [x expr=\thisrow{x}, y=asr, col sep=comma] {data/main/fashionMNIST/dba_mom_filterl2.txt};

\addplot+[
  gray, mark=none, line width=1.6pt, 
  smooth, 
  error bars/.cd, 
    y fixed,
    y dir=both, 
    y explicit
] table [x expr=\thisrow{x}, y=asr, col sep=comma] {data/main/fashionMNIST/dba_mom_ex_noregret.txt};

\addplot+[
  blue, dashed,  mark=none, line width=1.6pt, 
  smooth, 
  error bars/.cd, 
    y fixed,
    y dir=both, 
    y explicit
] table [x expr=\thisrow{x}, y=asr, col sep=comma] {data/main/fashionMNIST/dba_krum.txt};
\addlegendentry{Krum}

\addplot+[
  brown, dashed, mark=none, line width=1.6pt, 
  smooth, 
  error bars/.cd, 
    y fixed,
    y dir=both, 
    y explicit
] table [x expr=\thisrow{x}, y=asr, col sep=comma] {data/main/fashionMNIST/dba_trimmedmean.txt};
\addlegendentry{Trimmed Mean}

\addplot+[
  darkgreen, dashed, mark=none, line width=1.6pt, 
  smooth, 
  error bars/.cd, 
    y fixed,
    y dir=both, 
    y explicit
] table [x expr=\thisrow{x}, y=asr, col sep=comma] {data/main/fashionMNIST/dba_bulyankrum.txt};
\addlegendentry{Bulyan\_Krum}

\addplot+[
  orange, dashed, mark=none, line width=1.6pt, 
  smooth, 
  error bars/.cd, 
    y fixed,
    y dir=both, 
    y explicit
] table [x expr=\thisrow{x}, y=asr, col sep=comma] {data/main/fashionMNIST/dba_bulyantrm.txt};
\addlegendentry{Bulyan\_Trimmed\_Mean}

\addplot+[
  magenta, dashed, mark=none, line width=1.6pt, 
  smooth, 
  error bars/.cd, 
    y fixed,
    y dir=both, 
    y explicit
] table [x expr=\thisrow{x}, y=asr, col sep=comma] {data/main/fashionMNIST/dba_median.txt};
\addlegendentry{Median}

\addplot+[
  violet, dashed, mark=none, line width=1.6pt, 
  smooth, 
  error bars/.cd, 
    y fixed,
    y dir=both, 
    y explicit
] table [x expr=\thisrow{x}, y=asr, col sep=comma] {data/main/fashionMNIST/dba_bulyanmedian.txt};
\addlegendentry{Bulyan\_Median}

\addplot+[
  lime, dashed, mark=none, line width=1.6pt, 
  smooth, 
  error bars/.cd, 
    y fixed,
    y dir=both, 
    y explicit
] table [x expr=\thisrow{x}, y=asr, col sep=comma] {data/main/fashionMNIST/dba_history.txt};
\addlegendentry{\cite{karimireddy2021learning}}

\addplot+[
  teal, dashed, mark=none, line width=1.6pt, 
  smooth, 
  error bars/.cd, 
    y fixed,
    y dir=both, 
    y explicit
] table [x expr=\thisrow{x}, y=asr, col sep=comma] {data/main/fashionMNIST/dba_bucketing.txt};
\addlegendentry{\cite{karimireddy2020byzantine}}

\addplot+[
  olive, dashed, mark=none, line width=1.6pt, 
  smooth, 
  error bars/.cd, 
    y fixed,
    y dir=both, 
    y explicit
] table [x expr=\thisrow{x}, y=asr, col sep=comma] {data/main/fashionMNIST/dba_clustering.txt};

\legend{};

\end{axis}
\end{tikzpicture}}
         \caption{F-MNIST under DBA.}
         \label{fig:fashionmnist_dba_asr}
     \end{subfigure}
     \hfill
     \begin{subfigure}[b]{\figwidth}
         \centering
         \resizebox{\textwidth}{!}{\begin{tikzpicture}
\begin{axis}[
  set layers,
  grid=major,
  xmin=0, xmax=50,
  ytick align=outside, ytick pos=left,
  xtick align=outside, xtick pos=left,
  xlabel=\# Epoch,
  ylabel={Model Accuracy (\%)},
  enlarge y limits=0.05,
  legend style={at={(1.5,0.4)}, anchor=north, nodes={scale=0.75, transform shape}}]

\addplot+[
  red, dashed, mark=none, line width=1.6pt, 
  smooth, 
  error bars/.cd, 
    y fixed,
    y dir=both, 
    y explicit
] table [x expr=\thisrow{x}, y=acc, col sep=comma] {data/main/fashionMNIST/dba_average.txt};
\addlegendentry{Average}

\addplot+[
  cyan, mark=none, line width=2pt, 
  smooth, 
  error bars/.cd, 
    y fixed,
    y dir=both, 
    y explicit
] table [x expr=\thisrow{x}, y=acc, col sep=comma] {data/main/fashionMNIST/dba_filterL2.txt};
\addlegendentry{FilterL2}

\addplot+[
  green, mark=none, line width=2pt, 
  smooth, 
  error bars/.cd, 
    y fixed,
    y dir=both, 
    y explicit
] table [x expr=\thisrow{x}, y=acc, col sep=comma] {data/main/fashionMNIST/dba_ex_noregret.txt};
\addlegendentry{Explicit\_No\_Regret}

\addplot+[
  black, mark=none, line width=1.6pt, 
  smooth, 
  error bars/.cd, 
    y fixed,
    y dir=both, 
    y explicit
] table [x expr=\thisrow{x}, y=acc, col sep=comma] {data/main/fashionMNIST/dba_mom_filterl2.txt};

\addplot+[
  gray, mark=none, line width=1.6pt, 
  smooth, 
  error bars/.cd, 
    y fixed,
    y dir=both, 
    y explicit
] table [x expr=\thisrow{x}, y=acc, col sep=comma] {data/main/fashionMNIST/dba_mom_ex_noregret.txt};

\addplot+[
  blue, dashed,  mark=none, line width=1.6pt, 
  smooth, 
  error bars/.cd, 
    y fixed,
    y dir=both, 
    y explicit
] table [x expr=\thisrow{x}, y=acc, col sep=comma] {data/main/fashionMNIST/dba_krum.txt};
\addlegendentry{Krum}

\addplot+[
  brown, dashed, mark=none, line width=1.6pt, 
  smooth, 
  error bars/.cd, 
    y fixed,
    y dir=both, 
    y explicit
] table [x expr=\thisrow{x}, y=acc, col sep=comma] {data/main/fashionMNIST/dba_trimmedmean.txt};
\addlegendentry{Trimmed Mean}

\addplot+[
  darkgreen, dashed, mark=none, line width=1.6pt, 
  smooth, 
  error bars/.cd, 
    y fixed,
    y dir=both, 
    y explicit
] table [x expr=\thisrow{x}, y=acc, col sep=comma] {data/main/fashionMNIST/dba_bulyankrum.txt};
\addlegendentry{Bulyan\_Krum}

\addplot+[
  orange, dashed, mark=none, line width=1.6pt, 
  smooth, 
  error bars/.cd, 
    y fixed,
    y dir=both, 
    y explicit
] table [x expr=\thisrow{x}, y=acc, col sep=comma] {data/main/fashionMNIST/dba_bulyantrm.txt};
\addlegendentry{Bulyan\_Trimmed\_Mean}

\addplot+[
  magenta, dashed, mark=none, line width=1.6pt, 
  smooth, 
  error bars/.cd, 
    y fixed,
    y dir=both, 
    y explicit
] table [x expr=\thisrow{x}, y=acc, col sep=comma] {data/main/fashionMNIST/dba_median.txt};
\addlegendentry{Bulyan\_Median}

\addplot+[
  violet, dashed, mark=none, line width=1.6pt, 
  smooth, 
  error bars/.cd, 
    y fixed,
    y dir=both, 
    y explicit
] table [x expr=\thisrow{x}, y=acc, col sep=comma] {data/main/fashionMNIST/dba_bulyanmedian.txt};
\addlegendentry{Bulyan\_Median}

\addplot+[
  lime, dashed, mark=none, line width=1.6pt, 
  smooth, 
  error bars/.cd, 
    y fixed,
    y dir=both, 
    y explicit
] table [x expr=\thisrow{x}, y=acc, col sep=comma] {data/main/fashionMNIST/dba_history.txt};
\addlegendentry{\cite{karimireddy2021learning}}

\addplot+[
  teal, dashed, mark=none, line width=1.6pt, 
  smooth, 
  error bars/.cd, 
    y fixed,
    y dir=both, 
    y explicit
] table [x expr=\thisrow{x}, y=acc, col sep=comma] {data/main/fashionMNIST/dba_bucketing.txt};
\addlegendentry{\cite{karimireddy2020byzantine}}

\addplot+[
  olive, dashed, mark=none, line width=1.6pt, 
  smooth, 
  error bars/.cd, 
    y fixed,
    y dir=both, 
    y explicit
] table [x expr=\thisrow{x}, y=acc, col sep=comma] {data/main/fashionMNIST/dba_clustering.txt};

\legend{};

\end{axis}
\end{tikzpicture}}
         \caption{F-MNIST under DBA.}
         \label{fig:fashionmnist_dba_acc}
     \end{subfigure}
    \hfill
    \begin{subfigure}[b]{\figwidth}
         \centering
         \resizebox{\textwidth}{!}{\begin{tikzpicture}[spy using outlines={rectangle, magnification=2, width=6cm, height=2cm, connect spies}]
\begin{axis}[
  grid=major,
  xmin=0, xmax=50,
  ytick align=outside, ytick pos=left,
  xtick align=outside, xtick pos=left,
  xlabel=\# Epoch,
  ylabel={Model Accuracy (\%)},
  enlarge y limits=0.05
  ]

\addplot+[
  red, dashed, mark=none, line width=1.6pt, 
  smooth, 
  error bars/.cd, 
    y fixed,
    y dir=both, 
    y explicit
] table [x expr=\thisrow{x}, y=y, col sep=comma] {data/main/fashionMNIST/xie_average.txt};
\addlegendentry{Average}

\addplot+[
  cyan, mark=none, line width=2pt, 
  smooth, 
  error bars/.cd, 
    y fixed,
    y dir=both, 
    y explicit
] table [x expr=\thisrow{x}, y=y, col sep=comma] {data/main/fashionMNIST/xie_filterl2.txt};
\addlegendentry{filterL2}

\addplot+[
  green, mark=none, line width=2pt, 
  smooth, 
  error bars/.cd, 
    y fixed,
    y dir=both, 
    y explicit
] table [x expr=\thisrow{x}, y=y, col sep=comma] {data/main/fashionMNIST/xie_ex_noregret.txt};
\addlegendentry{Explicit\_No\_Regret}

\addplot+[
  black, mark=none, line width=1.6pt, 
  smooth, 
  error bars/.cd, 
    y fixed,
    y dir=both, 
    y explicit
] table [x expr=\thisrow{x}, y=y, col sep=comma] {data/main/fashionMNIST/xie_mom_filterl2.txt};

\addplot+[
  gray, mark=none, line width=1.6pt, 
  smooth, 
  error bars/.cd, 
    y fixed,
    y dir=both, 
    y explicit
] table [x expr=\thisrow{x}, y=y, col sep=comma] {data/main/fashionMNIST/xie_mom_ex_noregret.txt};

\addplot+[
  blue, dashed,  mark=none, line width=1.6pt, 
  smooth, 
  error bars/.cd, 
    y fixed,
    y dir=both, 
    y explicit
] table [x expr=\thisrow{x}, y=y, col sep=comma] {data/main/fashionMNIST/xie_krum.txt};
\addlegendentry{Krum}

\addplot+[
  brown, dashed, mark=none, line width=1.6pt, 
  smooth, 
  error bars/.cd, 
    y fixed,
    y dir=both, 
    y explicit
] table [x expr=\thisrow{x}, y=y, col sep=comma] {data/main/fashionMNIST/xie_trimmedmean.txt};
\addlegendentry{Trimmed Mean}

\addplot+[
  darkgreen, dashed, mark=none, line width=1.6pt, 
  smooth, 
  error bars/.cd, 
    y fixed,
    y dir=both, 
    y explicit
] table [x expr=\thisrow{x}, y=y, col sep=comma] {data/main/fashionMNIST/xie_bulyankrum.txt};
\addlegendentry{Bulyan\_Krum}

\addplot+[
  orange, dashed, mark=none, line width=1.6pt, 
  smooth, 
  error bars/.cd, 
    y fixed,
    y dir=both, 
    y explicit
] table [x expr=\thisrow{x}, y=y, col sep=comma] {data/main/fashionMNIST/xie_bulyantrimmedmean.txt};
\addlegendentry{Bulyan\_Trimmed\_Mean}

\addplot+[
  magenta, dashed, mark=none, line width=1.6pt, 
  smooth, 
  error bars/.cd, 
    y fixed,
    y dir=both, 
    y explicit
] table [x expr=\thisrow{x}, y=y, col sep=comma] {data/main/fashionMNIST/xie_median.txt};
\addlegendentry{Median}

\addplot+[
  violet, dashed, mark=none, line width=1.6pt, 
  smooth, 
  error bars/.cd, 
    y fixed,
    y dir=both, 
    y explicit
] table [x expr=\thisrow{x}, y=y, col sep=comma] {data/main/fashionMNIST/xie_bulyanmedian.txt};
\addlegendentry{Bulyan\_Median}

\addplot+[
  lime, dashed, mark=none, line width=1.6pt, 
  smooth, 
  error bars/.cd, 
    y fixed,
    y dir=both, 
    y explicit
] table [x expr=\thisrow{x}, y=y, col sep=comma] {data/main/fashionMNIST/xie_history.txt};
\addlegendentry{\cite{karimireddy2021learning}}

\addplot+[
  teal, dashed, mark=none, line width=1.6pt, 
  smooth, 
  error bars/.cd, 
    y fixed,
    y dir=both, 
    y explicit
] table [x expr=\thisrow{x}, y=y, col sep=comma] {data/main/fashionMNIST/xie_bucketing.txt};
\addlegendentry{\cite{karimireddy2020byzantine}}

\addplot+[
  olive, dashed, mark=none, line width=1.6pt, 
  smooth, 
  error bars/.cd, 
    y fixed,
    y dir=both, 
    y explicit
] table [x expr=\thisrow{x}, y=y, col sep=comma] {data/main/fashionMNIST/xie_clustering.txt};

\legend{};


\end{axis}
\end{tikzpicture}}
         \caption{F-MNIST under IMA.}
         \label{fig:fashionmnist_xie}
     \end{subfigure}
    
    \begin{subfigure}[t]{0.9\textwidth}
        \centering
        \resizebox{\textwidth}{!}{\newenvironment{customlegend}[1][]{%
    \begingroup
    \csname pgfplots@init@cleared@structures\endcsname
    \pgfplotsset{#1}%
}{%
    \csname pgfplots@createlegend\endcsname
    \endgroup
}%
\def\addlegendimage{\csname pgfplots@addlegendimage\endcsname}

\begin{tikzpicture}
\begin{customlegend}[
legend columns=3,
legend style={align=left,draw=none,column sep=8ex, font=\footnotesize},
legend entries={
    \textsc{Filtering}, 
    \textsc{No-Regret}, 
    \textsc{Bucketing-Filtering}, 
    \textsc{Bucketing-No-Regret}, 
    \textsc{Average}, 
    \textsc{Krum}, 
    \textsc{Trimmed Mean}, 
    \textsc{Median}, 
    \textsc{Bulyan Krum}, 
    \textsc{Bulyan Trimmed Mean}, 
    \textsc{Bulyan Median}, 
    \textsc{\citet{karimireddy2021learning}}, 
    \textsc{\citet{karimireddy2020byzantine}}, 
    \textsc{\citet{velicheti2021secure}},}]
        \addlegendimage{mark=none,cyan,line width=1pt}
        \addlegendimage{mark=none,green,line width=1pt}
        \addlegendimage{mark=none,black,line width=1pt}
        \addlegendimage{mark=none,gray,line width=1pt}
        \addlegendimage{mark=none,dashed,red,line width=1pt}
        \addlegendimage{mark=none,dashed,blue,line width=1pt}
        \addlegendimage{mark=none,dashed,brown,line width=1pt}
        \addlegendimage{mark=none,dashed,magenta,line width=1pt}
        \addlegendimage{mark=none,dashed,darkgreen,line width=1pt}
        \addlegendimage{mark=none,dashed,orange,line width=1pt}
        \addlegendimage{mark=none,dashed,violet,line width=1pt}
        \addlegendimage{mark=none,dashed,lime,line width=1pt}
        \addlegendimage{mark=none,dashed,teal,line width=1pt}
        \addlegendimage{mark=none,dashed,olive,line width=1pt}

\end{customlegend}
\end{tikzpicture}}
    \end{subfigure}
    
    \caption{Robust estimators' performance under attacks.}
    \label{fig:main_evaluation}
\end{figure*}

\subsection{Experiment Setup}
\label{subsec:eval-setup}

\noindent\textbf{Datasets, Models \& Runtime.} We evaluated the robust aggregators on two datasets, MNIST under CC BY-SA 3.0 License~\cite{lecun2010mnist} and F-MNIST under MIT License~\cite{xiao2017fashion}. The training and testing is by default as implemented in PyTorch.
%
%
%
The models used in most experiments are a convolutional network with two convolutional layers followed by two fully connected layers with ReLU.
In each round, each client runs 5 local epochs with batch size 10 before submitting the updated local model to the server.
%
All the experiments were conducted on Ubuntu18.04 LTS servers with 56 2.6GHz Intel Xeon CPUs, 252G RAM and 8
Geforce GTX 1080 Ti.
The code for evaluation is provided in \url{https://github.com/wanglun1996/secure-robust-federated-learning}.

\noindent\textbf{Attacks.} We chose the following attacks to evaluate the robust estimators: Krum Attack (KA)~\cite{fang2020local}, Trimmed Mean Attack (TMA)~\cite{fang2020local}, Model Poisoning Attack (MPA)~\cite{bhagoji2019analyzing}, Model Replacement Attack (MRA)~\cite{bagdasaryan2020backdoor}, Distributed Backdoor Attack (DBA)~\cite{xie2019dba} and Inner-Production Manipulation Attack (IMA)~\cite{xie2020fall}.
We report model accuracy for all attacks and also report attack success rate for backdoor attacks (\emph{i.e.} MPA, MRA and DBA.)

\noindent\textbf{Robust Aggregators.} We also chose nine other Byzantine-robust FL protocols as baselines: (1) Krum~\cite{blanchard2017machine}; (2) Trimmed Mean~\cite{yin2018byzantine}; (3) Median~\cite{yin2018byzantine}; (4) Bulyan Krum~\cite{mhamdi2018hidden}; (5) Bulyan Trimmed Mean~\cite{mhamdi2018hidden}; (6) Bulyan Median~\cite{mhamdi2018hidden}; (7) protocol in~\citet{karimireddy2021learning}; (8) protocol in~\citet{karimireddy2020byzantine}; (9)  protocol in~\citet{velicheti2021secure}.
As Sever~\cite{diakonikolas2019sever} requires much more rounds than other aggregators, we defer the evaluation of Sever to Appendix~\ref{sec:append-sever}.
For \textsc{Filtering}, \textsc{No-regret}, and the bucketing variants, it is computationally expensive to calculate the operator norm of the covariance matrix of the whole model.
Hence, we cut each layer of the model into several intervals and run the above robust aggregators on each interval.
For bucketing variants, we set the bucket size to 2.


\subsection{Evaluation Results}

\noindent\textbf{\textsc{Filtering}, \textsc{No-regret} and \textsc{Bucketing} Variants.} In this part, the data is randomly distributed to the $100$ federated clients, out of which $20$ are malicious.
For all attacks except DBA, we set the learning rate to $10^{-3}$ and the bounded variance as $10^{-5}$.
For DBA, we set the learning rate to $0.1$ and the bounded variance to $10^{-4}$ because DBA is launched on a pre-trained model.
For the proposed algorithms, the parameters are cut into intervals of size 1000.

\tikzset{font={\fontsize{15pt}{12}\selectfont}}
\captionsetup[sub]{skip=2pt}
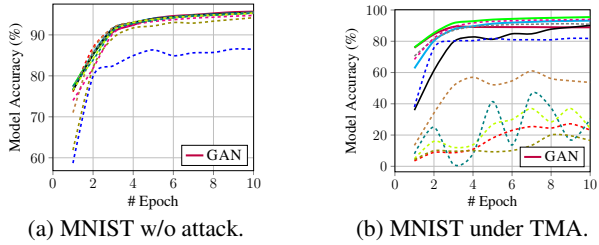
\begin{figure}[!h]
     \centering
     \begin{subfigure}[t]{0.21\textwidth}
         \centering
         \resizebox{\textwidth}{!}{\begin{tikzpicture}
\begin{axis}[
  grid=major,
  xmin=0, xmax=10,
  ytick align=outside, ytick pos=left,
  xtick align=outside, xtick pos=left,
  xlabel=\# Epoch,
  ylabel={Model Accuracy (\%)},
  enlarge y limits=0.05,
  legend pos=south east,
  ]

\addplot+[
  purple, mark=none, line width=2pt, 
  smooth, 
  error bars/.cd, 
    y fixed,
    y dir=both, 
    y explicit
] table [x expr=\thisrow{x}, y=y, col sep=comma] {data/gan/noattack_gan.txt};
\addlegendentry{\textsc{GAN}}

\addplot+[
  red, dashed, mark=none, line width=1.6pt, 
  smooth, 
  error bars/.cd, 
    y fixed,
    y dir=both, 
    y explicit
] table [x expr=\thisrow{x}, y=y, col sep=comma] {data/gan/noattack_average.txt};

\addplot+[
  cyan, mark=none, line width=2pt, 
  smooth, 
  error bars/.cd, 
    y fixed,
    y dir=both, 
    y explicit
] table [x expr=\thisrow{x}, y=y, col sep=comma] {data/gan/noattack_filterl2.txt};

\addplot+[
  green, mark=none, line width=2pt, 
  smooth, 
  error bars/.cd, 
    y fixed,
    y dir=both, 
    y explicit
] table [x expr=\thisrow{x}, y=y, col sep=comma] {data/gan/noattack_ex_noregret.txt};

\addplot+[
  black, mark=none, line width=1.6pt, 
  smooth, 
  error bars/.cd, 
    y fixed,
    y dir=both, 
    y explicit
] table [x expr=\thisrow{x}, y=y, col sep=comma] {data/gan/noattack_mom_filterl2.txt};

\addplot+[
  gray, mark=none, line width=1.6pt, 
  smooth, 
  error bars/.cd, 
    y fixed,
    y dir=both, 
    y explicit
] table [x expr=\thisrow{x}, y=y, col sep=comma] {data/gan/noattack_mom_ex_noregret.txt};

\addplot+[
  blue, dashed,  mark=none, line width=1.6pt, 
  smooth, 
  error bars/.cd, 
    y fixed,
    y dir=both, 
    y explicit
] table [x expr=\thisrow{x}, y=y, col sep=comma] {data/gan/noattack_krum.txt};

\addplot+[
  brown, dashed, mark=none, line width=1.6pt, 
  smooth, 
  error bars/.cd, 
    y fixed,
    y dir=both, 
    y explicit
] table [x expr=\thisrow{x}, y=y, col sep=comma] {data/gan/noattack_trimmedmean.txt};

\addplot+[
  magenta, dashed, mark=none, line width=1.6pt, 
  smooth, 
  error bars/.cd, 
    y fixed,
    y dir=both, 
    y explicit
] table [x expr=\thisrow{x}, y=y, col sep=comma] {data/gan/noattack_median.txt};

\addplot+[
  lime, dashed, mark=none, line width=1.6pt, 
  smooth, 
  error bars/.cd, 
    y fixed,
    y dir=both, 
    y explicit
] table [x expr=\thisrow{x}, y=y, col sep=comma] {data/gan/noattack_history.txt};

\addplot+[
  teal, dashed, mark=none, line width=1.6pt, 
  smooth, 
  error bars/.cd, 
    y fixed,
    y dir=both, 
    y explicit
] table [x expr=\thisrow{x}, y=y, col sep=comma] {data/gan/noattack_bucketing.txt};

\addplot+[
  olive, dashed, mark=none, line width=1.6pt, 
  smooth, 
  error bars/.cd, 
    y fixed,
    y dir=both, 
    y explicit
] table [x expr=\thisrow{x}, y=y, col sep=comma] {data/gan/noattack_clustering.txt};


\end{axis}
\end{tikzpicture}}
         \caption{MNIST w/o attack.}
         \label{fig:gan_noattack}
     \end{subfigure}
     \hfill
     \begin{subfigure}[t]{0.21\textwidth}
         \centering
         \resizebox{\textwidth}{!}{\begin{tikzpicture}
\begin{axis}[
  set layers,
  grid=major,
  xmin=0, xmax=10,
  ytick align=outside, ytick pos=left,
  xtick align=outside, xtick pos=left,
  xlabel=\# Epoch,
  ylabel={Model Accuracy (\%)},
  legend pos=south east,
  enlarge y limits=0.05
  ]

\addplot+[
  purple, mark=none, line width=2pt, 
  smooth, 
  error bars/.cd, 
    y fixed,
    y dir=both, 
    y explicit
] table [x expr=\thisrow{x}, y=y, col sep=comma] {data/gan/trimmedmean_gan.txt};
\addlegendentry{\textsc{GAN}}

\addplot+[
  red, dashed, mark=none, line width=1.6pt, 
  smooth, 
  error bars/.cd, 
    y fixed,
    y dir=both, 
    y explicit
] table [x expr=\thisrow{x}, y=y, col sep=comma] {data/gan/trimmedmean_average.txt};

\addplot+[
  cyan, mark=none, line width=2pt, 
  smooth, 
  error bars/.cd, 
    y fixed,
    y dir=both, 
    y explicit
] table [x expr=\thisrow{x}, y=y, col sep=comma] {data/gan/trimmedmean_filterl2.txt};

\addplot+[
  green, mark=none, line width=2pt, 
  smooth, 
  error bars/.cd, 
    y fixed,
    y dir=both, 
    y explicit
] table [x expr=\thisrow{x}, y=y, col sep=comma] {data/gan/trimmedmean_ex_noregret.txt};

\addplot+[
  black, mark=none, line width=1.6pt, 
  smooth, 
  error bars/.cd, 
    y fixed,
    y dir=both, 
    y explicit
] table [x expr=\thisrow{x}, y=y, col sep=comma] {data/gan/trimmedmean_mom_filterl2.txt};

\addplot+[
  gray, mark=none, line width=1.6pt, 
  smooth, 
  error bars/.cd, 
    y fixed,
    y dir=both, 
    y explicit
] table [x expr=\thisrow{x}, y=y, col sep=comma] {data/gan/trimmedmean_mom_ex_noregret.txt};

\addplot+[
  blue, dashed, mark=none, line width=1.6pt, 
  smooth, 
  error bars/.cd, 
    y fixed,
    y dir=both, 
    y explicit
] table [x expr=\thisrow{x}, y=y, col sep=comma] {data/gan/trimmedmean_krum.txt};

\addplot+[
  brown, dashed, mark=none, line width=1.6pt, 
  smooth, 
  error bars/.cd, 
    y fixed,
    y dir=both, 
    y explicit
] table [x expr=\thisrow{x}, y=y, col sep=comma] {data/gan/trimmedmean_trimmedmean.txt};

\addplot+[
  magenta, dashed, mark=none, line width=1.6pt, 
  smooth, 
  error bars/.cd, 
    y fixed,
    y dir=both, 
    y explicit
] table [x expr=\thisrow{x}, y=y, col sep=comma] {data/gan/trimmedmean_median.txt};

\addplot+[
  lime, dashed, mark=none, line width=1.6pt, 
  smooth, 
  error bars/.cd, 
    y fixed,
    y dir=both, 
    y explicit
] table [x expr=\thisrow{x}, y=y, col sep=comma] {data/gan/trimmedmean_history.txt};

\addplot+[
  teal, dashed, mark=none, line width=1.6pt, 
  smooth, 
  error bars/.cd, 
    y fixed,
    y dir=both, 
    y explicit
] table [x expr=\thisrow{x}, y=y, col sep=comma] {data/gan/trimmedmean_bucketing.txt};

\addplot+[
  olive, dashed, mark=none, line width=1.6pt, 
  smooth, 
  error bars/.cd, 
    y fixed,
    y dir=both, 
    y explicit
] table [x expr=\thisrow{x}, y=y, col sep=comma] {data/gan/trimmedmean_clustering.txt};


\end{axis}
\end{tikzpicture}}
         \caption{MNIST under TMA.}
         \label{fig:gan_krum}
     \end{subfigure}
    \caption{\textsc{GAN}'s performance under attack. Please refer to \F~\ref{fig:main_evaluation} for the complete legend.}
    \vspace{-15pt}
    \label{fig:gan}
\end{figure}

As shown in \F~\ref{fig:main_evaluation},
%
we observe that \textsc{Filtering} and \textsc{No-regret} achieve consistently good performance under all the chosen attacks, while all the other aggregators fail on at least one attack.
For instance, Krum and Bulyan Krum's accuracy drops on KA and typically cannot achieve as good accuracy as other estimators; Trimmed Mean fails on TMA and MPA; Median, Bulyan Trimmed Mean and Bulyan Median fail on MPA.
We also observe that the Bucketing variants do not perform as well as \textsc{Filtering} and \textsc{No-regret}.
We conjecture this is due to their low breaking point.


\noindent\textbf{\textsc{GAN}.}
In this part, we evaluate the performance of the \textsc{GAN}.
Because \textsc{GAN} needs a large amount of training data, we increase the number of clients to $1000$, out of which $200$ are malicious.
The learning rate is $0.1$.
For the proposed algorithms, we set the bounded variance to $10^{-3}$ and the 
parameters are cut into intervals of size 100.
We also do not evaluate the Bulyan aggregators because of their impractically long running time (\emph{e.g.} $\sim$ 200 hours per round for Bulyan Krum).

As shown in \F~\ref{fig:gan}, \textsc{GAN} does not achieve as good performance as \textsc{Filtering} and \textsc{No-regret}.
During the experiments, we observe that although \textsc{GAN} is usually faster in running time than \textsc{Filtering} and \textsc{No-regret}, it can be sensitive to hyper-parameters. This makes it hard to deploy in real-world applications compared to other proposed methods.
Consistent with \F~\ref{fig:main_evaluation}, \textsc{Krum}, \citet{karimireddy2020byzantine} and \citet{karimireddy2021learning} cannot achieve as good performance as other aggregators when there is no attack, and this phenomenon becomes more significant when the client number is larger.
\vspace{-5pt}
\section{Conclusion}
\vspace{-5pt}

In this paper, we propose Byzantine-robust FL protocols with optimal statistical rate and privacy guarantees.
Furthermore, we show on benchmark data that the proposed protocols achieve consistently good performance under different attacks.
%
%
This work leaves two main open questions: Can one design simple and practically implementable Byzantine-robust FL protocol with optimal statistical rate and  linear computational complexity? 
Besides these two properties, can a Byzantine-robust FL protocol also achieve high breakdown point simultaneously? 

\section{Acknowledgements}
Banghua Zhu and Jiantao Jiao were partially supported by NSF Grants IIS-1901252 and CCF-1909499, Michael I. Jordan was partially supported by NSF Grants IIS-1901252.
Qi Pang and Shuai Wang were supported in part by RGC RMGS under the contract RMGS20CR18 and RMGS20CR19.
Lun Wang and Dawn Song were partially supported by DARPA contract \#N66001-15-C-4066, the Center for LongTerm Cybersecurity, and Berkeley Deep Drive. 
Any opinions, findings, conclusions, or recommendations expressed in this material are those of the authors, and do not necessarily reflect the views of the sponsors.

\bibliographystyle{plainnat}
\bibliography{di}

\newpage
\appendix
\onecolumn
\aistatstitle{Supplementary Material for Byzantine-Robust Federated Learning with Optimal Statistical Rates}
\section{Related Work}
\label{sec:related_work}

We now review the client-side privacy vulnerabilities and Byzantine-malicious attacks in FL.
Existing defenses and their limitations are also discussed to motivate the present work.
Other attacks and defenses in FL are covered in this comprehensive surveys~\cite{kairouz2019advances}.



\smallskip
\noindent \textbf{Byzantine Malicious Clients.}~Byzantine-robust aggregation has
drawn enormous attention over the past few years due to the emergence of various
distributed attacks in FL.
\citet{fang2020local} formalize the attack as an optimization problem and
successfully migrate the data poisoning attack to FL.
The proposed attacks even work under Byzantine-robust FL.
\citet{sun2021data} manage to launch data poisoning attacks on the multi-task
FL framework. 
\citet{bhagoji2019analyzing} and~\citet{bagdasaryan2020backdoor} manage to
insert backdoors into the model via local model poisoning or local model
replacement.
\citet{xie2019dba} propose to segment one backdoor into several parts and
insert it into the global model.
\citet{chen2020backdoor} and~\citet{zizhanzhenglearning} migrate backdoor
attacks to federated meta-learning and federated reinforcement learning,
respectively.
Meanwhile, \citet{sun2019can} show that norm clipping and ``weak'' differential
privacy mitigate backdoor attacks in FL without impairing overall performance.
\citet{wang2020attack} refute this claim and illustrate that
robustness to backdoors requires model robustness to adversarial examples, an
open problem widely regarded to be difficult.

\smallskip
\noindent \textbf{Byzantine-Robust Protocols.}~A variety of Byzantine-robust FL
protocols are proposed to defend against malicious clients.
Krum~\cite{blanchard2017machine} selects and averages the subset of updates
that have a sufficient number of close neighbors.
\citet{yin2018byzantine} uses robust estimators like trimmed mean or median to estimate the gradient, and claims to  achieve order-optimal statistical error rates.
\citet{fung2018mitigating, alistarh2018byzantine} proposes a similar robust estimator relying on a
robust secure aggregation oracle based on the geometric median.
\citet{yin2019defending} proposes to use robust mean estimators to defend against
saddle point attacks.
\citet{fung2020limitations} studies Sybil attacks in FL and propose a defense based on the diversity
of client updates. 
\citet{ozdayipreventing} designs a defense against
backdoor attacks in FL by adjusting the server-side learning rate.
\citet{mhamdi2018hidden} points out that Krum, trimmed mean, and median all suffer
from $\mathcal{O}(\sqrt{d})$ ($d$ is the model size) estimation error and
propose a general framework Bulyan for reducing the error to $\mathcal{O}(1)$.

\smallskip
\noindent \textbf{Client Privacy Leakage and Mitigation.}~Inference attacks
against centralized learning~\cite{shokri2017membership}
aim to infer the private information of the model training data.
\cite{wang2019beyond} explore the feasibility of recovering user privacy from a
malicious server in FL.
\cite{nasr2019comprehensive} show that a malicious server can perform highly
accurate membership inference attacks against clients.

\noindent\textbf{Orchestrating Robustness and Privacy.} Several recent works~\cite{so2020byzantine,he2020secure} attempt to address the challenge through the use of secure multi-party computation (MPC).
\cite{so2020byzantine} propose to run multi-party distance-based filtering
among clients to remove potentially malicious updates.
This, however, needs clients to be online consistently, which is impractical for cross-device FL and also incurs significant communication cost.
\cite{he2020secure} propose a two-server protocol to provide bi-directional
protection, but in real-world use scenarios, two non-colluding servers are
uncommon.
Due to the lack of universal and effective two-way protection, user trust in FL systems is significantly eroded, preventing them from being employed in a wide variety of security-related applications such as home monitoring and autonomous driving.

\subsection{Comparison with existing Byzantine-Robust Protocols}
\smallskip
\noindent \textbf{Comparison with~\citet{yin2018byzantine}.}
\citet{yin2018byzantine} first analyzes the statistical rate for Byzantine-robust distributed learning when using coordinate-wise median or trimmed-mean as robust gradient estimators. They show that under the assumption of bounded variance $\mathbb{E}_{p^\star}[\|\vecg - \mathbb{E}[\vecg]\|_2^2]\leq \tilde \sigma^2$ and coordinate-wise bounded skewness (or coordinate-wise sub-exponential) on the gradient distribution, one can achieve a statistical error of $\bigo(\tilde \sigma (\epsilon/\sqrt{n} + \sqrt{d/mn}))$. Notably, under our assumption of bounded spectral norm of covariance in Assumption~\ref{asm:bounded-variance}, one has $\tilde \sigma^2 = \sigma^2 d$ in the worst case, and their rate becomes $\bigo(\sigma (\epsilon\sqrt{d} + \sqrt{d^2/mn}))$, which is $\sqrt{d}$ worse than the lower bound. Although the rate achieved is better than ours in terms of the dependence on $\epsilon$, this is from the strong assumption of coordinate-wise bounded skewness (or coordinate-wise sub-exponential). With only bounded covariance assumption as in our paper, coordinate-wise median or coordinate-wise trimmed mean must suffer a rate of $\sigma\sqrt{\epsilon d}$ (see e.g.~\cite{steinhardt2019lecture}), and thus is  far from the optimal  rate $\sigma\sqrt{\epsilon}$ achieved by the robust estimators proposed in this paper.

\smallskip
\noindent \textbf{Comparison with~\citet{karimireddy2020byzantine} and~\citet{karimireddy2021learning}}

Similar to \citet{yin2018byzantine}, the analysis of \citet{karimireddy2020byzantine} and~\citet{karimireddy2021learning} are based on the assumption of bounded trace of the covariance, leading to an extra $\sqrt{d}$ factor under our spectral norm bound assumption. \citet{karimireddy2020byzantine} has shown that the dimension dependence  of coordinate-wise median combined with bucketing is worse than that of Krum and geometric median. However, as we have discussed in the introduction, coordinate-wise median, coordinate-wise trimmed mean and geometric median all have an error of at least $\Omega(\sigma\sqrt{\epsilon d})$ under our setting. This motivates the application of modern high dimensional robust estimators, as we proposed in this paper. 

On the other hand, the robust aggregation algorithm with momentum is shown to achieve a rate $\mathcal{O}(\sqrt{\frac{d\sigma^2}{T}\cdot 
(\frac{1}{m}+\epsilon)}$, which seems to contradict our lower bound in Appendix~\ref{sec:lower_bound} since the lower bound suggests that the rate does not vanish as $T$ goes to infinity. The reason for this inconsistency is that ~\citet{karimireddy2021learning} assumes that the gradient in each round is mutually independent and the malicious clients are fixed. However, this assumption does not  hold in our setting when each client has $n$ fixed samples
and computes gradients out of (a mini-batch of) the samples. Intuitively, the assumption of independent gradient holds when each client gets fresh sample in a new iteration, thus can be approximately viewed as our parameter $n$. Furthermore, the assumption of independent gradient avoids the union bound argument in our analysis and~\citet{yin2018byzantine}, thus an extra $\sqrt{d}$ factor is avoided in the rate of~\citet{karimireddy2021learning}.

Furthermore, ~\citet{karimireddy2020byzantine} also proposes a similar idea of bucketing as pre-processing from a different motivation. In our analysis, bucketing is necessary to remove the logarithmic factor and achieve exactly tight rate in homogeneous setting. In~\citet{karimireddy2020byzantine}, it is mainly motivated by reducing the variance in heterogeneous data. 

\smallskip
\noindent \textbf{Comparison with ByzantineSGD~\cite{alistarh2018byzantine} and SafeguardSGD~\cite{allen2020byzantine}}

ByzantineSGD~\cite{alistarh2018byzantine} and SafeguardSGD~\cite{allen2020byzantine} analyze the rate of convergence for non-strongly convex and non-convex objectives via maintaining a set of good machines and detecting malicious machines in each round. However, their results are based on a stronger assumption that the norm of the difference between the gradient and its mean $\|\vecg - \mathbb{E}[\vecg]\|$ is always bounded. 
Thus the result is not comparable to our setting, which only assumes the second central moment is bounded.

\smallskip
\noindent \textbf{Comparison with Draco~\cite{chen2018draco}, BULYAN~\cite{mhamdi2018hidden} and Krum~\cite{blanchard2017machine}}

Draco~\cite{chen2018draco} and BULYAN~\cite{mhamdi2018hidden} state to
provide dimension-free estimation error for Byzantine-Robust Federated Learning.
However, Draco is incompatible with
FL as it requires redundant updates from each worker.
Bulyan
are based on much stronger assumptions than other contemporary works.
When the assumptions are
relaxed to the common case, Bulyan estimation errors still
scale up with the square root of the model size.
In particular, Bulyan assumes that the expectation of the distance between two benign updates is bounded by a constant $\sigma_1$, while Krum assumes that the distance is bounded by $\sigma_2\sqrt{d}$. We can easily see that if $\sigma_1=\sigma_2\sqrt{d}$, Bulyan falls back to the same order of estimation error as Krum. 
As a result,  the protocols fail to achieve near-optimal statistical rate under the only assumption of bounded operator norm of covariance  that can be used to mitigate
Byzantine adversaries. %
The proposed protocols overcomes limitations of existing Byzantine-robust FL protocols by
employing and calibrating well-established robust mean estimators~\cite{diakonikolas2016robust,diakonikolas2017being, diakonikolas2019nearly, diakonikolas2020outlier, steinhardt2017certified,  steinhardt2018resilience,  zhu2019deconstructing, zhu2019generalized,  zhu2022robust} 
 in FL scenarios.

\smallskip
\noindent \textbf{Comparison with~\citet{velicheti2021secure}}

The protocol in~\citet{velicheti2021secure} studies the combination of secure aggregation and Byzantine-robust estimators, similar to our bucketing idea. However, for the theoretical analysis, the authors make the assumption of bounded dissimilarity of the true gradient and bounded trace norm of the covariance of the gradients, which is the same as that in~\cite{yin2018byzantine} and thus stronger than our spectral norm bound. The paper also assumes a performance bound on the performance of aggregation algorithm, which does not have dependence on the sample size. Thus there is no sample complexity analyzed, but only the convergence rate.


\section{Robust Estimation Subroutines}\label{app:robust_guarantee}

We begin with introducing necessary notations and settings in robust mean estimation. Assume that  we observe $m$ samples from true distribution $\vecx_i\sim p^\star, i\in[m]$. Denote the true empirical distribution as $\hat p_m^\star = \frac{1}{m} \sum_{i\in[m]}\delta_{\vecx_i}$. The adversary observes the samples and may add, delete, or modify at most $\epsilon m$ samples, resulting in a corrupted dataset $\mathcal{D}_m$. We let the resulting empirical distribution be $\hat p_m$, which satisfies that $\TV(\hat p_m^\star, \hat p_m)\leq \epsilon$. Assume that the true distribution satisfies $\|\Cov_{p^\star}(\vecx)\|_2 \leq \sigma^2$. 
Let $\mu_q = \mathbb{E}_q[\vecx]$ be the mean of distribution $q$. We define  the quasi-gradient for $\|\Cov_q(\vecx)\|$ with respect to distribution $q$ as $g(q; \vecx) = (\vecv^\top(\vecx-\mu_q))^2$, where $\vecv\in\argmax_{\|v\|\leq 1} \bE_q[(\vecv^\top(\vecx-\mu_q))^2]$ is any of the supremum-achieving direction. 

\subsection{Algorithm Descriptions}\label{subsec:alg_description}

Now we are ready to present the three algorithms we used in the paper, namely \textsc{no-regret algorithm}~\cite{hopkins2020robust, zhu2021robust}, \textsc{filtering algorithm}~\cite{ diakonikolas2017being, li2018principled, steinhardt2018robust,  zhu2021robust} and \textsc{GAN} based algorithm~\cite{gao2018robust,gao2020generative, zhu2022robust}.
\begin{algorithm}[!htbp]
\centering
\caption{\textsc{No-regret algorithm} ($ \mathcal{D}_m, \epsilon, \sigma^2, \xi, \eta$)}
        \begin{algorithmic}
\STATE Input: corrupted dataset $\mathcal{D}_m = \{\vecx_1, \vecx_2,\cdots, \vecx_m\}$, fraction of corrupted samples $\epsilon$, spectral norm bound of covariance $\sigma^2$, threshold for termination $\xi$, step size constant $\eta\in(0, 1)$.
\STATE Set $d_{i, j}=\|\vecx_i-\vecx_j\|_2$ for all $i, j\in[m]$, remove all $\vecx_i$ with $|\{j\in[m]:d_{i,j}>\Omega(\sqrt{d \log(m)})\}|>2\epsilon m$. Let $\vecx_1', \vecx_2', \cdots, \vecx_{m'}'$ be the remaining samples.
Initialize uniform distribution on the remaining samples $q_{i}^{(0)} = 1/m'$, $i \in [m']$.
\FOR{$k = 0, 1, \dots$}
\IF {$\|\Cov_{q^{(k)}}(\vecx)\|_2\leq   \xi$}
    \STATE Return $\mathbb{E}_{q^{(k)}}[\vecx] = \sum_{i=1}^n q_i^{(k)}\vecx_i$
    \ELSE
    \STATE Compute $g_i^{(k)} = g(q^{(k)};\vecx_i)$,  $\tilde q_{i}^{(k+1)} = q_{i}^{(k)} \cdot (1- \frac{\eta\epsilon}{2\sigma^2d}\cdot g_i^{(k)} )$ for all $i\in[m]$
    \STATE Update $q^{(k+1)} = \argmin_{q\in\Delta_{m', \epsilon}} \KL(q || \tilde q^{(k+1)})$, where $\Delta_{m',\epsilon}=\{q\mid \sum q_i=1, q_i\leq \frac{1}{(1-\epsilon)m'}\}$. 
    \ENDIF
\ENDFOR
\end{algorithmic}
\label{algo:explicit_lowregret}
\end{algorithm}

\begin{algorithm}[!t]
\centering
\caption{\textsc{Filtering algorithm} ($\mathcal{D}_m, \xi$)}
        \begin{algorithmic}
\STATE Input: corrupted dataset $\mathcal{D}_m = \{\vecx_1, \vecx_2,\cdots, \vecx_m\}$, threshold for termination $\xi$.
\STATE Initialize $q_{i}^{(0)} = 1/m$, $i \in [m]$.
\FOR{$k = 0, 1, \dots$}
\IF {$\|\Cov_{q^{(k)}}(\vecx)\|_2\leq  \xi$}
    \STATE Return $\mathbb{E}_{q^{(k)}}[\vecx] = \sum_{i=1}^n q_i^{(k)}\vecx_i$
    \ELSE
    \STATE Compute $g_i^{(k)} = g(q^{(k)};\vecx_i)$,  $\tilde q_{i}^{(k+1)} = q_{i}^{(k)} \cdot (1- \frac{g_i^{(k)}}{\max_{j\in[m]}g_j^{(k)}} )$ for all $i\in[m]$.
    \STATE Update $q^{(k+1)} = \mathsf{Proj}^{KL}_{\Delta_{m}}(\tilde q^{(k+1)}) = \tilde q^{(k+1)}/ \sum_{i\in [m]} \tilde q^{(k+1)}_i $. 
    \STATE Remove samples with $q^{(k+1)}=0$.
    \ENDIF
\ENDFOR

\end{algorithmic}
\label{algo:filtering}
\end{algorithm}

\begin{algorithm}[!t]
\centering
\caption{\textsc{GAN} ($\mathcal{D}_m, \sigma^2$)}
        \begin{algorithmic}
\STATE Input: corrupted dataset $\mathcal{D}_m = \{\vecx_1, \vecx_2,\cdots, \vecx_m\}$, spectral norm bound of covariance $\sigma^2$.
\STATE Let $\hat p_m = \frac{1}{m}\sum_{i=1}^m \delta_{\vecx_i}$ be the empirical distribution of dataset $\mathcal{D}_m$.
\STATE Find $q = \argmin_{q\in\mathcal{G}(\sigma^2)} A(q, p_m)$, where  $\mathcal{G}(\sigma^2)=\{q\mid \|\Sigma_q\|_2\leq \sigma^2\}$, and 
\[   A(p, q) =&  \sup_{\|w\|_1\leq 1, \|v_j\|_2\leq 1, t_j\in\mathbb{R}} \Bigg|\mathbb{E}_{p}\left[ \mathsf{sigmoid}\left(\sum_{j\leq l} \vecw_j\mathsf{sigmoid}(\vecv_j^\top \vecx+t)\right)\right] \nonumber \\ 
    &-  \mathbb{E}_{q}\left[ \mathsf{sigmoid}\left(\sum_{j\leq l} \vecw_j\mathsf{sigmoid}(\vecv_j^\top \vecx+t)\right)\right]\Bigg|\]
\end{algorithmic}
\label{algo:GAN}
\end{algorithm}

We remark here that in Algorithm~\ref{algo:explicit_lowregret}, one needs to first naively filter some samples so that the maximum distances between samples are bounded, and then proceed with the multiplicative weights update method where learning rate is inverse proportional to the square of  maximum distances between samples. This naive filtering procedure can be replaced with any heuristic-based methods, as long as the maximum distances between samples are controlled. The projection step in Algorithm~\ref{algo:explicit_lowregret} can be done within $O(m)$ time. For Algorithm~\ref{algo:GAN}, the distance $A(p, q)$ can be optimized by a Generative Adversarial Network. One can replace the discriminator with any deep neural network with sigmoid activation. In order to guarantee that $q\in\mathcal{G}=\{q \mid \|\Sigma_q\|_2\leq \sigma^2\}$, one may either add a regularization term in the discriminator, or constrain the norm of parameters in the generator. 

\subsection{Guarantees}
For all the three robust estimators provided in the paper (Algorithm~\ref{algo:explicit_lowregret},~\ref{algo:filtering} and~\ref{algo:GAN}), the resulting guarantee is that one can find an estimator with estimation error controlled by $O(\sqrt{\epsilon}+\sqrt{d/m})$ (or $O(\sqrt{\epsilon+\sqrt{d/m}})$ for Algorithm~\ref{algo:GAN}).

\begin{theorem}\label{thm:robust_guarantee}
Assume that $\|\Cov_{p^\star}(\vecx)\|\leq \sigma^2$ and that we are given a corrupted dataset $\mathcal{D}_m$ from where at most $\epsilon$ fraction of samples are corrupted by adversary, the rest of samples follow the distribution of $p^\star$. Let 
\begin{align}
    \mu_1 &= \textsc{No}\text{-}\textsc{regret}\left(\mathcal{D}_m, \epsilon,  \sigma^2,  \Big(\frac{2\eta+7}{3(1 - (6+2\eta)\epsilon)}\Big)^2\cdot (1+\frac{d\log(d/\delta)}{m\epsilon})\cdot \sigma^2, \eta\right)\\
    \mu_2 &= \textsc{Filtering}\left(\mathcal{D}_m,  \frac{2(1-\epsilon)}{(1-2\epsilon)^2} \cdot(1+\frac{d\log(d/\delta)}{m\epsilon})\cdot\sigma^2\right) \\
\mu_3 &= \textsc{GAN}\left(\mathcal{D}_m, \sigma^2\right) \\
    \mu_4 &= \textsc{No}\text{-}\textsc{regret}\left(\mathcal{D}_m, \sigma^2, \eta, \frac{C_1}{(1-C_2(\epsilon+\log(1/\delta)/n)^2)}\cdot\left(1+\frac{d\log(d)+\log(1/\delta)}{m\epsilon}\right)\cdot\sigma^2\right) \\
    \mu_5 &= \textsc{Filtering}\left(\mathcal{D}_m,  \frac{C_3}{(1-C_4(\epsilon+\log(1/\delta)/n)^2)} \cdot\left(1+\frac{d\log(d)+\log(1/\delta)}{m\epsilon}\right)\cdot\sigma^2\right)   
\end{align}
We know that each of the following holds with probability at least $1-\delta$:
\begin{align*}
    \|\mu_{1}-\mu_{p^\star}\|&\leq \bigo\left(\sigma\cdot \left(\frac{\sqrt    \epsilon}{1-(6+2\eta)\epsilon} + \sqrt{\frac{d\log(d/\delta)}{m}}\right)\right), \\
    \|\mu_{2}-\mu_{p^\star}\|&\leq \bigo\left(\sigma\cdot\left(\frac{\sqrt{\epsilon}}{1-2\epsilon} + \sqrt{\frac{d\log(d/\delta)}{m}}\right)\right), \\
    \|\mu_{3}-\mu_{p^\star}\|&\leq \bigo\left(\frac{\sigma}{1-C_5\epsilon}\cdot \left({\sqrt{\epsilon+\sqrt{(d+\log(1/\delta))/m}}} \right)\right), \\
     \|\mu_{4}-\mu_{p^\star}\|&\leq \bigo\left(\frac{\sigma}{1-C_6(\epsilon+\log(1/\delta)/m)}\cdot \left({\sqrt{\epsilon}}+ \sqrt{\frac{d\log(d)+\log(1/\delta)}{m}}\right)\right), \\
         \|\mu_{5}-\mu_{p^\star}\|&\leq \bigo\left(\frac{\sigma}{1-C_7(\epsilon+\log(1/\delta)/m)}\cdot \left({\sqrt{\epsilon}}+ \sqrt{\frac{d\log(d)+\log(1/\delta)}{m}}\right)\right), 
\end{align*}
where $C_i>2$ are some universal constants.
\end{theorem}
\begin{proof}
The guarantee for $\mu_1$ comes from a combination of the statistical analysis (see e.g. \citet[Theorem G.3]{zhu2019generalized}, \citet{diakonikolas2017being}, \cite{steinhardt2017certified}) and computational analysis (see e.g.~\citet[Theorem 4.1]{zhu2021robust}). The guarantee for $\mu_1$ comes from a combination of the same statistical analysis and a different computational analysis (see~\citet[Theorem 4.2]{zhu2021robust}). The guarantee for $\mu_3$ comes from~\citet[Corollary 1]{zhu2022robust}. The guarantee for $\mu_4$ and $\mu_5$ comes from~\citet[Proposition 1.6]{diakonikolas2020outlier}.
\end{proof}

\revision{As a sharp contrast, \textsc{coordinate-wise median}, \textsc{coordinate-wise trimmed mean} and \textsc{geometric median} have an error of at least $\Omega(\sigma\sqrt{\epsilon d})$ under the same assumption, which can be highly sub-optimal when the dimension is high. Following our lower bound analysis in the next section, the lower bound for robust mean estimation also directly results in a lower bound of $\Omega({\sigma\sqrt{\epsilon d/ n}})$ for the performance of the FL protocols using either of the three algorithms as aggregation rules.  }

\section{Lower Bound}\label{sec:lower_bound}

Now we present a formal lower bound for the Federated Learning setting using the lower bound for robust mean estimation. Let $\mathcal{F}$ be the sets of loss functions $\ell$ and distribution $p^\star$ that satisfies Assumption~\ref{asm:smoothness}, ~\ref{asm:bounded-variance} and ~\ref{asm:convexity}(a). Let $P: \mathcal{W}\times \mathbb{R}^{m\times d} \mapsto \mathcal{W}$ be any FL protocol which  receive the current parameter $\vecw\in\mathcal{W}$ and corresponding gradient information for $m$ clients under the parameter $\vecw$, and outputs an updated  parameter $\vecw'\in\mathcal{W}$ according to the observation. We denote the set of such protocols as $\mathcal{P}$. 
For any $(f, p^\star)\in\mathcal{F}$, we let $\vecw^\star(f, p^\star) = \argmin_{\vecw\in\mathcal{W}} \mathbb{E}_{\vecz\sim p^\star}[f(\vecw;\vecz)]$ be the minimizer of the loss function. Let $\mathcal{D}_{m, n, T}({P})$ be the set of all possible distributions of observations we may see when executing $P\in \mathcal{P}$ after $T$ rounds. At round $t$, 
conditioned on the current parameter $\vecw^t$, the distribution of observed distribution is within $\epsilon$-$\TV$ distance to the true distribution of $\gradF_i(\vecw)$ due to possible corruptions by Byzantine workers. Let $\hat \vecw^T: \mathcal{P}\times \mathcal{D}_{m, n, T}(\cdot)\mapsto \mathcal{W}$ be the output of the protocol $P\in\mathcal{P}$ given all the observations from $\mathcal{D}_{m, n, T}({P})$. 
\begin{theorem}[Lower Bound]\label{thm:lower_bound}
For any  protocol with infinite computation power, there exists some  loss function $f$ and sample distribution $p^\star$ that satisfies Assumption~\ref{asm:smoothness}, ~\ref{asm:bounded-variance} and ~\ref{asm:convexity}(a), such that  the output of the protocol $\hat \vecw$ incurs a loss that is at least $ \Omega(\sqrt{\frac{\epsilon}{n}+\frac{d}{mn}})$. Formally, we have for any $T>0$, with constant probability,
\begin{align*}
    \inf_{P\in\mathcal{P}}\sup_{(f, p^\star)\in\mathcal{F}, \hat p\in\mathcal{D}_{m, n, T}({P})} \|\vecw^\star(f, p^\star)-\hat \vecw^T(P, \hat p)\|_2 = \Omega\left(\sigma\sqrt{\frac{\epsilon}{n}+\frac{d}{mn}}\right).
\end{align*}
\end{theorem}
\begin{proof}
Consider the set of problems of mean estimation under $\ell_2$ norm, where we fix $f(\vecw; \vecz) = \|\vecw-\vecz\|_2^2$, and let $p^\star$ to be any distribution with bounded covariance, i.e. we take $\mathcal{F}_{\sigma^2}=\{(f, p^\star) \mid f(\vecw; \vecz) = \|\vecw-\vecz\|_2^2, \|\Cov_{p^\star}(\vecz)\|_2\leq \sigma^2\}$. In this case, we know that $\vecw^\star(f, p^\star) = \mathbb{E}_{p^\star}[\vecz]$, and the gradient from each client is $\gradF_i(\vecw) = \frac{2\sum_{i=1}^n (\vecw-\vecz_i)}{n} = 2(\vecw - \frac{\sum_{j=1}^n \vecz_{ij}}{n})$. Consider another protocol $\mathcal{P}'$ where the server can directly query the mean $\frac{\sum_{j=1}^n \vecz_{ij}}{n} $ from each honest client. One can see that for the fixed squared loss $f$, any protocol in $\mathcal{P}$ can be converted into a protocol in $\mathcal{P}'$ since the mean $\frac{\sum_{j=1}^n \vecz_{ij}}{n} $ can be  directly computed from the gradient $\gradF_i(\vecw)$ on any given $\vecw$  (and thus any multiple-round local model updates in FL can be also directly computed from the mean). Similarly, we can also convert any protocol in $\mathcal{P}'$ to a protocol in $\mathcal{P}$. Thus the two protocols are equivalent, and we have
\begin{align*}
    \inf_{P\in\mathcal{P}}\sup_{\substack{(f, p^\star)\in\mathcal{F}_{\sigma^2},\\ \hat p\in\mathcal{D}_{m, n, T}({P})}} \|\vecw^\star(f, p^\star)-\hat \vecw^T(P, \hat p)\|_2  =  \inf_{P\in\mathcal{P}'}\sup_{\substack{(f, p^\star)\in\mathcal{F}_{\sigma^2}, \\\hat p\in\mathcal{D}_{m, n, T}({P})}} \|\vecw^\star(f, p^\star)-\hat \vecw^T(P, \hat p)\|_2. 
\end{align*}
In this case, we can see that for honest clients, the output in each round will always be the mean $\frac{\sum_{j=1}^n \vecz_{ij}}{n} $. We further restrict the behavior of adversary by considering a fixed set of adversary, and force the adversary to output the same mean in each round. Thus we know that this minimax rate can be further lower bounded by that of single-round robust mean estimation problem where the queried sample from honest client is an average of $n$ samples from the true distribution $p^\star$. Formally, let $p^\star_n$ be the distribution of $\frac{\sum_{j=1}^n \vecz_{ij}}{n} $ where each $\vecz_{ij}\sim p^\star$, we have
\begin{align*}
    \inf_{P\in\mathcal{P}}\sup_{\substack{(f, p^\star)\in\mathcal{F}_{\sigma^2},\\ \hat p\in\mathcal{D}_{m, n, T}({P})}} \|\vecw^\star(f, p^\star) -\hat \vecw^T(P, \hat p)\|_2 \geq \inf_{P\in\mathcal{P}}\sup_{(p^\star, \hat p): \|\Cov_{p^\star}(\vecz)\|\leq \sigma^2, \TV(p^\star_n, \hat p)\leq \epsilon} \|\mathbb{E}_{p^\star}[\vecx]-\hat \vecw^T(P, \hat p)\|_2.
\end{align*}
From the argument of modulus of continuity (see e.g. Theorem 3.2 in~\cite{chen2016general}, Lemma D.4 in \cite{zhu2019generalized}), we know that with constant probability, 
\begin{align*}
  \inf_{P\in\mathcal{P}}\sup_{\substack{(p^\star, \hat p): \|\Cov_{p^\star}(\vecz)\|\leq \sigma^2,\\ \TV(p^\star_n, \hat p)\leq \epsilon}} \|\mathbb{E}_{p^\star}[\vecx] -\hat \vecw^T(P, \hat p)\|_2 \geq \sup_{\substack{(p^\star, {p^\star}'): \|\Cov_{p^\star}(\vecz)\|\leq \sigma^2,\\ \|\Cov_{{p^\star}'}(\vecz)\|\leq \sigma^2, \TV(p^\star_n, {p^\star_n}')\leq \epsilon}} \|\mathbb{E}_{p^\star}[\vecx]-\mathbb{E}_{{p^\star}'}[\vecx]\|_2  + \sigma\sqrt{\frac{d}{mn}}.
\end{align*}
We further lower bound the first term. Let $\epsilon' = 1-(1-\epsilon)^{1/n}$. We construct two one-dimensional distributions $p^\star$, ${p^\star}'$ as follows:
\begin{align*}
    \mathbb{P}_{p^\star}(x) &= \begin{cases} \epsilon', & x=\frac{\sigma}{3} \sqrt{\frac{1-\epsilon'}{\epsilon'}},
    \\
    1-\epsilon', & x = 0,
    \end{cases} \\
    \mathbb{P}_{{p^\star}'}(0) &=
    1.
\end{align*}
One can verify that both $p^\star$ and ${p^\star}'$ has spectral norm of its covariance bounded by $\sigma^2$, and furthermore, $\TV(p^\star_n, {p^\star_n}') = (1-\epsilon')^n=\epsilon$. In the meantime, one has
\begin{align*}
   \|\mathbb{E}_{p^\star}[\vecx]-\mathbb{E}_{{p^\star}'}[\vecx]\|_2 \gtrsim \sigma\sqrt{\epsilon'}\gtrsim \sigma\sqrt{\frac{\epsilon}{n}},
\end{align*}
where the last inequality uses the fact that $1-(1-\epsilon)^{1/n}\gtrsim \epsilon/n$.
\end{proof}

\section{Proof of Theorem~\ref{thm:main-gd-sc}}\label{prf:main-gd-sc}
Since $H=1$, we have $\vecw^t = \vecw_i^t$ for all $i\in[m]$. 
Given that each machine has $n$ samples, the variance of the returned local update $\vecg^i(\vecw^t) = -\eta_t\gradF_i(\vecw^t)$ satisfies
\begin{align*}
    \|\Cov_{p^\star}(\vecg^i(\vecw^t)) \| \leq \frac{\eta_t^2\sigma^2}{n}.
\end{align*}
Let $\mu_{p^\star}(\vecw), \mu_{i}(\vecw)$ denote the true mean of local updates under parameter $\vecw$ and the output of the $i$-th algorithm in~\eqref{eq:robust_subroutine1}, ~(\ref{eq:robust_subroutine2}),~(\ref{eq:robust_subroutine3}),~(\ref{eq:robust_subroutine4}),~(\ref{eq:robust_subroutine5}), respectively. 
We first show that the three robust estimators are Lipschitz with respect to $\vecw$ in the following lemma.
\begin{lemma}\label{lem:lipschitz}
For any $\vecw_1, \vecw_2\in\cW$ with $\|\vecw_1-\vecw_2\|_2\leq \gamma$,  one has for any $i\in\{1, 2,3, 4,5\}$, there exists some universal constants $C_j$ such that 
\begin{align*}
  \|\mu_{1}(\vecw_1)-\mu_{1}(\vecw_2)\|_2&\leq \frac{C_1\eta_t\sqrt{\epsilon}}{1-3\epsilon}\left({\frac{\sigma}{\sqrt{n}}}+L\gamma\right) + \frac{C_1\eta_t\sigma}{1-3\epsilon}\cdot \sqrt{\frac{d\log(d/\delta)}{mn}}  + \eta_tL \gamma,
\\ 
  \|\mu_{2}(\vecw_1)-\mu_{2}(\vecw_2)\|_2&\leq \frac{C_2\eta_t\sqrt{\epsilon}}{1-2\epsilon}\left({\frac{\sigma}{\sqrt{n}}}+L\gamma\right) + \frac{C_2\eta_t\sigma}{1-2\epsilon}\cdot \sqrt{\frac{d\log(d/\delta)}{mn}}  + \eta_tL \gamma,
  \\
  \|\mu_{3}(\vecw_1)-\mu_{3}(\vecw_2)\|_2&\leq \frac{C_3\eta_t\sqrt{\epsilon+\sqrt{d+\log(1/\delta)/n}}}{1-C_4(\epsilon+\sqrt{d+\log(1/\delta)/n})}\left({\frac{\sigma}{\sqrt{n}}}+L\gamma\right) + \eta_tL \gamma.\\ 
  \|\mu_{4}(\vecw_1)-\mu_{4}(\vecw_2)\|_2&\leq \frac{C_5\eta_t\sqrt{\epsilon}}{1-C_6\epsilon}\left({\frac{\sigma}{\sqrt{n}}}+L\gamma\right) + \frac{C_5\eta_t\sigma}{1-C_6\epsilon}\cdot \sqrt{\frac{d\log(d)+\log(1/\delta)}{mn}}+ \eta_tL \gamma,\\ 
  \|\mu_{5}(\vecw_1)-\mu_{5}(\vecw_2)\|_2&\leq \frac{C_7\eta_t\sqrt{\epsilon}}{1-C_8\epsilon}\left({\frac{\sigma}{\sqrt{n}}}+L\gamma\right) + \frac{C_7\eta_t\sigma}{1-C_8\epsilon}\cdot \sqrt{\frac{d\log(d)+\log(1/\delta)}{mn}}+ \eta_tL \gamma.
\end{align*}
\end{lemma}
\begin{proof}
We first prove the result for Algorithm~\ref{algo:explicit_lowregret} in Equation~(\ref{eq:robust_subroutine1}). 
From the smoothness assumption in Assumption~\ref{asm:smoothness}, we know that $\|\vecw_1-\vecw_2\|_2\leq \gamma$ implies that for any $\vecz$,
\begin{align}
    \|\nabla f(\vecw_1; \vecz)-\nabla f(\vecw_2; \vecz)\|_2\leq L\gamma.\label{} 
\end{align}
Assume without loss of generality that the first $(1-\epsilon)m$ clients are honest while the rest $\epsilon m$ clients are adversarial. Denote the local gradient $-\eta_t\nabla f(\vecw_1; \vecz)$ for $i$-th client as $\vecg_i$, and $-\eta_t\nabla f(\vecw_2; \vecz)$ for $i$-th client as $\vecg_i'$. 
Since $\mu_{1}(\vecw_1),\mu_{1}(\vecw_2)$ are the output of the algorithm in Equation~(\ref{eq:robust_subroutine1}), we know that for any $\|\vecg_1-\vecg_1'\|\leq \eta_tL\gamma, \|\vecg_2-\vecg_2'\|\leq \eta_tL\gamma, \cdots, \|\vecg_{(1-\epsilon)m}-\vecg_{(1-\epsilon)m}'\|\leq \eta_tL\gamma, \vecg_{(1-\epsilon)m+1}=\vecg_{(1-\epsilon)m+1}', \cdots, \vecg_m=\vecg_m'$, one can find some distribution $q_1, q_1'\in \Delta_{m, {2\epsilon}}$ such that $\mu_{1}(\vecw_1) = \mathbb{E}_{q_1}[\vecg]$, $\mu_{1}(\vecw_2) = \mathbb{E}_{q_1'}[\vecg']$, and for some universal constant $C$,
\begin{align*}
   \|\Cov_{q_1}(\vecg)\|_2 =  \sup_{\|v\|_2\leq 1} \mathbb{E}_{q_1}[(v^\top (\vecg-\mu_{1}(\vecw_1)))^2] \leq  \Big(\frac{C}{1 - 3\epsilon}\Big)^2\cdot \left(1+\frac{d\log(d/\delta)}{m\epsilon}\right)\cdot \frac{\eta_t^2\sigma^2}{n},\\
      \|\Cov_{q_1'}(\vecg')\|_2 =  \sup_{\|v\|_2\leq 1} \mathbb{E}_{q_1'}[(v^\top (\vecg'-\mu_{2}(\vecw_2)))^2] \leq  \Big(\frac{C}{1 - 3\epsilon}\Big)^2\cdot \left(1+\frac{d\log(d/\delta)}{m\epsilon}\right)\cdot \frac{\eta_t^2\sigma^2}{n}.
\end{align*}
This implies that 
\begin{align*}
      \|\Cov_{q_1'}(\vecg)\|_2 & = \sup_{\|v\|_2\leq 1} \mathbb{E}_{q_1'}[(v^\top (\vecg-\mathbb{E}_{q_1'}[\vecg]))^2]  
      \\
    & \leq  \sup_{\|v\|_2\leq 1} \mathbb{E}_{q_1'}[(v^\top (\vecg'-\mathbb{E}_{q_1'}[\vecg']))^2]   + 2\eta_t^2L^2\gamma^2 \\
    & \leq  \Big(\frac{C}{1 - 3\epsilon}\Big)^2\cdot \left(1+\frac{d\log(d/\delta)}{m\epsilon}\right)\cdot \frac{\eta_t^2\sigma^2}{n} + 2\eta_t^2L^2\gamma^2.
\end{align*}
Since $\TV(q_1, q_1')\leq \frac{2\epsilon}{1-3\epsilon}$ and $\epsilon<1/3$, from e.g.~\citet[Proposition 2.3]{steinhardt2018robust} and~\citet[Lemma E.2]{zhu2019generalized} we know that for some different universal constant $C$,
\begin{align*}
   \| \mathbb{E}_{q_1}[\vecg]-\mathbb{E}_{q_1'}[\vecg]\|_2& \leq \frac{C\eta_t\sqrt{\epsilon}}{1-3\epsilon}\left(\sqrt{\frac{1}{n}+\frac{d\log(d/\delta)}{mn\epsilon}}\sigma+L\gamma\right)
   \\
  & = \frac{C\eta_t\sqrt{\epsilon}}{1-3\epsilon}\left({\frac{\sigma}{\sqrt{n}}}+L\gamma\right) + \frac{C\eta_t\sigma}{1-3\epsilon}\cdot \sqrt{\frac{d\log(d/\delta)}{mn}}.
\end{align*}
And thus 
\begin{align*}
   \| \mu_{1}(\vecw_1)-\mu_{1}(\vecw_2)\|_2& = 
     \| \mathbb{E}_{q_1}[\vecg]-\mathbb{E}_{q_1'}[\vecg']\|_2
   \\&\leq  \|\mathbb{E}_{q_1}[\vecg]-\mathbb{E}_{q_1'}[\vecg]\|_2 + \|\mathbb{E}_{q_1'}[\vecg]-\mathbb{E}_{q_1'}[\vecg']\|_2
   \\
  & = \frac{C\eta_t\sqrt{\epsilon}}{1-3\epsilon}\left({\frac{\sigma}{\sqrt{n}}}+L\gamma\right) + \frac{C\eta_t\sigma}{1-3\epsilon}\cdot \sqrt{\frac{d\log(d/\delta)}{mn}} + \eta_tL \gamma.
\end{align*}
The results for other algorithms follows the same line of argument.
\end{proof}
Consider a $\gamma$-covering in $\ell_2$ distance for the set $\mathcal{W}$, denoted as $\mathcal{W}_\gamma$. From~\cite{vershynin2010introduction}, we know that $|\mathcal{W}_\gamma|\leq (1+\frac{D}{\gamma})^d$. 
By union bound and Theorem~\ref{thm:robust_guarantee}, we know that each of the following holds with probability at least $1-(1+D/\gamma)^d\delta$, 
\begin{align*}
    \sup_{\vecw_\gamma\in\mathcal{W}_\gamma} \|-\mu_{1}(\vecw_\gamma) - \eta_t\nabla F(\vecw_\gamma)\|_2&\leq \bigo\left(\frac{\eta_t\sigma}{\sqrt{n}}\cdot \left(\frac{\sqrt    \epsilon}{1-3\epsilon} + \sqrt{\frac{d\log(d/\delta)}{m}}\right)\right), \\
       \sup_{\vecw_\gamma\in\mathcal{W}_\gamma} \|-\mu_{2}(\vecw_\gamma) - \eta_t\nabla F(\vecw_\gamma)\|_2&\leq \bigo\left(\frac{\eta_t\sigma}{\sqrt{n}}\cdot \left(\frac{\sqrt{\epsilon}}{1-2\epsilon} + \sqrt{\frac{d\log(d/\delta)}{m}}\right)\right), \\
      \sup_{\vecw_\gamma\in\mathcal{W}_\gamma} \|-\mu_{3}(\vecw_\gamma) - \eta_t\nabla F(\vecw_\gamma)\|_2&\leq \bigo\left(\frac{\eta_t\sigma}{\sqrt{n}}\cdot \left(\frac{\sqrt{\epsilon+\sqrt{(d+\log(1/\delta))/m}}}{1-C\epsilon}  \right)\right), \\
      \sup_{\vecw_\gamma\in\mathcal{W}_\gamma} \|-\mu_{4}(\vecw_\gamma) - \eta_t\nabla F(\vecw_\gamma)\|_2&\leq \bigo\left(\frac{\eta_t\sigma}{1-C(\epsilon+\log(1/\delta)/m)}\cdot \left({\sqrt{\epsilon}}+ \sqrt{\frac{d\log(d)+\log(1/\delta)}{m}}\right)\right) \\
      \sup_{\vecw_\gamma\in\mathcal{W}_\gamma} \|-\mu_{5}(\vecw_\gamma) - \eta_t\nabla F(\vecw_\gamma)\|_2&\leq \bigo\left(\frac{\eta_t\sigma}{1-C(\epsilon+\log(1/\delta)/m)}\cdot \left({\sqrt{\epsilon}}+ \sqrt{\frac{d\log(d)+\log(1/\delta)}{m}}\right)\right)
\end{align*}
Since both $\mu_{1}(\vecw_\gamma) $ and $\nabla F(\vecw_\gamma)$ are Lipschitz (from Lemma~\ref{lem:lipschitz} and Assumption~\ref{asm:smoothness}), we know that with probability at least $1-\delta$,
\begin{align*}
    \sup_{\vecw\in\mathcal{W}} \|-\mu_{1}(\vecw) - \eta_t\nabla F(\vecw)\|_2\leq  & \sup_{\vecw_\gamma\in\mathcal{W}_\gamma} \|-\mu_{1}(\vecw_\gamma) - \eta_t\nabla F(\vecw_\gamma)\|_2 + \eta_tL\gamma + \frac{C\eta_t\sqrt{\epsilon}}{1-3\epsilon}\left({\frac{\sigma}{\sqrt{n}}}+L\gamma\right)  \\ 
    & \quad + \frac{C\eta_t\sigma}{1-3\epsilon}\cdot \sqrt{\frac{d\log(d/\delta)}{mn}}\\ 
    \leq &\bigo\left(\frac{\eta_t\sigma}{(1-3\epsilon)\sqrt{n}}\cdot \left({\sqrt\epsilon} + \sqrt{\frac{d^2\log(1+D/\gamma)+d\log(d/\delta)}{m}}\right)+\frac{\eta_tL\gamma}{{1-3\epsilon}}\right),
\end{align*}
By taking $\gamma = 1/nmL$, we know that with probability at least $1-\delta$, 
\begin{align}\label{eq:proof_delta1}
    \sup_{\vecw\in\mathcal{W}} \|-\mu_{1}(\vecw) - \eta_t\nabla F(\vecw)\|_2\leq \bigo\left(\frac{\eta_t\sigma}{(1-3\epsilon)\sqrt{n}}\cdot \left({\sqrt    \epsilon} + \sqrt{\frac{d^2\log(1+nmDL)+d\log(d/\delta)}{m}}\right)\right).
\end{align}
Similarly, we can get for Algorithm~\ref{algo:filtering},~\ref{algo:GAN} that with probability at least $1-\delta$,
\begin{align}\label{eq:proof_delta2}
        \sup_{\vecw\in\mathcal{W}} \|-\mu_{2}(\vecw) - \eta_t\nabla F(\vecw)\|_2&\leq \bigo\left(\frac{\eta_t\sigma}{(1-2\epsilon)\sqrt{n}}\cdot \left({\sqrt    \epsilon} + \sqrt{\frac{d^2\log(1+nmDL)+d\log(d/\delta)}{m}}\right)\right),
        \\
    \label{eq:proof_delta3}   \sup_{\vecw\in\mathcal{W}} \|-\mu_{3}(\vecw) -\eta_t \nabla F(\vecw)\|_2&\leq \bigo\left(\frac{\eta_t\sigma}{(1-C\epsilon)\sqrt{n}}\cdot \left({\sqrt{\epsilon+\sqrt{\frac{d\log(1+nmDL)+\log(1/\delta)}{m}}}}  \right)\right) \\
    \label{eq:proof_delta4} \sup_{\vecw\in\mathcal{W}} \|-\mu_{4}(\vecw) -\eta_t \nabla F(\vecw)\|_2& \leq \bigo\left(\frac{\eta_t\sigma}{1-C(\epsilon+\frac{\log(1/\delta)}{m})} \cdot \left({\sqrt{\epsilon}}+ \sqrt{\frac{d(\log(d)+\log(1+nmDL))+\log(1/\delta)}{m}}\right)\right)\\
        \label{eq:proof_delta5} \sup_{\vecw\in\mathcal{W}} \|-\mu_{5}(\vecw) - \eta_t\nabla F(\vecw)\|_2& \leq  \bigo\left(\frac{\eta_t\sigma}{1-C(\epsilon+\frac{\log(1/\delta)}{m})}\cdot \left({\sqrt{\epsilon}}+ \sqrt{\frac{d(\log(d)+\log(1+nmDL))+\log(1/\delta)}{m}}\right)\right)
\end{align}
Then, we proceed to analyze the convergence of the robust distributed gradient descent algorithm. This part follows directly with the analysis in~\cite{yin2018byzantine}. For completeness we include the proof here.

\paragraph*{Convergence analysis for strongly convex case.} 
We condition on the event that the bound above is satisfied for all $\vecw\in\W$. Then, in the $t$-th iteration, let the updated parameter before projection step be $
\widehat{\vecw}^{t+1} = \vecw^t + \vecg(\vecw^t).$
Thus, we have $\vecw^{t+1} = \Pi_\W(\widehat{\vecw}^{t+1})$. From the property of Euclidean projection, we know that 
\begin{align}
   \twonms{\vecw^{t+1} - \vecw^*} & \le \twonms{\widehat{\vecw}^{t+1} - \vecw^*} \nonumber\\ 
   & = \twonms{\vecw^t + \vecg(\vecw^t) - \vecw^*} \nonumber \\
   & \leq \twonms{\vecw^t - \eta_t\gradF(\vecw^t) - \vecw^*} + \twonms{-\vecg(\vecw^t) -\eta_t\gradF(\vecw^t)}. \label{eq:converge-1}
\end{align}
For the first term, we have
\begin{equation}\label{eq:iter-cvx-1}
\twonms{\vecw^t - \eta_t\gradF(\vecw^t) - \vecw^*}^2 = \twonms{\vecw^t - \vecw^*}^2 - 2\eta_t\innerps{\vecw^t - \vecw^*}{\gradF(\vecw^t)} + \eta_t^2\twonms{\gradF(\vecw^t)}^2.
\end{equation}
Since $F(\vecw)$ is $\lambda$-strongly convex, by the co-coercivity of strongly convex functions (see e.g. Lemma 3.11 in~\cite{bubeck2015convex}), we obtain
$$
\innerps{\vecw^t - \vecw^*}{\gradF(\vecw^t)} \ge \frac{L\lambda}{L + \lambda}\twonms{\vecw^t - \vecw^*}^2 + \frac{1}{L + \lambda}\twonms{\gradF(\vecw^t)}^2.
$$
Let $\eta_t = \frac{1}{L}$. Then we get
\begin{align*}
\twonms{\vecw^t - \eta_t\gradF(\vecw^t) - \vecw^*}^2 & \le (1- \frac{2\lambda}{L + \lambda}) \twonms{\vecw^t - \vecw^*}^2 - \frac{2}{L(L+\lambda)}\twonms{\gradF(\vecw^t)}^2 + \frac{1}{L^2}\twonms{\gradF(\vecw^t)}^2 \\
& \le (1- \frac{2\lambda}{L + \lambda}) \twonms{\vecw^t - \vecw^*}^2,
\end{align*}
where in the second inequality we use the fact that $\lambda \le L$. Using the fact $\sqrt{1-x} \le 1-\frac{x}{2}$, we get
\begin{equation}\label{eq:iter-linear-decay}
\twonms{\vecw^t - \eta_t\gradF(\vecw^t) - \vecw^*} \le (1- \frac{\lambda}{L + \lambda})\twonms{\vecw^t - \vecw^*}.
\end{equation}
Combining~\eqref{eq:converge-1} and~\eqref{eq:iter-linear-decay}, we get
\begin{equation}\label{eq:each-iter-sc}
\twonms{\vecw^{t+1} - \vecw^*} \le (1-\frac{\lambda}{L + \lambda})\twonms{\vecw^t - \vecw^*} + \frac{1}{L}\Delta_i,
\end{equation}
where $\Delta_i$ is defined as in~\eqref{eq:def-error-delta1}, (\ref{eq:def-error-delta2}),  (\ref{eq:def-error-delta3}),  (\ref{eq:def-error-delta4}), or (\ref{eq:def-error-delta5}), respectively.
Then we can complete the proof by iterating~\eqref{eq:each-iter-sc}.

\paragraph*{Convergence analysis for non-strongly Convex  Losses:}
Let $\tilde \vecg(\vecw^t) = -\vecg(\vecw^t)/\eta_t$ be the virtual gradient  at round $t$. 
We first show that when Assumption~\ref{asm:convexity}(b) is satisfied and we choose $\eta_t = \frac{1}{L}$, the iterates $\vecw^t$ stays in $\W$ without using projection. Namely, let 
$
\vecw^{t+1} = \vecw^t - \eta_t \tilde \vecg(\vecw^t),
$for $T=0,1,\ldots, T-1$, then $\vecw^{t}\in\W$ for all $t=0,1,\ldots, T$.
To see this, we have
$$
\twonms{\vecw^{t+1} - \vecw^*} \le \twonms{\vecw^t - \eta_t \gradF(\vecw^t) -\vecw^*} + \eta_t \twonms{\tilde \vecg(\vecw^t) - \gradF(\vecw^t)},
$$
and
\begin{align*}
\twonms{\vecw^t - \eta_t \gradF(\vecw^t) -\vecw^*}^2 &= \twonms{\vecw^t - \vecw^*}^2 - 2\eta_t\innerps{\gradF(\vecw^t)}{\vecw^t - \vecw^*} + \eta_t^2\twonms{\gradF(\vecw^t)}^2 \\
&\le \twonms{\vecw^t - \vecw^*}^2 -2\eta_t \frac{1}{L}\twonms{\gradF(\vecw^t)}^2 + \eta_t^2\twonms{\gradF(\vecw^t)}^2 \\
&= \twonms{\vecw^t - \vecw^*}^2 - \frac{1}{L^2}\twonms{\gradF(\vecw^t)}^2 \\
&\le \twonms{\vecw^t - \vecw^*}^2
\end{align*}
where the inequality is due to the co-coercivity of convex functions. Thus, we get
$$
\twonms{\vecw^{t+1} - \vecw^*} \le \twonms{\vecw^{t} - \vecw^*}  + \frac{\Delta_i}{L}.
$$
Let $D_t:=\twonms{\vecw^0 - \vecw^*} + \frac{t \gamma}{L}$ for $t=0,1,\ldots, T$. Since $T = \frac{LD_0}{\gamma}$, according to Assumption~\ref{asm:convexity}(b) we know that $\vecw^t\in\W$ for all $t=0,1,\ldots, T$. Then, we proceed to study the algorithm without projection. 

Using the smoothness of $F(\vecw)$, we have
\begin{align*}
F(\vecw^{t+1}) & \le F(\vecw^t) + \innerps{\gradF(\vecw^t)}{\vecw^{t+1} - \vecw^t} + \frac{L}{2}\twonms{\vecw^{t+1} - \vecw^t}^2 \\
&= F(\vecw^t) + \eta_t \innerps{\gradF(\vecw^t)}{ - \tilde \vecg(\vecw^t) + \gradF(\vecw^t) - \gradF(\vecw^t)} + \eta_t^2 \frac{L}{2} \twonms{ \tilde \vecg(\vecw^t) - \gradF(\vecw^t) + \gradF(\vecw^t)}^2.
\end{align*}
Since $\eta_t = \frac{1}{L}$ and $\twonms{ \tilde \vecg(\vecw^t) - \gradF(\vecw^t)} \le \gamma$, by simple algebra, we obtain
\begin{equation}\label{eq:smoothness-noise}
F(\vecw^{t+1}) \le F(\vecw^t) - \frac{1}{2L}\twonms{\gradF(\vecw^t)}^2 + \frac{1}{2L}\gamma^2.
\end{equation}
We state the following lemma from~\cite{yin2018byzantine} without proof.
\begin{lemma}\label{lem:cvx-first-part}
Condition on the event that~\eqref{eq:proof_delta1}, (\ref{eq:proof_delta2}), (\ref{eq:proof_delta3}), (\ref{eq:proof_delta4}) or (\ref{eq:proof_delta5}) holds for all $\vecw\in\W$. When $F(\vecw)$ is convex, by running $T =  \frac{LD_0}{\Delta}$ parallel iterations, there exists $t\in\{0,1,2,\ldots, T\}$ such that
$$
F(\vecw^t) - F(\vecw^*) \le 16D_0\Delta_i.
$$
for some $i\in\{1, 2, 3, 4, 5\}$.
\end{lemma}
Next, we show that $F(\vecw^T) - F(\vecw^*) \le  16 D_0\Delta_i + \frac{1}{2L}\Delta_i^2$. More specifically, let $t=t_0$ be the first time that $F(\vecw^{t}) - F(\vecw^*) \le 16 D_0\Delta_i$, and we show that for any $t > t_0$, $F(\vecw^{t}) - F(\vecw^*) \le 16 D_0\Delta_i + \frac{1}{2L}\Delta_i^2$. If this statement is not true, then we let $t_1 > t_0$ be the first time that $F(\vecw^{t}) - F(\vecw^*) > 16 D_0\Delta_i + \frac{1}{2L}\Delta_i^2$. Then there must be $F(\vecw^{t_1-1}) < F(\vecw^{t_1})$. According to~\eqref{eq:smoothness-noise}, there should also be 
$$
F(\vecw^{t_1-1}) - F(\vecw^*) \ge F(\vecw^{t_1}) - F(\vecw^*) - \frac{1}{2L}\Delta_i^2 > 16D_0\Delta_i.
$$
Then, according to the first order optimality of convex functions, for any $\vecw$,
$$
F(\vecw) - F(\vecw^*) \le \innerps{\gradF(\vecw)}{\vecw - \vecw^*} \le \twonms{\gradF(\vecw)}\twonms{\vecw - \vecw^*},
$$
and thus
$$
\twonms{\gradF(\vecw)} \ge \frac{F(\vecw) - F(\vecw^*)}{\twonms{\vecw - \vecw^*}}.
$$
This gives that 
$$
\twonms{\gradF(\vecw^{t_1-1})} \ge \frac{F(\vecw^{t_1-1}) - F(\vecw^*)}{\twonms{\vecw^{t_1-1} -\vecw^* }} > 8\Delta_i.
$$
Then according to~\eqref{eq:smoothness-noise}, this implies $ F(\vecw^{t_1}) \le F(\vecw^{t_1-1}) $,
which contradicts with the fact that $F(\vecw^{t_1-1}) < F(\vecw^{t_1})$.

\paragraph*{Convergence analysis for non-convex  Losses:}
 We condition on the event that~\eqref{eq:proof_delta1}, (\ref{eq:proof_delta2}), (\ref{eq:proof_delta3}), (\ref{eq:proof_delta4}) or (\ref{eq:proof_delta5})   holds for all $\vecw\in\W$ (for Algorithm~\ref{algo:explicit_lowregret},~\ref{algo:filtering}, \ref{algo:GAN} respectively). We first show that when Assumption~\ref{asm:convexity}(c) is satisfied and we choose $\eta_t = \frac{1}{L}$, the iterates $\vecw^t$ stays in $\W$ without using projection. Since we have
$$
\twonms{\vecw^{t+1} - \vecw^*} \le \twonms{\vecw^t - \vecw^*} + \eta_t(\twonms{\gradF(\vecw^t)} + \twonms{\tilde \vecg(\vecw^t) - \gradF(\vecw^t)}) \le \twonms{\vecw^t - \vecw^*}  + \frac{1}{L}(M+\Delta_i).
$$
Then, we know that by running $T = \frac{2L}{\Delta_i^2}(F(\vecw^0) - F(\vecw^*))$ parallel iterations, using Assumption~\ref{asm:convexity}(c), we know that $\vecw^t \in \W$ for $t=0,1,\ldots, T$ without projection.

We proceed to study the convergence rate of the algorithm. By the smoothness of $F(\vecw)$, we know that when choosing $\eta_t = \frac{1}{L}$, the inequality~\eqref{eq:smoothness-noise} still holds. More specifically, for all $t=0,1,\ldots, T-1$,
\begin{equation}\label{eq:smoothness-noise-2}
F(\vecw^{t+1}) - F(\vecw^*) \le F(\vecw^t) - F(\vecw^*) - \frac{1}{2L}\twonms{\gradF(\vecw^t)}^2 + \frac{1}{2L}\Delta_i^2.
\end{equation}
Sum up~\eqref{eq:smoothness-noise-2} for $t=0,1,\ldots, T-1$. Then, we get
$$
0 \le F(\vecw^T) - F(\vecw^*) \le F(\vecw^0) - F(\vecw^*) - \frac{1}{2L}\sum_{t=0}^{T-1}\twonms{\gradF(\vecw^t)}^2 + \frac{T}{2L}\Delta_i^2.
$$
This implies that
$$
\min_{t=0,1,\ldots, T} \twonms{\gradF(\vecw^t)}^2  \le \frac{2L}{T} (F(\vecw^0) - F(\vecw^*)) + \Delta_i^2,
$$
which completes the proof.

\section{Proof of Theorem~\ref{thm:main-gd-improved}}\label{app:proof_improved}

The proof follows the same route as of Theorem~\ref{thm:main-gd-sc}, except that we utilize the result of sub-Gaussian rate for the new procedure. In particular, we rely on the following lemma that appears in~\cite{diakonikolas2020outlier}.
\begin{lemma}[{Proposition 1.6, \cite{diakonikolas2020outlier}}]\label{lem:robust_guarantee}
Assume that $\|\Cov_{p^\star}(\vecx)\|\leq \sigma^2$. Consider the procedure which first randomly bucket $m$ samples from $p^\star$ into $k=\floor{\epsilon m+\log(1/\delta)}$ disjoint buckets of equal size, compute their empirical means $\vecz^1, \cdots, \vecz^k$, and apply~\eqref{eq:robust_subroutine6} or (\ref{eq:robust_subroutine7}) onto the empirical means. Let the output be $\mu$. Then  with probability at least $1-\delta$:
\begin{align*}
    \|\mu-\mu_{p^\star}\|&\leq\bigo\left(\frac{\sigma}{(1-C\epsilon)}\cdot \left({\sqrt    \epsilon} + \sqrt{\frac{d+\log(1/\delta)}{m}}\right)\right).
\end{align*}
\end{lemma}
The rest of the proof follows the same as that in Appendix~\ref{prf:main-gd-sc}. 

\section{Analysis for multi-round local model update}\label{app:multi-round}

In this section, we analyze the FL protocol in Algorithm~\ref{alg:robust-gd} with $k=m$ when $H\geq 2$ total rounds are updated locally before aggregation in central server. Similar analysis can also be applied to convex and non-convex cases, along with the case of $k<m$. 

In this case, we have the local update for the $i$-th honest machine be
\begin{align*}
\vecg_i(\vecw^{Ht}) &= -\sum_{s=Ht}^{H(t+1)-1}\eta_t \gradF_i(\vecw_i^s) 
\end{align*}
In this case, we slightly modify our Assumption~\ref{asm:bounded-variance} to consider the $H$-round local model update:
\begin{asm}\label{asm:bounded-variance-FL}
Assume that for any  fixed $\vecw^{Ht}\in\mathcal{W}$ and any $\eta_t<1/(2H), i\in[m]$, $\|\Cov(\vecg_i(\vecw^{Ht}))\|_2 \leq \bigo(\frac{\eta_t^2\sigma_H^2}{n}).$
\end{asm}
Note that the  distribution of $\vecg_j(\vecw^{Ht})$ has complex dependence: for any $l\in[H]$, the randomness of $\vecw^{Ht+l}_i$ depends on both the randomness of noise $\vecz^{i,j}_{Ht+l}$ and the parameter in the previous step $\vecw^{Ht+l-1}_i$. However, one can still verify that in the  simple case of  mean estimation with $f(\vecw; \vecz) = \|\vecw-\vecz\|_2^2$. Assumption~\ref{asm:bounded-variance} implies Assumption~\ref{asm:bounded-variance-FL} with $\sigma_H=H\sigma$. This is from that
\begin{align*}
    \|\Cov(\vecg_i(\vecw^{Ht}))\|_2  & = \eta_t^2\left\|\Cov\left(\frac{1}{n}\sum_{j=1}^{n}\sum_{s=Ht}^{H(t+1)-1}\gradf(\vecw_i^s; \vecz_s^{i, j})\right)\right\|_2 \\ 
    & = 4\eta_t^2\left\|\Cov\left(\frac{1}{n}\sum_{j=1}^{n}\sum_{s=Ht}^{H(t+1)-1}(\vecw_i^s- \vecz_s^{i, j})\right)\right\|_2  \\
    & = 4\eta_t^2\left\|\Cov\left(\frac{1}{n}\sum_{j=1}^{n}\sum_{s=Ht}^{H(t+1)-1}(\vecw_i^{Ht}-\vecz_s^{i, j}-\eta_t \sum_{l=Ht}^s (\vecw^l_i - \vecz_{l}^{i, j})) \right)\right\|_2 \\
    & = 4\eta_t^2\left\|\Cov\left(\frac{1}{n}\sum_{j=1}^{n}\sum_{s=Ht}^{H(t+1)-1}-\vecz_s^{i, j}-\eta_t \sum_{l=Ht}^{s-1} (\vecw^l_i - \vecz_{l}^{i, j}) \right)\right\|_2  \\ 
    & \leq  4\eta_t^2\left\|\Cov\left(\frac{1}{n}\sum_{j=1}^{n}\sum_{s=Ht}^{H(t+1)-1} (\sum_{k=s-Ht}^{H}(\eta_t H)^k)\vecz_s^{i, j} \right)\right\|_2 \\
    & \leq \bigo\left(\frac{\eta_t^2H^2\sigma^2}{n}\right).
\end{align*}
The last inequality is from the independence between $\vecz_s^{i, j}$. Similarly, one can also show that under mild regularization conditions on the parameter space, Assumption~\ref{asm:bounded-variance-FL} is satisfied in linear regression.

Now we are ready to prove the main theorem as below.

\begin{theorem}\label{thm:main-gd-multi}
Let $\vecw_i^t, i\in\{1, 2, 3,4, 5\}$  denote the output of Algorithm~\ref{alg:robust-gd} with $H\geq 2$ and step-size $\eta_t = a/L(t+a)$ at round $t$ when the $\mathsf{RobustEstimation subroutine}$ takes the algorithm in~\eqref{eq:robust_subroutine1},~(\ref{eq:robust_subroutine2}),~(\ref{eq:robust_subroutine3}),~(\ref{eq:robust_subroutine4}),~(\ref{eq:robust_subroutine5}), respectively, where $\sigma^2$ is replaced with $\sigma_H^2$. Here $a=(L+\lambda)/\lambda$.    Under Assumption~\ref{asm:parameter},~\ref{asm:smoothness},~\ref{asm:bounded-variance-FL} and~\ref{asm:convexity}(a), with probability at least $1-\delta$, we have 
\begin{align*}
     \twonms{{\vecw^{HT}} - \vecw^*}\leq \left(\frac{a}{T+a-1}\right)^H \|\vecw^0-\vecw^\star\|_2 +\frac{C_1a^2H^2D}{(T+a-1)^2} +  C_2 \Delta_i,\end{align*}
where
\begin{align}\label{eq:def-error-delta1-multi}
\Delta_1 &:= \bigo\left(\frac{\sigma_H}{(1-6\epsilon)\sqrt{n}}\cdot \left({\sqrt    \epsilon} + \sqrt{\frac{d^2\log(1+nmL)+d\log(d/\delta)}{m}}\right)\right), \\
\label{eq:def-error-delta2-multi}\Delta_2 &:= \bigo\left(\frac{\sigma_H}{(1-2\epsilon)\sqrt{n}}\cdot \left({\sqrt    \epsilon} + \sqrt{\frac{d^2\log(1+nmL)+d\log(d/\delta)}{m}}\right)\right), \\
\label{eq:def-error-delta3-multi}\Delta_3 &:= \bigo\left(\frac{\sigma_H}{({1-C_1\epsilon})\sqrt{n}}\cdot \left({\sqrt    {\epsilon+\sqrt{\frac{d\log(1+nmL)+\log(1/\delta)}{m}}}}\right)\right), \\
\label{eq:def-error-delta4-multi}\Delta_4 &:=\bigo\left(\frac{\sigma_H}{(1-C_2(\epsilon+\log(1/\delta)/m))\sqrt{n}}\cdot \left({\sqrt{\epsilon}}+ \sqrt{\frac{d\log(d)+\log(1+nmL)+\log(1/\delta)}{m}}\right)\right), \\
\label{eq:def-error-delta5-multi}\Delta_5 &:=\bigo\left(\frac{\sigma_H}{(1-C_3(\epsilon+\log(1/\delta)/m))\sqrt{n}}\cdot \left({\sqrt{\epsilon}}+ \sqrt{\frac{d\log(d)+\log(1+nmL)+\log(1/\delta)}{m}}\right)\right),
\end{align}
and $C_i$ are universal constants.
\end{theorem}
\begin{proof}
We condition on the event that the bound above is satisfied for all $\vecw\in\W$. Then, in the $t$-th iteration, let the updated parameter before projection step be $
\widehat{\vecw}^{H(t+1)} = \vecw^{Ht} +  \vecg(\vecw^{Ht})$ (note that $\vecg$ is the local model update which can be viewed as negative of gradient).
Thus, we have $\vecw^{H(t+1)} = \Pi_\W(\widehat{\vecw}^{H(t+1)})$.
For the sake of analysis, we create an artificial sequence of $\bar \vecw^{Ht+l}, l\in[H]$, such that $\bar \vecw^{Ht} = \vecw^{Ht}$, $\bar \vecw^{Ht+l+1} = \bar\vecw^{Ht+l} - \eta_t \gradF(\bar\vecw^{Ht+l})$, $\forall l\in[H]$.
From the property of Euclidean projection, we know that 
\begin{align}
   \twonms{\vecw^{H(t+1)} - \vecw^*} & \le \twonms{\widehat{\vecw}^{H(t+1)} - \vecw^*} \nonumber\\ 
   & = \twonms{\vecw^{Ht} +  \vecg(\vecw^{Ht}) - \vecw^*} \nonumber \\
   & \leq \twonms{\vecw^{Ht} - \eta_t\sum_{l=0}^{H-1}\gradF(\bar\vecw^{Ht+l}) - \vecw^*} + \twonms{-\vecg(\vecw^{Ht}) -\eta_t\sum_{l=0}^{H-1}\gradF(\bar\vecw^{Ht+l})} \nonumber \\
   & = \twonms{\bar \vecw^{Ht} - \eta_t\sum_{l=0}^{H-1}\gradF(\bar \vecw^{Ht+l}) - \vecw^*} + \twonms{-\vecg(\vecw^{Ht}) -\eta_t\sum_{l=0}^{H-1}\gradF(\bar\vecw^{Ht+l})} \nonumber \\ 
   & =\twonms{\bar \vecw^{H(t+1)} - \vecw^\star} + \twonms{-\vecg(\vecw^{Ht}) -\eta_t\sum_{l=0}^{H-1}\gradF(\bar\vecw^{Ht+l})}\label{eq:converge-1-multi}
\end{align}
Following the same analysis as in~\eqref{eq:iter-linear-decay}, we know that when $\eta_t\leq 1/L$, for any $l\in[H]$,
\begin{align*}
    \twonms{\bar \vecw^{Ht+l+1}  - \vecw^*} \leq (1-\frac{\eta_tL\lambda}{L+\lambda})  \twonms{\bar \vecw^{Ht+l}  - \vecw^*}
\end{align*}
Thus we know that
\begin{align}
      \twonms{\bar \vecw^{H(t+1)}  - \vecw^*} \leq (1-\frac{\eta_tL\lambda}{L+\lambda})^H  \twonms{\bar \vecw^{Ht}  - \vecw^*} \label{eq:first_term_multiple}
\end{align}
For the second term, we know that
\begin{align*}
\twonms{-\vecg(\vecw^{Ht}) -\eta_t\sum_{l=0}^{H-1}\gradF(\bar\vecw^{Ht+l})} 
& \leq \twonms{\vecg(\vecw^{Ht}) - \mathbb{E}[\vecg_i(\vecw^{Ht}) |\vecw^{Ht}]}
\\& \quad + \twonms{-\mathbb{E}[\vecg_i(\vecw^{Ht}) |\vecw^{Ht}] - \eta_t\sum_{l=0}^{H-1}\gradF(\bar\vecw^{Ht+l})}.
\end{align*}
From Assumption~\ref{asm:bounded-variance-FL} and Theorem~\ref{thm:robust_guarantee}, we know that the first term satisfies w.p. $1-\delta$
\begin{align}
     \twonms{\vecg(\vecw^{Ht}) - \mathbb{E}[\vecg_i(\vecw^{Ht}) |\vecw^{Ht}]} \leq \Delta_i.\label{eq:statistical_multiple}
\end{align}
For the second term, we have
\begin{align*}
    \twonms{\mathbb{E}[-\vecg_i(\vecw^{Ht}) |\vecw^{Ht}] - \eta_t\sum_{l=0}^{H-1}\gradF(\bar\vecw^{Ht+l})} = \twonms{\eta_t\sum_{l=0}^{H-1}(\mathbb{E}[ \gradF_i(\vecw_i^{Ht+l})|\vecw^{Ht}] - \gradF(\bar\vecw^{Ht+l}))} 
\end{align*}
Note that for any $l\in[H-1]$, we have
\begin{align}
    \twonms{\mathbb{E}[ \gradF_i(\vecw_i^{Ht+l+1})|\vecw^{Ht}] - \gradF(\bar\vecw^{Ht+l+1})} & = \twonms{\mathbb{E}[ \gradF_i(\vecw_i^{Ht+l}-\eta_t\gradF_i(\vecw_i^{Ht+l}))|\vecw^{Ht}] \nonumber \\ 
    & \quad - \gradF(\bar\vecw^{Ht+l}-\eta_t\gradF(\bar \vecw^{Ht+l}))} \nonumber\\ &\stackrel{(i)}{\leq} \twonms{\mathbb{E}[ \gradF_i(\vecw_i^{Ht+l})|\vecw^{Ht}] - \gradF(\bar\vecw^{Ht+l})} \nonumber \\ 
    & \quad  +\eta_tL(\twonms{\gradF(\bar \vecw^{Ht+l})}+\twonms{\mathbb{E}[\gradF_i(\vecw_i^{Ht+l})|\vecw_i^{Ht}]})\nonumber \\ 
    & \stackrel{(ii)}{\leq} \twonms{\mathbb{E}[ \gradF_i(\vecw_i^{Ht+l})|\vecw^{Ht}] - \gradF(\bar\vecw^{Ht+l})} +2\eta_tL^2D. \label{eq:iterative-multiple}
\end{align}
Here (i) is from the Lipschitzness of $\mathbb{E}[\gradF_i(\vecw)]$ and $\gradF(\vecw)$, and (ii) uses the fact that the parameter space is compact with diameter $D$, and the function is smoothness with minimizer inside $\mathcal{W}$, as in Assumption~\ref{asm:parameter},~\ref{asm:smoothness}, and~\ref{asm:convexity}(a).

By iteratively applying~\eqref{eq:iterative-multiple}, we know that
\begin{align*}
      \twonms{\mathbb{E}[ \gradF_i(\vecw_i^{Ht+l+1})|\vecw^{Ht}] - \gradF(\bar\vecw^{Ht+l+1})} &\leq \twonms{\mathbb{E}[ \gradF_i(\vecw_i^{Ht})|\vecw^{Ht}] - \gradF(\bar\vecw^{Ht})} +2l\eta_tL^2D \\ 
      & = 2l\eta_tL^2D.
\end{align*}
The last equality uses the fact that $\gradF_i(\vecw)$ is unbiased.
This leads to 
\begin{align}\label{eq:second_term_multiple}
    \twonms{\mathbb{E}[\vecg_i(\vecw^{Ht}) |\vecw^{Ht}] - \eta_t\sum_{l=0}^{H-1}\gradF(\bar\vecw^{Ht+l})} & = \twonms{\eta_t\sum_{l=0}^{H-1}(\mathbb{E}[ \gradF_i(\vecw_i^{Ht+l})|\vecw^{Ht}] - \gradF(\bar\vecw^{Ht+l}))} \nonumber \\ 
    & \leq H^2\eta_t^2L^2D.
\end{align}
Combining~\eqref{eq:converge-1-multi}, (\ref{eq:first_term_multiple}), (\ref{eq:statistical_multiple}), and~(\ref{eq:second_term_multiple}), we get
\begin{equation}\label{eq:each-iter-sc-multiple}
\twonms{\vecw^{H(t+1)} - \vecw^*} \le (1-\frac{\eta_t L\lambda}{L + \lambda})^H\twonms{\vecw^{Ht} - \vecw^*} + \eta_t\Delta_i + H^2\eta_t^2L^2D.
\end{equation}

By taking $\eta_t= \frac{a}{L(t+a)}$
where $a=\frac{L+\lambda}{\lambda}$, we can verify that $\eta_t\leq 1/L$. Now we multiply $(t+a)^{H-1}/\eta_t$ on both sides, which gives
\begin{align*}
    \frac{(t+a)^{H-1}\twonms{{\vecw^{H(t+1)}} - \vecw^*}}{\eta_t} &\le (\frac{t+a-1}{t+a})^H\frac{(t+a)^{H-1}}{\eta_t}\twonms{\vecw^{Ht} - \vecw^*} + (t+a)^{H-1}\Delta_i + (t+a)^{H-1}\eta_t H^2L^2D \\ 
    & = \frac{(t+a-1)^{H-1}}{\eta_{t-1}}\twonms{\vecw^{Ht} - \vecw^*} + (t+a)^{H-1}\Delta_i + (t+a)^{H-1}\eta_t H^2L^2D.
\end{align*}
By recursively applying the above inequality, we get
\begin{align*}
     \frac{(T+a-1)^{H-1}\twonms{{\vecw^{HT}} - \vecw^*}}{\eta_{T-1}}\leq {a^{H-1}L} \|\vecw^0-\vecw^\star\|_2 + (\sum_{t=1}^{T-1} (t+a)^{H-1} \Delta_i + {a(t+a)^{H-2}H^2LD}).
\end{align*}
Rearranging the above formula gives
\begin{align*}
     \twonms{{\vecw^{HT}} - \vecw^*}\leq \left(\frac{a}{T+a-1}\right)^H \|\vecw^0-\vecw^\star\|_2 +\frac{C_1a^2H^2D}{(T+a-1)^2} +  C_2 \Delta_i.
\end{align*}

\end{proof}

\vspace{-5pt}
\section{Incorporating the Privacy Guarantee}\label{sec:privacy}
\vspace{-5pt}

Secure aggregation techniques~\cite{bonawitz2017practical,bell2020secure} are designed to conceal individual client's updates and reveal only the aggregated global update to a semi-honest server that attempts to infer the clients' privacy from their updates.
In addition to Byzantine robustness, we aim to achieve privacy guarantees. 
Orchestrating robust FL estimators (to mitigate malicious clients) with secure aggregation schemes (to mitigate semi-honest servers) is challenging, as robust estimators require access to local updates, whereas secure aggregation schemes normally hide them from the server.
Consequently, most de facto FL protocols cannot protect both the server and the clients \textit{simultaneously}, but must aim to satisfy only one of the two criteria.
Two recent works~\cite{so2020byzantine,he2020secure} attempt to address this challenge through the use of secure multi-party computation (MPC) but they either incur high communication cost or rely on non-colluding servers. This does not fit in the FL setting (see Appendix~\ref{sec:related_work} for further discussion).
Thus we are motivated to consider the following question:
\begin{quote}
    \emph{Can we achieve robustness and privacy simultaneously in FL protocols?}
\end{quote}
%
\revision{
The idea of bucketing can also be used to reconcile robust FL with secure aggregation. Similar idea of adding secure aggregation in the bucketing procedure has also been proposed and analyzed in~\cite{velicheti2021secure, burkhalter2021rofl}. }
Secure aggregation~\cite{bonawitz2017practical,bell2020secure} is a class of cryptographic protocols designed to augment user privacy in FL.
Concretely, secure aggregation conceals all the unnecessary information (\emph{i.e.} the gradient from each client) from the server and only reveals the minimum information needed to update the global model (\emph{i.e} the summed gradient).
Thus, it was considered to be incompatible with robust FL because robust FL requires each client's gradient to filter out the malicious ones.

However, with the idea of bucketing, we no longer require all the client's gradients.
Instead, we only need the sums of gradients for the buckets.
This leaves space for secure aggregation within each bucket and thus reconciles the two originally incompatible techniques to support bidirectional protection in federated learning. 
%
%
A subtlety is that in order to apply secure aggregation, we need to clip and quantize the clients' local updates to fit them in a finite field.
As a result, the operator norm of the clients' updates are changed. 
However, it is easy to show that the difference will only be in the constants so the asymptotic results will remain the same.

\revision{
\begin{theorem}
[Security against semi-honest server]
Let $\Pi$ be an instantiation of \system, there exists a PPT (probabilistic polynomial Turing machine) simulator $\textsc{Sim}$ which can only see the averaged updates from the shards. For all clients $\mathcal{C}$, the output of $\textsc{Sim}$ is computationally indistinguishable from the view of that real server $\Pi_\mathcal{C}$ in that execution, i.e., $\Pi_\mathcal{C} \approx \textsc{Sim}(\{ \textbf{g}_t^{H_j} \}_{j \in [p]})$.

\label{cor:sec}

\end{theorem}

\begin{proof}[Proof for Theorem~\ref{cor:sec}]
The transcript of the server is the updates from the sharded clients $\{ \textbf{g}_t^{H_j} \}_{j \in [p]}$.
Hence, Theorem~\ref{cor:sec} is equivalent to the following lemma since the $\textsc{Sim}$ can split the aggregated updates into several random shards which is computationally indistinguishable from the true transcript. 

\begin{lemma}[Lemma 6.1 in~\cite{bonawitz2017practical}]
Given any shard $H_k$ formed by a set of clients $C_k$, the parameter size $d$, the group size $q$, and the updates $\textbf{g}^{(i)}$ where $\forall i\in C_k$, $\textbf{g}^{(i)}\in \mathbb{Z}^d_q$, we have
\begin{equation*}
\begin{split}
    &\{\{\textbf{u}_{ij}\overset{\$}{\leftarrow}\mathbb{Z}_q^d\}_{i<j}, \textbf{u}_{ij}\coloneqq -\textbf{u}_{ji} \; \forall \; i,j \in C_k, \; i > j \quad : \quad \{\textbf{g}^{(i)}+\sum_{j\in C_k \slash i}\textbf{u}_{ij} \pmod{q}\}_{i\in C_k}\}
    \\ & \equiv \; \{\{\textbf{v}_i\overset{\$}{\leftarrow}\mathbb{Z}_q^d\}_{i\in C_k}~s.t.~ \sum_{i\in C_k}\textbf{v}_i=\sum_{i\in C_k}\textbf{g}^{(i)} \pmod{q} \quad : \quad \{\textbf{v}_{i}\}_{i\in C_k}\}
\end{split}
\end{equation*}

\noindent, where $\textbf{u}_{ij}$ is the random mask shared between client $i$ and $j$, $\overset{\$}{\leftarrow}$ donates uniformly sampling from some field, and $\equiv$ denotes that the distributions are identical.

\label{lem:sec}
\vspace{-5pt}
\end{lemma}

Lemma~\ref{lem:sec} illustrates that the distribution of updates with random
masks added is identical to uniformly sampling from $\mathbb{Z}_q^d$. Thus,
individual clients' updates are securely concealed behind the random masks added
by secure aggregation, and a semi-honest server can infer zero information about 
individual clients from the aggregated updates alone.
In the following, we prove Lemma~\ref{lem:sec} via induction on $n$, where $n$
is the size of the clients set $C_k$, $n = |C_k|$.

\noindent \textit{Base Case:}~When $n=2$, assume $C_k = \{i, j\}$, $i < j$, and $\sum_{i\in C_k}\textbf{g}^{(i)} \pmod{q} = c$, $c$ is a constant. The first elements of two distributions are $\textbf{g}^{(i)} + \textbf{u}_{ij} \pmod{q}$ and $\textbf{v}_{i}$, respectively, both of which are uniformly random sampled from $\mathbb{Z}_q^d$.
The second elements are $\textbf{g}^{(j)} + \textbf{u}_{ji} \pmod{q} = c - (\textbf{g}^{(i)} + \textbf{u}_{ij}) \pmod{q}$ and $\textbf{v}_{j} = c - \textbf{v}_{i} \pmod{q}$, respectively, which are the sum $c$ minus the corresponding first elements.
As a result, the distributions are identical.

\noindent \textit{Inductive Hypothesis:}
When $n=k$, the lemma holds.

\noindent \textit{Inductive Step:}
According to the inductive hypothesis, the left and right distributions of the first $k$ clients are indistinguishable.
We follow the protocol to generate the left transcript when the $(k+1)^{th}$ client is added to the shard.
To deal with the right-hand-side transcript, we first add the same randomness as for the left-hand-side to the first $k$ updates and then subtract them from the total sum to obtain the $(k+1)^{th}$ update.
It's easy to prove that the first $k$ updates on both sides follow the
same uniformly random distribution, and that the $(k+1)^{th}$ update is the
difference between the total sum and the sum of the first $k$ updates.
Hence, the left and right transcripts are indistinguishable.

\end{proof}

\vspace{-0.5em}
In case the readers are not familiar with the simulation proof technique, please refer to~\cite{Lindell2017} for more information about simulation-based security proof.
}
\newpage
\section{Evaluation on Breaking Points}
\label{sec:append-breaking-point}

To empirically evaluate the breaking point of the proposed robust aggregators, we test \textsc{Filtering} and \textsc{No-Regret} on MNIST with different number of malicious clients as shown in Figure~\ref{fig:breaking_point}.
We can tell that when $\epsilon<=0.4$, both \textsc{Filtering} and \textsc{No-Regret} works well.
Only when $\epsilon$ reaches $0.5$, \textsc{No-Regret} breaks down because the number of malicious clients is the same as the benign ones and no robust aggregator can work under such settings.

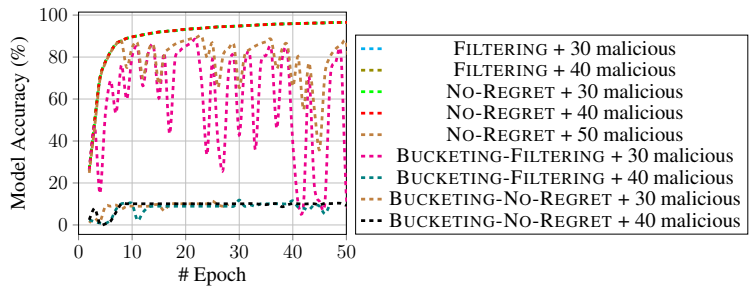
\begin{figure}[!h]
    \centering
    \resizebox{0.6\textwidth}{!}{\begin{tikzpicture}
\begin{axis}[
  set layers,
  grid=major,
  xmin=0, xmax=50,
  ytick align=outside, ytick pos=left,
  xtick align=outside, xtick pos=left,
  xlabel=\# Epoch,
  ylabel={Model Accuracy (\%)},
  legend pos=south east,
  enlarge y limits=0.05,
  legend style={at={(2.5,0)}}
  ]

\addplot+[
  cyan, dashed, mark=none, line width=2pt, 
  smooth, 
  error bars/.cd, 
    y fixed,
    y dir=both, 
    y explicit
] table [x expr=\thisrow{x}+1, y=y, col sep=comma] {data/breaking_point/MNIST/trimmedmean_filterl2_30.txt};
\addlegendentry{\textsc{Filtering} + 30 malicious}

\addplot+[
  olive, dashed, mark=none, line width=2pt, 
  smooth, 
  error bars/.cd, 
    y fixed,
    y dir=both, 
    y explicit
] table [x expr=\thisrow{x}+1, y=y, col sep=comma] {data/breaking_point/MNIST/trimmedmean_filterl2_40.txt};
\addlegendentry{\textsc{Filtering} + 40 malicious}

\addplot+[
  green, dashed, mark=none, line width=2pt, 
  smooth, 
  error bars/.cd, 
    y fixed,
    y dir=both, 
    y explicit
] table [x expr=\thisrow{x}+1, y=y, col sep=comma] {data/breaking_point/MNIST/trimmedmean_ex_noregret_30.txt};
\addlegendentry{\textsc{No-Regret} + 30 malicious}

\addplot+[
  red, dashed, mark=none, line width=2pt, 
  smooth, 
  error bars/.cd, 
    y fixed,
    y dir=both, 
    y explicit
] table [x expr=\thisrow{x}+1, y=y, col sep=comma] {data/breaking_point/MNIST/trimmedmean_ex_noregret_40.txt};
\addlegendentry{\textsc{No-Regret} + 40 malicious}

\addplot+[
  brown, dashed, mark=none, line width=2pt, 
  smooth, 
  error bars/.cd, 
    y fixed,
    y dir=both, 
    y explicit
] table [x expr=\thisrow{x}+1, y=y, col sep=comma] {data/breaking_point/MNIST/trimmedmean_ex_noregret_50.txt};
\addlegendentry{\textsc{No-Regret} + 50 malicious}

\addplot+[
  magenta, dashed, mark=none, line width=2pt, 
  smooth, 
  error bars/.cd, 
    y fixed,
    y dir=both, 
    y explicit
] table [x expr=\thisrow{x}+1, y=y, col sep=comma] {data/breaking_point/MNIST/trimmedmean_mom_filterl2_30.txt};
\addlegendentry{\textsc{Bucketing-Filtering} + 30 malicious}

\addplot+[
  teal, dashed, mark=none, line width=2pt, 
  smooth, 
  error bars/.cd, 
    y fixed,
    y dir=both, 
    y explicit
] table [x expr=\thisrow{x}+1, y=y, col sep=comma] {data/breaking_point/MNIST/trimmedmean_mom_filterl2_40.txt};
\addlegendentry{\textsc{Bucketing-Filtering} + 40 malicious}

\addplot+[
  brown, dashed, mark=none, line width=2pt, 
  smooth, 
  error bars/.cd, 
    y fixed,
    y dir=both, 
    y explicit
] table [x expr=\thisrow{x}+1, y=y, col sep=comma] {data/breaking_point/MNIST/trimmedmean_mom_ex_noregret_30.txt};
\addlegendentry{\textsc{Bucketing-No-Regret} + 30 malicious}

\addplot+[
  black, dashed, mark=none, line width=2pt, 
  smooth, 
  error bars/.cd, 
    y fixed,
    y dir=both, 
    y explicit
] table [x expr=\thisrow{x}+1, y=y, col sep=comma] {data/breaking_point/MNIST/trimmedmean_mom_ex_noregret_40.txt};
\addlegendentry{\textsc{Bucketing-No-Regret} + 40 malicious}


\end{axis}
\end{tikzpicture}}
    \caption{Evaluation w/ different number of malicious clients.}
    \label{fig:breaking_point}
\end{figure}
\section{Evaluation of Sever}
\label{sec:append-sever}

We also evaluate the performance of Sever~\cite{diakonikolas2019sever} in federated learning.
Sever requires that the clients' updates are from a $\gamma$-approximate learner.
Thus, we invoke Sever after the model converges (i.e., Epoch 20 in \F~\ref{fig:sever_evaluation}).
Sever computes outlier scores for the clients' updates and filters the malicious clients.
And then, we retrain the model and invoke Sever after convergence till we reach a stationary point.
We evaluate the performance of Sever on MNIST under MPA attack. 
As shown in \F~\ref{fig:sever_evaluation}, Sever cannot successfully mitigate MPA.

\begin{figure*}[!htbp]
     \centering
     \hfill
     \begin{subfigure}[b]{\figwidth}
         \centering
         \resizebox{\textwidth}{!}{\begin{tikzpicture}
\begin{axis}[
  grid=major,
  xmin=0, xmax=100,
  ytick align=outside, ytick pos=left,
  xtick align=outside, xtick pos=left,
  xlabel=\# Epoch,
  ylabel={Model Accuracy (\%)},
  enlarge y limits=0.05
  ]

\addplot+[
  red, dashed, mark=none, line width=1.6pt, 
  error bars/.cd, 
    y fixed,
    y dir=both, 
    y explicit
] table [x expr=\thisrow{x}, y=y, col sep=comma] {data/main/MNIST/noattack_sever.txt};
\addlegendentry{Sever}

\legend{};


\end{axis}
\end{tikzpicture}}
         \caption{MNIST w/o attack.}
         \label{fig:sever_mnist_noattack}
     \end{subfigure}
     \hfill
     \begin{subfigure}[b]{\figwidth}
         \centering
         \resizebox{\textwidth}{!}{\begin{tikzpicture}
\begin{axis}[
  set layers,
  grid=major,
  xmin=0, xmax=100,
  ytick align=outside, ytick pos=left,
  xtick align=outside, xtick pos=left,
  xlabel=\# Epoch,
  ylabel={Attack Success Rate (\%)},
  legend pos=south east,
  enlarge y limits=0.05
]

\addplot+[
  red, dashed, mark=none, line width=1.6pt, 
  error bars/.cd, 
    y fixed,
    y dir=both, 
    y explicit
] table [x expr=\thisrow{x}, y=asr, col sep=comma] {data/main/MNIST/modelpoisoning_sever.txt};

\legend{};

\end{axis}
\end{tikzpicture}}
         \caption{MNIST under MPA.}
         \label{fig:sever_mnist_mpa_asr}
     \end{subfigure}
     \hfill
     \begin{subfigure}[b]{\figwidth}
         \centering
         \resizebox{\textwidth}{!}{\begin{tikzpicture}
\begin{axis}[
  set layers,
  grid=major,
  xmin=0, xmax=100,
  ytick align=outside, ytick pos=left,
  xtick align=outside, xtick pos=left,
  xlabel=\# Epoch,
  ylabel={Model Accuracy (\%)},
  legend pos=south east,
  enlarge y limits=0.05
]

\addplot+[
  red, dashed, mark=none, line width=1.6pt, 
  error bars/.cd, 
    y fixed,
    y dir=both, 
    y explicit
] table [x expr=\thisrow{x}, y=acc, col sep=comma] {data/main/MNIST/modelpoisoning_sever.txt};

\legend{};

\end{axis}
\end{tikzpicture}}
         \caption{MNIST under MPA.}
         \label{fig:sever_mnist_mpa_acc}
     \end{subfigure}
    
    \caption{Performance of Sever on MNIST without attack and with MPA.}
    \label{fig:sever_evaluation}
\end{figure*}
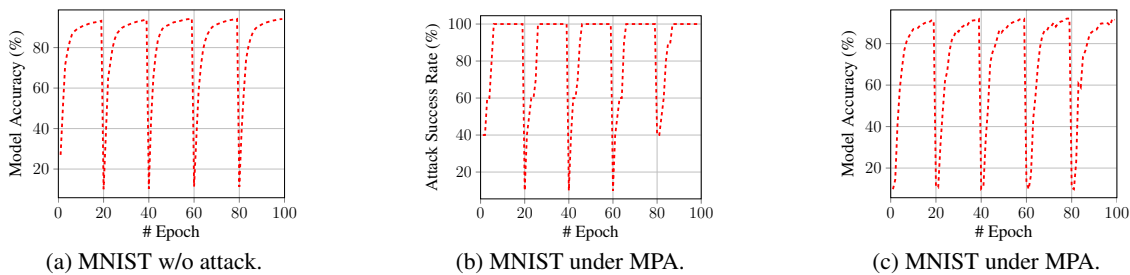

\end{document}